\PassOptionsToPackage{shortlabels}{enumitem}
\documentclass[11pt, a4paper]{sakana_format}

\usepackage{orcidlink}

\numberwithin{equation}{section}
\newtheorem{proposition}{Proposition}
\definecolor{SakanaRed}{RGB}{225,6,0}
\definecolor{SakanaBlack}{RGB}{30,30,30}
\colorlet{black}{SakanaBlack}
\AtBeginDocument{\color{SakanaBlack}}
\def\increase#1{\textcolor{codegreen}{~(#1 $\uparrow$)}}
\def\decrease#1{\textcolor{SakanaRed}{~(-#1 $\downarrow$)}}

\DeclareRobustCommand{\Benchmark}{Percept-Lens\xspace}
\DeclareRobustCommand{\TotalImagesFull}{7.1 million\xspace}

\DeclareRobustCommand{\TotalDatasets}{39\xspace}

\DeclareRobustCommand{\TotalMixedDatasets}{24\xspace}

\DeclareRobustCommand{\ProGAN}{\texttt{ProGAN}\xspace}

\DeclareRobustCommand{\CNNSpot}{\texttt{CNNSpot}\xspace}
\DeclareRobustCommand{\GenImage}{\texttt{GenImage}\xspace}
\DeclareRobustCommand{\DRCT}{\texttt{DRCT-2M}\xspace}
\DeclareRobustCommand{\ELSA}{\texttt{ELSA-D3}\xspace}
\DeclareRobustCommand{\CF}{\texttt{CommunityForensics}\xspace}
\DeclareRobustCommand{\ImageNet}{\texttt{ImageNet-1k}\xspace}
\DeclareRobustCommand{\COCO}{\texttt{COCO}\xspace}
\DeclareRobustCommand{\LAION}{\texttt{LAION-400M}\xspace}

\DeclareRobustCommand{\UnivFD}{\texttt{UnivFD}\xspace}
\DeclareRobustCommand{\AIDE}{\texttt{AIDE}\xspace}
\DeclareRobustCommand{\Effort}{\texttt{Effort}\xspace}

\title{
    Prior-Conditioned Gaussian Discriminants for Generalizable AI-generated Image Detection
}
\correspondingauthor{Shashank Kotyan (\texttt{shashank@sakana.ai})}
\author{Shashank Kotyan}
\author{Makoto Shing}
\author{Yuki Imajuku}
\author{Rujikorn Charakorn}
\author{Tarin Clanuwat}
\affil{Sakana AI, Tokyo, Japan}

\begin{document}

\begin{abstract}
Diffusion-based generators have made synthetic images ubiquitous, but detectors often fail under simultaneous shifts in generator, prompt/style, and source-domain.
We study AI-generated image detection as a transfer system described by training prior, frozen encoder feature space, and decision rule, and ask when classifier head training adds value beyond what is already separable in modern features.
As a controlled diagnostic, we fit a prior-conditioned Gaussian discriminant ladder: closed-form heads built from first- and second-order feature statistics under nested covariance assumptions.
On \Benchmark, a unified protocol over \TotalDatasets public datasets (\TotalImagesFull images), the best rung is frequently competitive with, and sometimes exceeds, released AI-generated image detector heads when matched on both prior and encoder.
We further quantify strong sensitivity to the training prior, data-efficiency of moment-based heads, and representation dependence of Gaussian shift metrics, motivating (prior, encoder, head)-level reporting and stronger analytical baselines for AIGI transfer.

\end{abstract}

\maketitle
\tableofcontents
\clearpage

\section{Introduction}
\label{sec:intro}

\begin{figure*}[!t]
    \centering    \includegraphics[page=1,width=\textwidth]{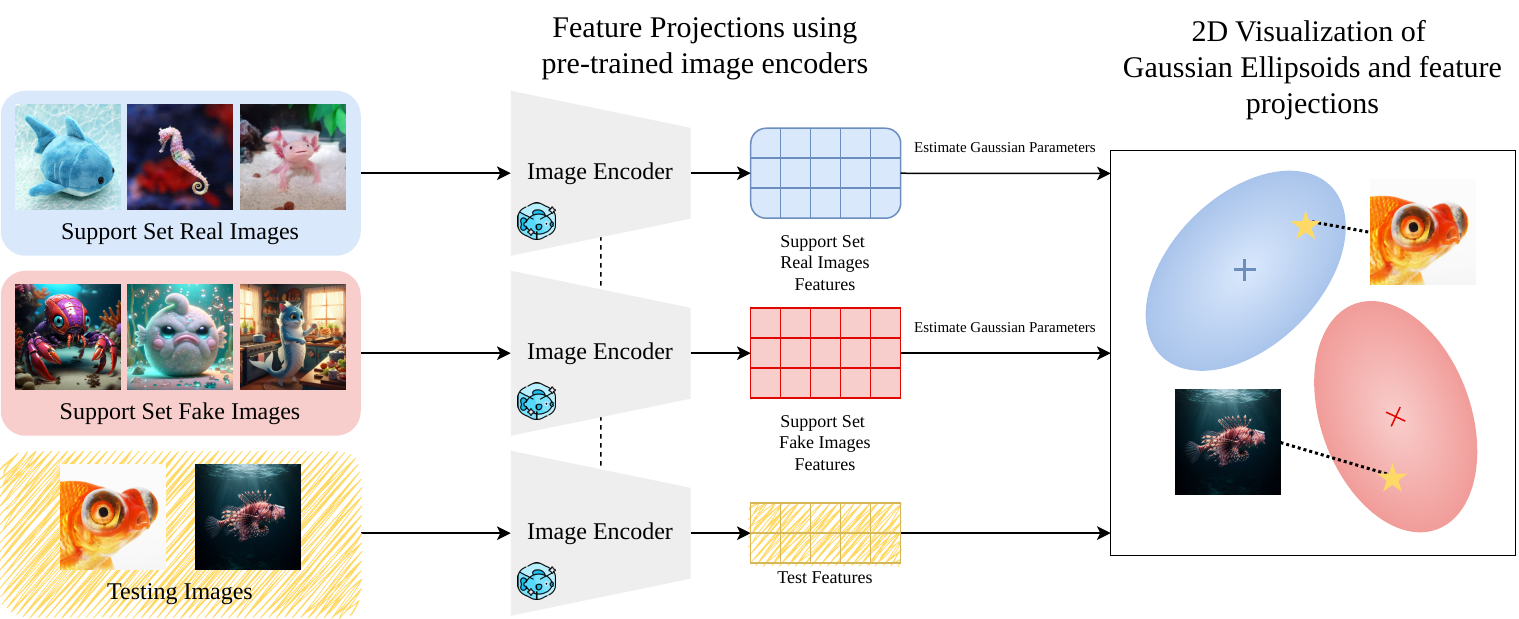}
    \caption{
    \textbf{Training-prior-conditioned Gaussian discriminants in a frozen feature space.}
    A frozen image encoder $\phi$ maps images $x$ to embeddings $z$.
    Given a (possibly few-shot) support set $\mathcal{S}$ drawn from a \emph{training prior}, we estimate class-conditional moments $(\mu_c,\Sigma_c)$ and instantiate a ladder of \emph{classical} closed-form heads under nested covariance assumptions.
    By keeping the encoder and prior fixed, the ladder isolates what the decision rule contributes.
    }
    \label{fig:analytic_pipeline}
\end{figure*}

Diffusion-based generative models \citep{ramesh2021zeroshot,rombach2022high,midjourney2022midjourney,saharia2022photorealistic} have made high-quality AI-generated imagery widely accessible.
These models have practical benefits, but they also complicate trust in visual media and enable misuse of synthetic content \citep{daniel2024how,explosionnews,pentagonnews}.
As a result, AI-generated image (AIGI) detection has become a core recognition problem with direct implications for safety and privacy.

Modern detectors must transfer across shifts in \emph{both} generator family and image-source domain.
In practice these shifts co-occur (new generators, new prompting styles, new platforms, and post-processing), and performance can collapse even when in-distribution metrics look strong.
Moreover, collecting large-scale labeled data from newly emerging generators is often expensive or infeasible.
Together, these conditions make AIGI detection a concrete \emph{transfer / low-shot recognition} problem, where a deployed detector must operate under joint shift and sometimes with only a small support set for calibration.

Existing evaluations \citep{zhu2023genimage,chen2024drct,baraldi2025contrasting} typically probe only part of this space, frequently reusing narrow prompt regimes and closely related sources.
Such evaluations can conflate genuine generalization with prior-specific shortcuts.
We take a stricter, system-level view.
A reported number is a property of a \emph{system} consisting of a training prior, a frozen image encoder, and a decision rule on top.
Here, \emph{training prior} denotes the data distribution induced by a public training dataset, such as \CF \citep{park2025community}; it is not a Bayesian prior over model parameters.
If generalization fails, it is not obvious whether the bottleneck is the representation, the head optimization, or the prior itself.

To isolate the decision-rule factor while keeping the prior and representation fixed, we fit \emph{training-prior-conditioned Gaussian discriminants}, which are classical closed-form heads obtained from first- and second-order statistics of frozen encoder features (\Cref{fig:analytic_pipeline}).
We organize these rules into a \emph{Gaussian discriminant ladder} spanning nested covariance assumptions (isotropic, diagonal, shared, class-specific).
Because each rung has an analytic solution, the ladder is a reproducible, hyperparameter-light baseline.
The identity of the best-performing rung also indicates which low-order feature statistics transfer across generator and domain shifts.
Although Gaussian discriminants are classical, we argue that they are underused as diagnostic baselines in AIGI detection, and we find that they can be competitive under matched (prior, encoder) conditions.

Empirically, across \TotalDatasets public test sets (\TotalImagesFull images), at least one rung is often competitive with trained heads under \emph{matched} priors and frozen encoders.
As a retrospective diagnostic upper bound, the best rung can even surpass released heads in several settings.
Our decomposition is not fully symmetric across the three factors.
The most direct intervention is a matched head replacement, while the encoder and prior experiments are controlled one-factor sweeps.
We use the procedure as a source-side audit.
If a trained head does not outperform the best rule in the ladder on the same final representation and support prior, classifier training is unlikely to be the source of OOD transfer gains.
If it succeeds, it signals structure beyond second moments.

\paragraph{Contributions.}

~

\noindent\textbf{Prior-conditioned Gaussian discriminant ladder (diagnostic baseline).}
We instantiate a ladder of \emph{classical} closed-form Gaussian heads on frozen encoder features, varying covariance assumptions (isotropic $\rightarrow$ diagonal $\rightarrow$ shared full $\rightarrow$ class-specific full).
The best-transfer rung provides an interpretable diagnostic of which low-order feature statistics are preserved under shift.

\noindent\textbf{Post-hoc controlled evaluation under matched (prior, encoder) conditions.}
On \Benchmark, a unified evaluation \emph{protocol} built entirely from existing public datasets, we compare released AI-generated image detector heads against the ladder using the same training prior and the same frozen encoder features.
This isolates when trained heads add value beyond low-order geometry.

\noindent\textbf{Empirical drivers of generalization under joint shift.}
We quantify sensitivity to the training prior under a fixed encoder, show that moment-based heads can adapt with few labeled samples on strong frozen encoders, and demonstrate that Gaussian shift metrics are strongly representation-conditioned.

\section{Related Works}

\subsection{AI-generated Image (AIGI) detection under Distribution Shift}

The literature on AIGI detection spans artifact-driven detectors \citep{yang2019exposing,he2021forgerynet,wang2020cnngenerated} and representation-based detectors \citep{ojha2023universal,baraldi2025contrasting,park2025community}.
Early work focused on GAN-era artifacts and domain-specific settings such as face forgeries \citep{yang2019exposing,he2021forgerynet,wang2020cnngenerated}.
Diffusion models increased photorealism and prompt diversity, and exposed a persistent failure mode: detectors trained on narrow synthetic prior exploit shortcuts that do not transfer across generators or domains \citep{zhu2023genimage,chen2024drct,baraldi2025contrasting,xiao2025are}.
\citet{chen2024drct} and \citet{park2025community} broaden training priors by scaling generator coverage.
\citet{ojha2023universal}, \citet{baraldi2025contrasting}, and \citet{zhou2025brought} instead use foundation-model features and lightweight heads.
\citet{yang_your_2026} adjust detector logits or thresholds after training under shift.
Our intervention instead replaces the classifier in feature space while matching the prior and encoder.
Thus our emphasis is not another trained detector, but what is already separable in frozen encoder space and what the trained head adds under the \emph{same} prior.
This is closely related to the observation that many OOD detectors can be expressed as generative scoring rules on features, but the implications for AIGI detection under joint shift have not been carefully isolated.

\subsection{Analytical Gaussian Discriminants and Feature-Space Scoring}
Gaussian discriminant analysis and Mahalanobis scoring are classical tools for uncertainty and OOD detection \citep{lee2018simple}.
\citet{wu2025fewshot} connect few-shot prototypical inference to Euclidean nearest-centroid rules, while \citet{yan2025orthogonal} explore orthogonal subspace decompositions that preserve pretrained structure under adaptation.

We build on these ideas, but use them as a \emph{diagnostic ladder}: isotropic, diagonal, tied, and class-specific covariance assumptions correspond to increasing geometric expressivity.
The decision rules themselves are not new.
Our contribution is to condition them on explicit training priors and use the ladder covariance model to interpret what transfers, and what does not, across generator and domain shifts.

\section{Preliminaries}
\label{sec:preliminaries}
\subsection{Problem setup}
\label{sec:problem}
We study binary detection of AI-generated images.
Given a training set $\mathcal{D}_\text{train} = \{(x_i, y_i)\}_{i=1}^{n}$, where $x_i$ are images and $y_i \in \{0,1\}$ are labels for real ($y=0$) and synthetic ($y=1$) images, the goal of AIGI detection is to learn a classifier that can accurately predict the label of a previously unseen image.

\subsection{Feature Extraction and Trained Classifiers}
\label{sec:feature}
We use a frozen image encoder $\phi: \mathcal{X} \rightarrow \mathbb{R}^D$ from a pre-trained foundation model to map each image $x$ to a $D$-dimensional feature vector $z = \phi(x)$.
A standard approach trains a discriminative head, such as a linear probe, on these features by minimizing a loss function (e.g., binary cross-entropy) via gradient descent: $\min_{w,b} \sum_{i=1}^n \text{CE}(\sigma(w^\top z_i + b), y_i)$
where $\sigma(\cdot)$ is the sigmoid activation function and $w \in \mathbb{R}^D$ and $b \in \mathbb{R}$ define a separating hyperplane.

More complex models, such as multi-layer perceptrons (MLPs), can learn non-linear decision boundaries.
This process iteratively searches for an optimal boundary based on the training data.
Although powerful, these methods can overfit to artifacts specific to the generative models in the training set, potentially limiting generalization to unseen generative processes \citep{yan2025orthogonal,park2025community}.

We compare (i) \emph{trained} discriminative detectors that optimize a supervised objective on a given prior and (ii) \emph{Gaussian discriminant rule} classifiers fitted on frozen features (\Cref{sec:analytical}).
The audit applies to released systems that expose the feature vector used by a classifier, including frequency/statistical detectors \citep{frank2020leveraging,yan2024sanity}, reconstruction-based detectors \citep{cazenavette2024fakeinversion,chen2024drct}, LoRA-adapted detectors \citep{yan2025orthogonal}, or end-to-end detectors \citep{wang2020cnngenerated,baraldi2025contrasting} after their final representation is fixed.
Methods that expose only a scalar anomaly score can still be evaluated by \Benchmark as complete detectors, but their internal head cannot be replaced by the ladder.
The split is diagnostic: if a trained head underperforms a closed-form rule under OOD shift, the trained head's decision surface and/or training prior is implicated rather than only the encoder capacity.

\section{Prior-Conditioned Gaussian Discriminants as Few-Shot Discriminant Heads}
\label{sec:analytical}

We use Gaussian class-conditional models as a \emph{controlled diagnostic} of what information is already separable in a frozen feature space.
Motivated by classical generative classification and recent analyses of contrastive representations using mixture models \citep{bansal2025understanding}, we approximate

\begin{equation}
    p(z \mid y=c) ~\approx~ \mathcal{N}(\mu_c, \Sigma_c), \quad c \in \{0,1\}.
\end{equation}

Given a support set $\mathcal{S}$ drawn from a training prior, $\mathcal{S} \subseteq \mathcal{D}_\text{train}$, we estimate $(\mu_c, \Sigma_c)$ by sample moments.
We also define the pooled covariance $\Sigma_p = \frac{(n_0-1)\Sigma_0 + (n_1-1)\Sigma_1}{n_0+n_1-2}$, used when assuming a shared covariance structure.
In high-dimensional embeddings, covariance inversion can be ill-conditioned.
We therefore use standard regularization (diagonal loading and shrinkage) when computing $\Sigma^{-1}$ and log-determinants (Appendix~\Cref{sec:extended_setup}).
The Gaussian assumption is diagnostic rather than literal.
Appendix \Cref{tab:univariate_normality_tests,tab:multivariate_normality_tests} shows mostly near-Gaussian marginal summaries on PE-Core features, but also clear failures such as \texttt{FourierSpectrumDiscrepancies} \citep{dzanic2020fourier} and \texttt{DiffusionForensics}-fake \citep{wang2023dire}.
Multimodal or heavy-tailed regimes are exactly where trained non-linear heads may add value.

This yields a ladder of \emph{classical} discriminant rules in closed form, and the best-performing covariance assumption becomes an interpretable indicator of which feature statistics transfer across domains, bypassing iterative training.
(The same estimation extends to $N$-way $K$-shot classification. We focus on binary AIGI detection.)

In this article, we consider Cosine Nearest Centroid Matching (Cos-NCM), Gaussian Naive Bayes (GNB), Mahalanobis Nearest Centroid Matching (Mah-NCM), and Quadratic Discriminant Analysis (QDA) as Gaussian discriminants.
We also include squared Euclidean nearest-centroid matching (Euc-NCM), which mirrors the inference rule used in prototypical-network detectors \citep[e.g.][]{wu2025fewshot} but is applied directly in the frozen encoder space without metric learning.
\cref{tab:compare} provides an overview of Gaussian discriminants in analytical form, while \Cref{fig:gaussian_ladder} provides a visual explanation.

\begin{figure*}[!t]
    \centering
    \includegraphics[page=3,width=\textwidth]{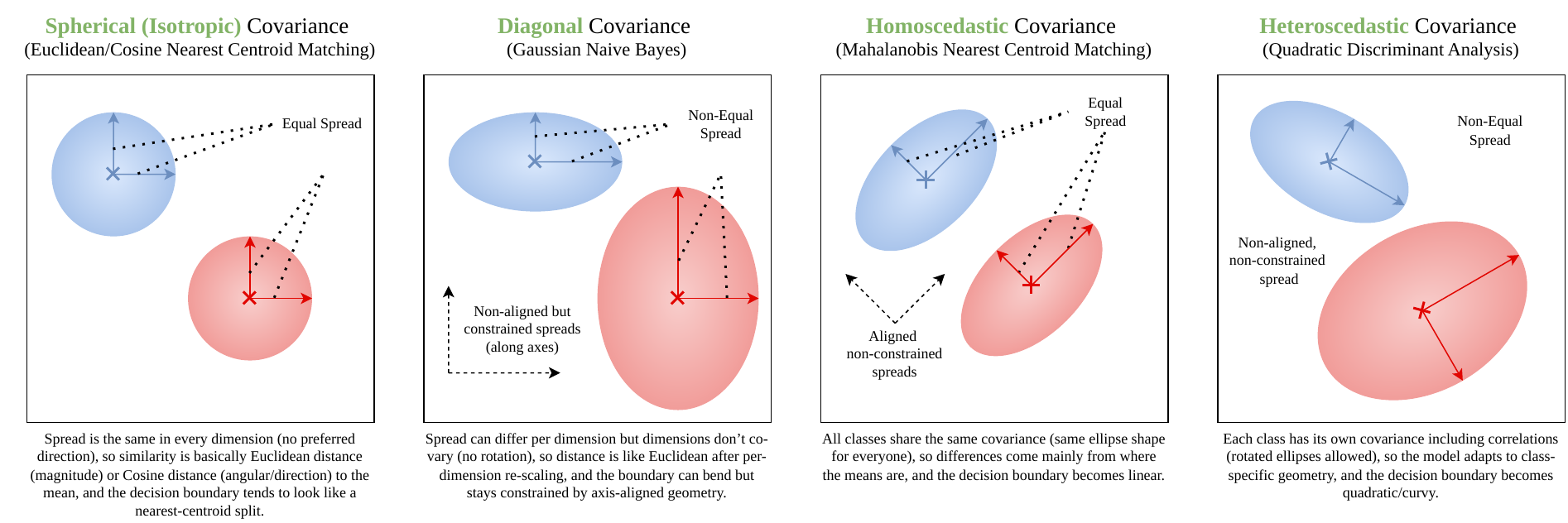}
    \caption{
    \textbf{Gaussian discriminant ladder as nested covariance assumptions.}
    Each rung is a classical generative classifier obtained by restricting the class-conditional covariance: isotropic (Euc-/Cos-NCM), diagonal (GNB), shared full covariance (Mah-NCM), or class-specific full covariance (QDA).
    This nested family makes it possible to diagnose how much transferable signal is captured by first- and second-order feature statistics.
    }
    \label{fig:gaussian_ladder}
\end{figure*}

\begin{table*}[!t]
\centering
\caption{
\textbf{Gaussian discriminant ladder used as analytical baselines.}
Each rung is a generative classifier obtained by restricting the class-conditional covariance (isotropic, diagonal, homoscedastic, or heteroscedastic), which determines the geometry of the induced decision boundary.
We use this ladder to isolate head effects in a fixed frozen encoder feature space rather than to propose a new detector.
}
\label{tab:compare}
\resizebox{\linewidth}{!}{%
\begin{tabular}{@{}p{4.8cm}lllp{6cm}@{}}
\toprule
\textbf{Gaussian Discriminant} & \multicolumn{2}{c}{\textbf{Covariance Structure Assumption}} &  \textbf{Decision Boundary} & \textbf{Characteristic}\\
\Midrule
Euclidean Nearest Centroid Matching (Euc-NCM) & Isotropic & $\Sigma_{c \in \{0,1\}} = \sigma^2 I$ & Linear & Assumes spherical class distributions. \\
Cosine Nearest Centroid Matching (Cos-NCM) & Isotropic & $\Sigma_{c \in \{0,1\}} = \sigma^2 I$ & Linear & Assumes spherical class distributions; decision is based on angular proximity, suitable for normalized embeddings. \\
Gaussian Naive Bayes (GNB) & Diagonal & $\Sigma_{c \in \{0,1\}} = \text{diag}(\sigma_{c,j}^2)$  & Non-Linear & Simplest model; assumes feature independence, useful for identifying decorrelated representations. \\
Mahalanobis Nearest Centroid Matching (Mah-NCM) & Homoscedastic & $\Sigma_0 = \Sigma_1 = \Sigma_p$ & Linear & Assumes classes share the same elliptical shape. It accounts for feature correlations by whitening the space, yielding a correlation-aware linear boundary. \\
Quadratic Discriminant Analysis (QDA) & Heteroscedastic & $\Sigma_0 \neq \Sigma_1$ & Quadratic & Most general model; allows each class to have a unique elliptical shape and orientation, capturing complex separations. \\
\bottomrule
\end{tabular}%
}
\end{table*}

\paragraph{Euclidean Nearest Centroid Matching (Euc-NCM).}
Assuming an isotropic shared covariance $\Sigma_c=\sigma^2 I$, the Gaussian log-likelihood (up to additive constants) is proportional to the squared Euclidean distance to the class mean.
Prediction reduces to nearest-centroid classification in the frozen feature space, matching prototypical inference without metric learning.

\paragraph{Cosine Nearest Centroid Matching (Cos-NCM).}
When features are $\ell_2$-normalized, isotropic Gaussian scoring is monotone in cosine similarity.
Cos-NCM predicts the class whose centroid has the highest cosine similarity to the query.

\paragraph{Gaussian Naive Bayes (GNB).}
GNB assumes diagonal covariance $\Sigma_c = \mathrm{diag}(\sigma_{c,1}^2, \dots, \sigma_{c,D}^2)$, which implies conditionally independent features.
The assumption simplifies classification but ignores feature correlations.

\paragraph{Mahalanobis Nearest Centroid Matching (Mah-NCM).}
Mah-NCM assumes homoscedasticity, i.e., a shared covariance matrix for both classes $\Sigma_0 = \Sigma_1 = \Sigma_p$.
The classification decision is based on the Mahalanobis distance $(z-\mu_c)^\top\Sigma_p^{-1}(z-\mu_c)$, which accounts for the correlations of the features by transforming the feature space into one in which the pooled covariance is the identity.

\paragraph{Quadratic Discriminant Analysis (QDA).}
QDA is the most general of the four classifiers, allowing each class to have its own full-rank covariance matrix ($\Sigma_0 \neq \Sigma_1$), thus assuming heteroskedasticity.
The resulting quadratic boundary can adapt to classes with different shapes and orientations.

\subsection{Connection to classical generative classification}
\label{sec:framework}

The efficacy of these Gaussian discriminants in analytical form can be understood through the lens of Bayesian decision theory, which provides a formal basis for our geometric approach.
Under the Gaussian assumption, they approximate or realize the Bayes-optimal classifier, which minimizes the probability of error.
Full proofs are provided in the Appendix \cref{sec:proof}.

\begin{proposition}[Bayes-Optimal Classifier under Gaussian Assumption]
\label{prop:bayes_optimal}
If the class-conditional densities are Gaussian, $p(z \mid y=c) ~\approx~ \mathcal{N}(\mu_c, \Sigma_c)$, and the class priors are equal, $P(y=0)=P(y=1)=\frac{1}{2}$, the Bayes-optimal decision rule is given by Quadratic Discriminant Analysis (QDA).
\end{proposition}

\begin{proposition}[Optimality under Homoscedasticity]
\label{prop:homoscedastic}
If the classes are additionally assumed to be homoscedastic (i.e., they share a common covariance matrix $\Sigma_0 = \Sigma_1 = \Sigma_p$), the Bayes-optimal decision rule simplifies to Mahalanobis Nearest Centroid Matching (Mah-NCM).
\end{proposition}

\begin{proposition}[Performance Stability under Bounded Drift]
\label{prop:predictive_sufficiency}
Assume a homoscedastic classifier with shared covariance $\Sigma$ and equal priors.
If the statistical drift between the training and testing distributions is bounded by $\|\mu_c^{\text{test}} - \mu_c^{\text{train}}\|_2 \le \epsilon_\mu$ and $\|\Sigma^{\text{test}} - \Sigma^{\text{train}}\|_F \le \epsilon_\Sigma$, then $|\mathrm{AUC}_{\text{test}} - \mathrm{AUC}_{\text{train}}| \le L(\epsilon_\mu + \epsilon_\Sigma)$
where $L$ is the Lipschitz constant. % depends on the spectral norm of the inverse covariance, $\|\Sigma^{-1}\|_2$.
\end{proposition}

\begin{proposition}[Stability of the Fisher Margin under Distributional Drift]
\label{prop:stability}
Assume that both the training and testing domains are characterized by homoscedastic Gaussian parameters.
If the distributional drift is bounded such that $\|\delta^{\text{test}}-\delta^{\text{train}}\|_2 \le \epsilon$ and $\|\Sigma_p^{\text{test}}-\Sigma_p^{\text{train}}\|_F \le \eta$, where $\delta = \mu_1 - \mu_0$, then the absolute change in the Fisher Margin $\mathcal{D}_\mu$ is bounded, to a first order, by $|\mathcal{D}_\mu^{\text{test}} - \mathcal{D}_\mu^{\text{train}}| \le 2 \|(\Sigma_p^{\text{train}})^{-1}\|_2 \|\delta^{\text{train}}\|_2 \epsilon + O(\eta)$.
\end{proposition}

These propositions should be read as \emph{model-conditional}: they characterize the optimal decision rules \emph{if} frozen features are well-approximated by Gaussian class-conditionals.
They provide the theoretical foundation for our geometric analysis of generalization.
Under this model, Quadratic Discriminant Analysis is Bayes-optimal in the heteroscedastic case (\Cref{prop:bayes_optimal}) and simplifies to a linear rule under a shared covariance (Mahalanobis-NCM; \Cref{prop:homoscedastic}).
\Cref{prop:predictive_sufficiency} shows that, under the homoscedastic Gaussian model, the AUC of Mah-NCM is a monotone function of the Fisher margin $\mathcal{D}_\mu$.
\Cref{prop:stability} bounds how $\mathcal{D}_\mu$ changes under drift in low-order statistics.
Together, these results imply smooth degradation when class means and covariances drift.
In this model, reliable generalization requires geometric alignment between training and testing data.

The linear decision boundary of Mah-NCM has the same functional form as a linear probe trained with cross-entropy.
Both induce a single separating hyperplane in $z$-space.
The key difference is optimization rather than representation.
A neural network iteratively searches for the boundary, whereas Gaussian discriminants compute it analytically from the data's first- and second-order statistics.
Thus, when a trained linear head does not outperform Gaussian discriminants under matched priors and encoders, the result is consistent with the hypothesis that second-order feature geometry is already sufficient for separation.
Consistent gaps instead suggest non-Gaussian or higher-order structure that the ladder cannot capture.

\paragraph{Limitations.}
\label{sec:limitations}
By design, the Gaussian ladder is restricted to a Gaussian class-conditional approximation based on first- and second-order statistics.
The result is a diagnostic baseline rather than a universal detector, one that leaves out higher-order, multimodal, or heavy-tailed cues from strong post-processing, unusual content, or generator-specific artifacts.
In those regimes, trained non-linear heads or end-to-end adaptation can legitimately outperform the ladder.

\section{Investigating Generalization with \Benchmark suite}
\label{sec:results}
We use the Gaussian ladder to investigate which components of a detection system drive OOD performance under realistic distribution shift.
We focus on the training prior, the frozen encoder, and the decision rule on top.
To assess the reliability of current AI-generated image detectors under realistic generative variability, we begin by evaluating pre-trained models on the \Benchmark test suite.
These results help to localize whether failures are attributable to the training prior, the encoder feature space, or the trained decision surface, motivating more principled system-level investigations.

\subsection{Evaluation Protocol and Data}
\label{sec:benchmark}

\begin{figure*}[!t]
    \centering
    \includegraphics[page=2,width=\textwidth]{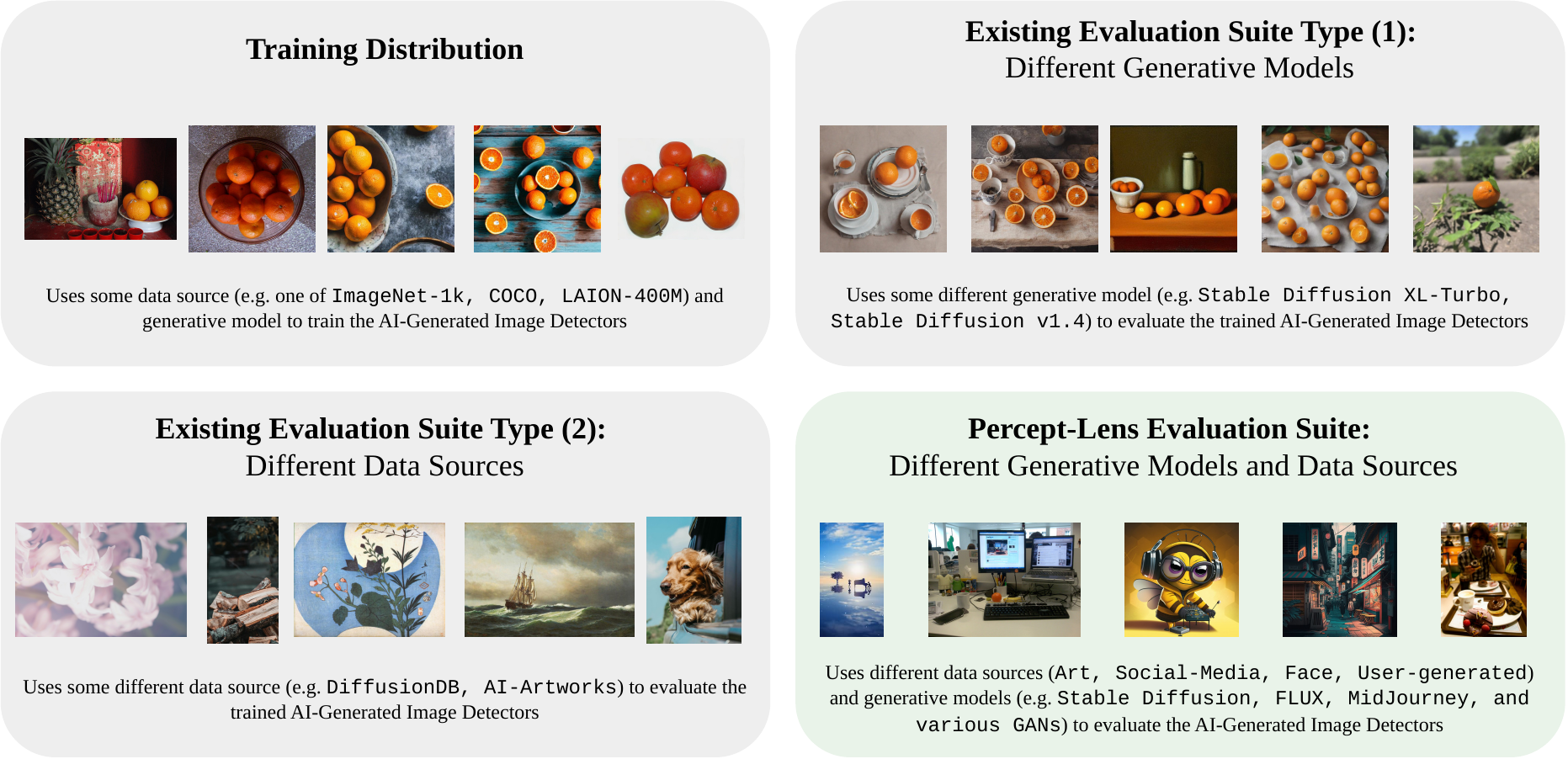}
    \caption{
    \textbf{\Benchmark stress-tests detectors under joint distribution shift.}
    \Benchmark is an evaluation \emph{protocol} built entirely from existing public datasets.
    Many AIGI evaluations vary either the generator family \emph{or} the image source in isolation; we instead evaluate under \emph{joint} shifts in generator family, prompt/style, and source domain, closer to deployment where these factors drift together.
    }
    \label{fig:benchmark_axes}
\end{figure*}

\Benchmark denotes a unified evaluation \emph{protocol} constructed entirely from existing public datasets to stress-test AIGI detectors under joint distribution shift.
Unlike settings where only the generator family changes while the underlying image source and prompt distribution remain fixed, \Benchmark explicitly couples shifts in generator family, prompt/style, and acquisition pipeline (\Cref{fig:benchmark_axes}).
To avoid conflating in-distribution performance with transfer, the large training priors used for conditioning (e.g., \CF, \GenImage, \DRCT, \ELSA) are treated as support-only and are not part of the evaluation suite.

\paragraph{Evaluation Datasets.}
The \Benchmark suite aggregates \TotalDatasets existing public datasets totaling \TotalImagesFull images.
It spans real-only, synthetic-only, and mixed regimes, with joint shifts in (i) generator family, (ii) prompt/style, and (iii) image-source domain and post-processing.
Concretely, it combines in-the-wild real sources, synthetic-only prompt/style collections, and mixed forensic benchmarks spanning diverse generators and post-processing.

\Benchmark standardizes evaluation over existing public datasets to ensure head--prior--encoder comparisons are tested under a realistic mixture of shifts.
The complete manifest of the dataset and the composition of the class are provided in Appendix~\Cref{tab:summaryval}.

\begin{table*}[!t]
\centering
\caption{
\textbf{Closed-form Gaussian discriminants versus released AI-generated image detector heads under matched (prior, encoder) conditions.}
Mean class accuracy $\mathrm{CA}$ (dataset-wise balanced accuracy where defined; \Cref{sec:benchmark}), macro-averaged across \Benchmark evaluation datasets.
For each released detector checkpoint, we report (i) the released decision head (\emph{Out-of-the-Shelf}) and (ii) closed-form baselines fitted on the \emph{same frozen encoder features} using the corresponding public training prior: Euc-NCM \citep{wu2025fewshot} and the best-performing rung of the Gaussian ladder.
The full ladder is reported in \Cref{tab:comparative}.
(Parentheses indicate the training subset reported by the original work when applicable.)
}
\label{tab:overall}
\resizebox{0.99\textwidth}{!}{%
\begin{tabular}{@{}lrrr@{}}
\toprule
\textbf{Detection Model}
& \textbf{Out-of-the-Shelf}
& \textbf{Euc-NCM \citep{wu2025fewshot}}
& \textbf{Gaussian ladder (best)}
\\
\Midrule
\rowcolor{gray!20}
\multicolumn{4}{c}{\shortstack{Trained with \CNNSpot \citep{wang2020cnngenerated} (\ProGAN Images based on \texttt{LSUN})}}  \\
\UnivFD \citep{ojha2023universal} & 54.63\% & 50.98\% & \textbf{57.62\%} (Mah-NCM) \increase{2.99\%}\\
\AIDE \citep{yan2024sanity}       & 56.41\% & 54.25\% & \textbf{62.76\%} (Mah-NCM) \increase{6.35\%}\\
\midrule
\rowcolor{gray!20}
\multicolumn{4}{c}{\shortstack{Trained with \GenImage \citep{zhu2023genimage} (Diffusion Model Images based on \ImageNet)}} \\
\AIDE (\texttt{GenImage-SDv1}) \citep{yan2024sanity}               & 48.99\% & 54.53\% & \textbf{57.35\%} (Mah-NCM) \increase{8.36\%}\\
\Effort (\texttt{GenImage-SDv1}) \citep{yan2025orthogonal}         & 72.58\% & 75.06\% & \textbf{76.82\%} (Mah-NCM) \increase{4.24\%}\\
\texttt{DRCT-UnivFD} (Full \texttt{GenImage}) \citep{chen2024drct} & 65.77\% & 66.08\% & \textbf{73.01\%} (Mah-NCM) \increase{7.24\%} \\
\AIDE (Full \texttt{GenImage}) \citep{yan2024sanity}               & 54.98\% & 52.24\% & \textbf{64.64\%} (Mah-NCM) \increase{9.66\%} \\
\midrule
\rowcolor{gray!20}
\multicolumn{4}{c}{\shortstack{Trained with \DRCT \citep{chen2024drct} (Stable Diffusion Model Images based on \COCO)}} \\
\texttt{DRCT-UnivFD} (\texttt{DRCT-SDv1}) \citep{chen2024drct} & 63.36\% & 65.36\% & \textbf{70.20\%} (QDA) \increase{6.84\%} \\
\texttt{DRCT-UnivFD} (\texttt{DRCT-SDv2}) \citep{chen2024drct} & 62.37\% & 66.31\% & \textbf{70.31\%} (QDA) \increase{7.94\%} \\
\midrule
\rowcolor{gray!20}
\multicolumn{4}{c}{\shortstack{Trained with \ELSA \citep{baraldi2025contrasting} or \CF \citep{park2025community} (Diffusion Model Images based on \LAION Dataset)}} \\
\texttt{CoDE-kNN} (\ELSA) \citep{baraldi2025contrasting} & 64.01\% & 63.63\% & \textbf{66.15\%} (Mah-NCM) \increase{2.14\%}\\
\texttt{CF-224} (\CF) \citep{park2025community}          & 81.98\% & 78.53\% & \textbf{82.61\%} (Mah-NCM) \increase{0.63\%} \\
\texttt{CF-384} (\CF) \citep{park2025community}          & \textbf{87.54\%} & 82.65\% & 84.55\% (Mah-NCM) \decrease{-2.99\%} \\
\midrule
\rowcolor{gray!20}
\multicolumn{4}{c}{\shortstack{Frozen image encoder (no detection fine-tuning)}} \\
\texttt{PE-Core-bigG-14-448} \citep{pe-core-bigG} & --- & 86.38\% & \textbf{94.46\%} (Mah-NCM) \\
\bottomrule
\end{tabular}
}
\end{table*}

\paragraph{Evaluation Protocol.}
For a training prior, detectors are evaluated on every dataset in the \Benchmark test suite.
We report a dataset-wise class accuracy that equals balanced accuracy when both classes are present and reduces to class-conditional accuracy on one-class sets.
Concretely, for an evaluation dataset $t$ with observed classes $\mathcal{C}_t \subseteq \{0,1\}$, we define Class Accuracy $\mathrm{CA}$ as

\begin{equation}
\mathrm{CA}^{(\phi)}(t) \;=\; \frac{1}{|\mathcal{C}_t|}\sum_{c\in \mathcal{C}_t} \Pr(\hat{y}^{\phi}=c \mid y=c),
\end{equation}

and report the macro-average $\frac{1}{|\mathcal{T}|}\sum_{t\in\mathcal{T}} \mathrm{CA}(t)$ across evaluation datasets $\mathcal{T}$ so that no single large dataset dominates.
For all detectors we use the default $\arg\max$ decision on predicted class probabilities and do not tune thresholds per dataset.
On real-only sets, $\mathrm{CA}$ corresponds to true-negative accuracy (1--FPR); on synthetic-only sets it corresponds to true-positive accuracy (1--FNR).
We additionally compute AUC on the \TotalMixedDatasets mixed datasets (those containing both real and synthetic images) and report in Appendix \Cref{sec:results_auc}.

\subsection{Closed-form Gaussian Discriminants versus Trained Heads}

\begin{table*}[!t]
\centering
\caption{
\textbf{Comparative Performance of Different Gaussian Discriminants.}
Mean class accuracy $\mathrm{CA}$ (macro-averaged across \Benchmark evaluation datasets) for each rung of the Gaussian discriminant ladder, conditioned on the same public training prior and fitted on the same frozen features as in \Cref{tab:overall}.
(Parentheses indicate the training subset reported by the original work; underlining denotes improvement over the corresponding released head as reported in \Cref{tab:overall}.)
}
\label{tab:comparative}
\resizebox{0.99\textwidth}{!}{%
\begin{tabular}{@{}lrrrr@{}}
\toprule
\textbf{Detection Model}
& \textbf{Cos-NCM} & \textbf{GNB} & \textbf{Mah-NCM} & \textbf{QDA}
\\
\Midrule
\rowcolor{gray!20}
\multicolumn{5}{c}{\shortstack{Trained with \CNNSpot \citep{wang2020cnngenerated} (\ProGAN Images based on \texttt{LSUN} Dataset)}}  \\
\UnivFD \citep{ojha2023universal} & 51.01\% & 48.44\% & \textbf{\underline{57.62\%}} & 48.23\% \\
\AIDE \citep{yan2024sanity}       & 54.36\% & \underline{56.69\%} & \textbf{\underline{62.76\%}} & \underline{61.28\%} \\
\midrule
\rowcolor{gray!20}
\multicolumn{5}{c}{\shortstack{Trained with \GenImage \citep{zhu2023genimage} (Diffusion Model Images based on \ImageNet Dataset)}} \\
\AIDE (\texttt{GenImage-SDv1}) \citep{yan2024sanity}               & \underline{54.23\%} & \underline{54.39\%} & \textbf{\underline{57.35\%}} & \underline{57.15\%} \\
\Effort (\texttt{GenImage-SDv1}) \citep{yan2025orthogonal}         & \underline{74.44\%} & \underline{74.92\%} & \textbf{\underline{76.82\%}} & \underline{73.21\%} \\
\texttt{DRCT-UnivFD} (Full \texttt{GenImage}) \citep{chen2024drct} & \underline{67.97\%} & \underline{68.28\%} & \textbf{\underline{73.01\%}} & \underline{69.74\%} \\
\AIDE (Full \texttt{GenImage}) \citep{yan2024sanity}               & 52.66\% & 53.92\% & \underline{\textbf{64.64\%}} & \underline{56.20\%} \\
\midrule
\rowcolor{gray!20}
\multicolumn{5}{c}{\shortstack{Trained with \DRCT \citep{chen2024drct} (Stable Diffusion Model Images based on \COCO Dataset)}} \\
\texttt{DRCT-UnivFD} (\texttt{DRCT-SDv1}) \citep{chen2024drct} & \underline{65.16\%} & \underline{66.37\%} & \underline{65.21\%} & \textbf{\underline{70.20\%}} \\
\texttt{DRCT-UnivFD} (\texttt{DRCT-SDv2}) \citep{chen2024drct} & \underline{65.53\%} & \underline{66.69\%} & \underline{67.91\%} & \textbf{\underline{70.31\%}} \\
\midrule
\rowcolor{gray!20}
\multicolumn{5}{c}{\shortstack{Trained with \ELSA \citep{baraldi2025contrasting} or \CF \citep{park2025community} \\ (Diffusion Model Images based on \LAION Dataset)}} \\
\texttt{CoDE-kNN} (\ELSA) \citep{baraldi2025contrasting} & 62.70\% & \underline{65.98\%} & \textbf{\underline{66.15\%}} & \underline{64.00\%} \\
\texttt{CF-224} (\CF) \citep{park2025community} & 78.98\% & 79.45\% & \textbf{\underline{82.61\%}} & 80.70\%\\
\texttt{CF-384} (\CF) \citep{park2025community} & 84.46\% & 80.01\% & 84.55\% & 79.52\%\\
\midrule
\rowcolor{gray!20}
\multicolumn{5}{c}{\shortstack{Frozen image encoder (no detection fine-tuning)}} \\
\texttt{PE-Core-bigG-14-448} \citep{pe-core-bigG} & 86.06\% & 87.26\% & \textbf{94.46\%} & 86.15\% \\
\bottomrule
\end{tabular}
}
\end{table*}

\Cref{tab:overall,tab:comparative} compares released AI-generated image detector heads with Gaussian discriminant rules fitted on the same frozen features and conditioned on the same training prior.
Several rows in \Cref{tab:overall,tab:comparative} share the same underlying backbone.
\UnivFD and \texttt{DRCT-UnivFD} both use \texttt{ViT-L-14-quickgelu-openai} \citep{vitl14} as the image encoder; differences arise from the public training prior and the trained head.
\texttt{DRCT} exposes classifier features after reconstruction training.
\Effort applies LoRA adaptation before its head.
\AIDE uses frequency/statistical experts plus semantic features, and \texttt{CoDE}/\texttt{CF} train backbones and heads end-to-end.
All are auditable once the encoder's final feature space is exposed.

Across several priors, the best closed-form Gaussian discriminant improves class accuracy, with the largest gains occurring when separation is well-captured by first- and second-order moments (often Mah-NCM).
These patterns are consistent with \Cref{prop:bayes_optimal,prop:homoscedastic}: when separation is well-captured by first- and second-order moments in a frozen encoder feature space, the corresponding Gaussian discriminants recover most of the available discriminative geometry.
A main exception is \texttt{CF-384} \citep{park2025community}, where the released detector remains ahead of the ladder (87.54\% vs.\ 84.55\% CA), consistent with a regime in which the trained head's decision surface exploits structure not explained by second-order moments alone.
Per-dataset results in \Cref{tab:individual_datasets,tab:auc_individual_datasets} also show large variability.
For example, Mah-NCM reaches near-ceiling performance on many synthetic-only sets but is much weaker on \texttt{DeepFakeBench} \citep{yan2023deepfakebench}, while QDA performs best on \texttt{CelebA-Spoof} \citep{zhang2020celebaspoof}.

\subsection{Encoder Sensitivity of Closed-form Gaussian Discriminants}

\begin{table*}[!t]
\centering
\caption{
\textbf{Encoder scaling under a fixed prior: representation choice strongly affects transfer.}
Mean class accuracy $\mathrm{CA}$ (macro-averaged across \Benchmark evaluation datasets) for Gaussian discriminant heads conditioned on the same \CF training prior \citep{park2025community}, using frozen CLIP encoders of increasing capacity pretrained by OpenAI.
Only the head is changed. No detector fine-tuning is performed.
}
\label{tab:architecture}
\resizebox{0.99\textwidth}{!}{%
\begin{tabular}{@{}lrrrrr@{}}
\toprule
\textbf{Frozen Encoders} &
\textbf{Euc-NCM \citep{wu2025fewshot}} & \textbf{Cos-NCM} & \textbf{GNB} & \textbf{Mah-NCM} & \textbf{QDA}
\\
\Midrule
\texttt{ResNet-50} \citep{resnet50}    & 65.27\% & 64.36\% & 65.72\% & \textbf{72.00\%} & 63.15\% \\
\texttt{ResNet-101} \citep{resnet101}  & 65.20\% & 64.70\% & 64.46\% & \textbf{70.27\%} & 62.29\% \\
\midrule
\texttt{ResNet-50x4} \citep{resnet50x4}   & 66.24\% & 65.91\% & 65.12\% & \textbf{72.34\%} & 64.28\% \\
\texttt{ResNet-50x16} \citep{resnet50x16} & 67.35\% & 66.77\% & 66.21\% & \textbf{74.98\%} & 65.71\% \\
\texttt{ResNet-50x64} \citep{resnet50x64} & 68.42\% & 68.19\% & 65.02\% & \textbf{74.57\%} & 68.38\% \\
\midrule
\texttt{ViT-B/16} \citep{vitb16} & 64.66\% & 64.65\% & 66.37\% & \textbf{74.18\%} & 65.70\% \\
\texttt{ViT-B/32} \citep{vitb32} & 64.53\% & 64.25\% & 64.83\% & \textbf{70.61\%} & 62.81\% \\
\texttt{ViT-L/14} \citep{vitl14} & 66.18\% & 68.44\% & 69.79\% & \textbf{76.96\%} & 68.87\% \\
\bottomrule
\end{tabular}
}
\end{table*}

\begin{table*}[!t]
\centering
\caption{
\textbf{Pretraining objective affects transfer even without detector training.}
Mean class accuracy $\mathrm{CA}$ (macro-averaged across \Benchmark evaluation datasets) for Gaussian discriminant heads conditioned on the same \CF training prior \citep{park2025community}, using frozen encoders trained with different objectives.
}
\label{tab:model}
\resizebox{0.99\textwidth}{!}{%
\begin{tabular}{@{}lrrrrr@{}}
\toprule
\textbf{Frozen Encoders} &
\textbf{Euc-NCM \citep{wu2025fewshot}} & \textbf{Cos-NCM} & \textbf{GNB} & \textbf{Mah-NCM} & \textbf{QDA}
\\
\Midrule
\texttt{MAE-Huge} \citep{mae-huge}                   & 55.18\% & 55.22\% & 58.33\% & \textbf{69.37\%} & 64.01\% \\
\texttt{BEiT-Large} \citep{beit-large}               & 58.53\% & 58.49\% & 58.80\% & \textbf{61.78\%} & 53.63\% \\
\texttt{CLIP-XLM-RoBERTa-Large} \citep{xlm-roberta-large} & 70.72\% & 69.70\% & 70.61\% & \textbf{77.67\%} & 65.79\% \\
\texttt{SigLIP-Large} \citep{siglip-large}           & 58.25\% & 58.25\% & 58.07\% & \textbf{64.14\%} & 60.85\% \\
\texttt{BLIP-Large} \citep{blip-large}               & 55.75\% & 56.89\% & 53.87\% & \textbf{70.17\%} & 55.42\% \\
\texttt{BLIP2} \citep{blip2}                         & 68.61\% & 68.46\% & 65.41\% & \textbf{77.73\%} & 62.21\% \\
\texttt{DINOv2-giant} \citep{dinov2-giant}           & 63.15\% & 63.59\% & 61.95\% & \textbf{73.30\%} & 58.52\% \\
\texttt{DINOv3-ViT-7b} \citep{dinov3-7b}             & 83.85\% & 85.82\% & 83.18\% & \textbf{88.91\%} & 78.88\% \\
\bottomrule
\end{tabular}
}
\end{table*}

We investigate the model-agnostic behavior of the ladder by evaluating a range of frozen architectures and pretraining objectives (\Cref{tab:architecture,tab:model}).
Mah-NCM is consistently the top performer across these encoders, indicating that a shared full-covariance geometry captures a stable component of real/fake separability in modern representation spaces.
This motivates a practical reporting guideline.
Trained heads on frozen encoders should be compared against the best Gaussian rung on the \emph{same} frozen features.
Additional models are presented in Appendix \Cref{sec:extended_results}.

\subsection{Training Prior Sensitivity of Closed-form Gaussian Discriminants}

\begin{table*}[!t]
\centering
\caption{
\textbf{Training prior dominates transfer even under a fixed encoder.}
Mean class accuracy $\mathrm{CA}$ (macro-averaged across \Benchmark evaluation datasets) for Gaussian discriminant heads when varying only the \emph{support prior} used to estimate moments, under the fixed frozen \texttt{PE-Core-bigG-14-448} encoder \citep{pe-core-bigG}. Underlining denotes improvement over the best released head of \texttt{CF-384} \citep{park2025community} as reported in \Cref{tab:overall}.
}
\label{tab:comparative_datasets}
\resizebox{0.99\textwidth}{!}{%
\begin{tabular}{@{}lrrrrr@{}}
\toprule
\textbf{Training Dataset} &
\textbf{Euc-NCM \citep{wu2025fewshot}} & \textbf{Cos-NCM} & \textbf{GNB} & \textbf{Mah-NCM} & \textbf{QDA}
\\
\Midrule
\CNNSpot \citep{wang2020cnngenerated}         & 74.95\% & 78.57\% & 59.21\% & \textbf{77.93\%} & 49.24\% \\
\GenImage \citep{zhu2023genimage}             & 86.68\% & \underline{88.75\%} & 87.50\% & \underline{\textbf{92.43\%}} & 81.83\% \\
\quad \texttt{GenImage-SDv1}                 & \underline{88.74\%} & \underline{89.80\%} & 86.34\% & \underline{\textbf{92.57\%}} & 65.06\% \\
\DRCT \citep{chen2024drct}                    & 81.90\% & 81.76\% & 86.21\% & 83.25\% & \underline{\textbf{91.09\%}} \\
\quad \texttt{DRCT-SDv1} \citep{chen2024drct} & 82.29\% & 83.81\% & \underline{88.70\%} & \underline{88.31\%} & \underline{\textbf{89.92\%}} \\
\quad \texttt{DRCT-SDv2} \citep{chen2024drct} & 83.50\% & 84.38\% & \underline{89.56\%} & \underline{90.08\%} & \underline{\textbf{91.44\%}} \\
\ELSA \citep{baraldi2025contrasting}          & \underline{88.87\%} & \underline{89.03\%} & 86.98\% & \underline{\textbf{94.45\%}} & 86.40\% \\
\CF \citep{park2025community}                 & 86.38\% & 86.06\% & 87.26\% & \underline{\textbf{94.46\%}} & 86.15\% \\
\bottomrule
\end{tabular}
}
\end{table*}

We isolate training-prior effects by fixing \texttt{PE-Core-bigG-14-448} \citep{pe-core-bigG} as the encoder and varying only the support prior used to estimate Gaussian parameters (\Cref{tab:comparative_datasets}).
The spread is large: the best CA ranges from 77.92\% (\CNNSpot) to 94.46\% (\CF), despite using the \emph{same} representation and evaluation suite.

Moreover, the prevailing covariance assumption shifts with the prior: \DRCT favors QDA (heteroscedastic) while all others favor Mah-NCM (homoscedastic).
Thus, if two papers use the same backbone but different priors, their OOD behavior can differ more than what is attributable to the classifier family within the Gaussian ladder.

\subsection{Data Efficiency of Closed-form Gaussian Discriminants}

\begin{table}[!t]
\centering
\caption{
\textbf{Data-efficiency of moment-based heads.}
Mean class accuracy $\mathrm{CA}$ (macro-averaged across \Benchmark evaluation datasets) for Gaussian discriminant heads as the support set used to estimate moments is subsampled from \CF \citep{park2025community}, under the fixed frozen \texttt{PE-Core-bigG-14-448} encoder \citep{pe-core-bigG}.
(Superscript/subscript denote max/min deviations (in percentage points) from the reported mean).
}
\label{tab:efficiency}
\resizebox{0.99\linewidth}{!}{%
\begin{tabular}{@{}lccccc@{}}
\toprule
\textbf{Amount} &
\textbf{Euc-NCM \citep{wu2025fewshot}} & \textbf{Cos-NCM} & \textbf{GNB} & \textbf{Mah-NCM} & \textbf{QDA} \\
\Midrule
0.001\% & 83.97\%$^{+1.93}_{-2.74}$ & 83.46\%$^{+2.25}_{-2.47}$ & \textbf{87.14\%}$^{+1.21}_{-1.68}$ & 85.56\%$^{+1.84}_{-1.75}$ & 85.13\%$^{+1.07}_{-0.66}$ \\[0.5em]
0.005\% & 86.04\%$^{+0.28}_{-0.40}$ & 85.78\%$^{+0.27}_{-0.53}$ & 87.23\%$^{+0.76}_{-0.77}$ & \textbf{90.89\%}$^{+0.73}_{-1.34}$ & 85.53\%$^{+1.51}_{-1.71}$ \\[0.5em]
0.01\%  & 86.32\%$^{+0.67}_{-0.35}$ & 86.02\%$^{+0.77}_{-0.42}$ & 87.42\%$^{+1.19}_{-0.90}$ & \textbf{91.55\%}$^{+0.67}_{-0.69}$ & 86.48\%$^{+1.16}_{-0.82}$ \\[0.5em]
0.1\%   & 86.26\%$^{+0.71}_{-0.44}$ & 85.92\%$^{+0.74}_{-0.47}$ & 87.33\%$^{+0.22}_{-0.19}$ & \textbf{93.18\%}$^{+0.42}_{-0.54}$ & 84.18\%$^{+0.33}_{-0.71}$ \\[0.5em]
1\%     & 86.39\%$^{+0.11}_{-0.12}$ & 86.05\%$^{+0.11}_{-0.10}$ & 87.28\%$^{+0.04}_{-0.05}$ & \textbf{93.52\%}$^{+0.28}_{-0.17}$ & 85.16\%$^{+0.20}_{-0.34}$ \\[0.5em]
\midrule
100\%   & 86.38\% & 86.06\% & 87.26\% & \textbf{94.46\%} & 86.15\% \\
\bottomrule
\end{tabular}%
}
\end{table}

We investigate data efficiency by subsampling the support set used to estimate Gaussian parameters (\Cref{tab:efficiency}).
Performance saturates quickly.
With only 1\% of the prior (around 43k samples), Mah-NCM nearly matches the full-support estimate, while at extremely small support sizes (0.001\% of the prior, or 44 samples) the diagonal GNB assumption can be more stable.
The pattern reinforces the central observation from the ladder.
For this support prior and encoder, low-order feature statistics capture a large share of the signal used by the best head, without iterative optimization.

With the frozen \texttt{PE-Core-bigG-14-448} encoder \citep{pe-core-bigG}, Mah-NCM fitted from only 0.005\% of the \CF prior (219 labeled samples) reaches 90.89\% CA.
Although this is not an apples-to-apples comparison to \texttt{CF-384} (which uses a different backbone), it indicates that moment-based adaptation can be data-efficient once the representation is strong.

\subsection{Representation Sensitivity of Wasserstein-2 Shift Estimates}
\label{sec:w2_sensitivity}

\begin{figure*}[t]
    \centering
    \includegraphics[width=0.49\linewidth]{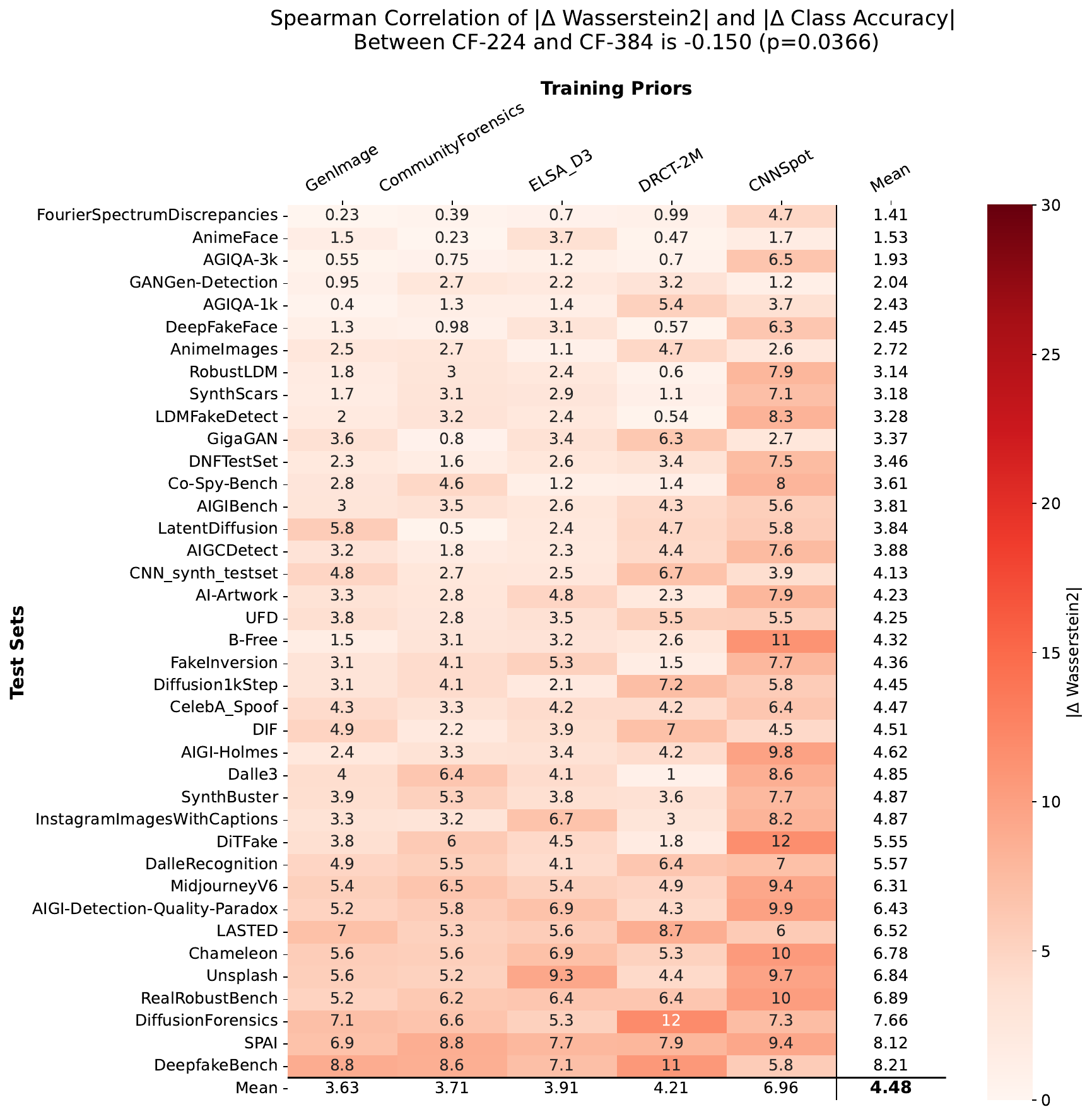}\hfill
    \includegraphics[width=0.49\linewidth]{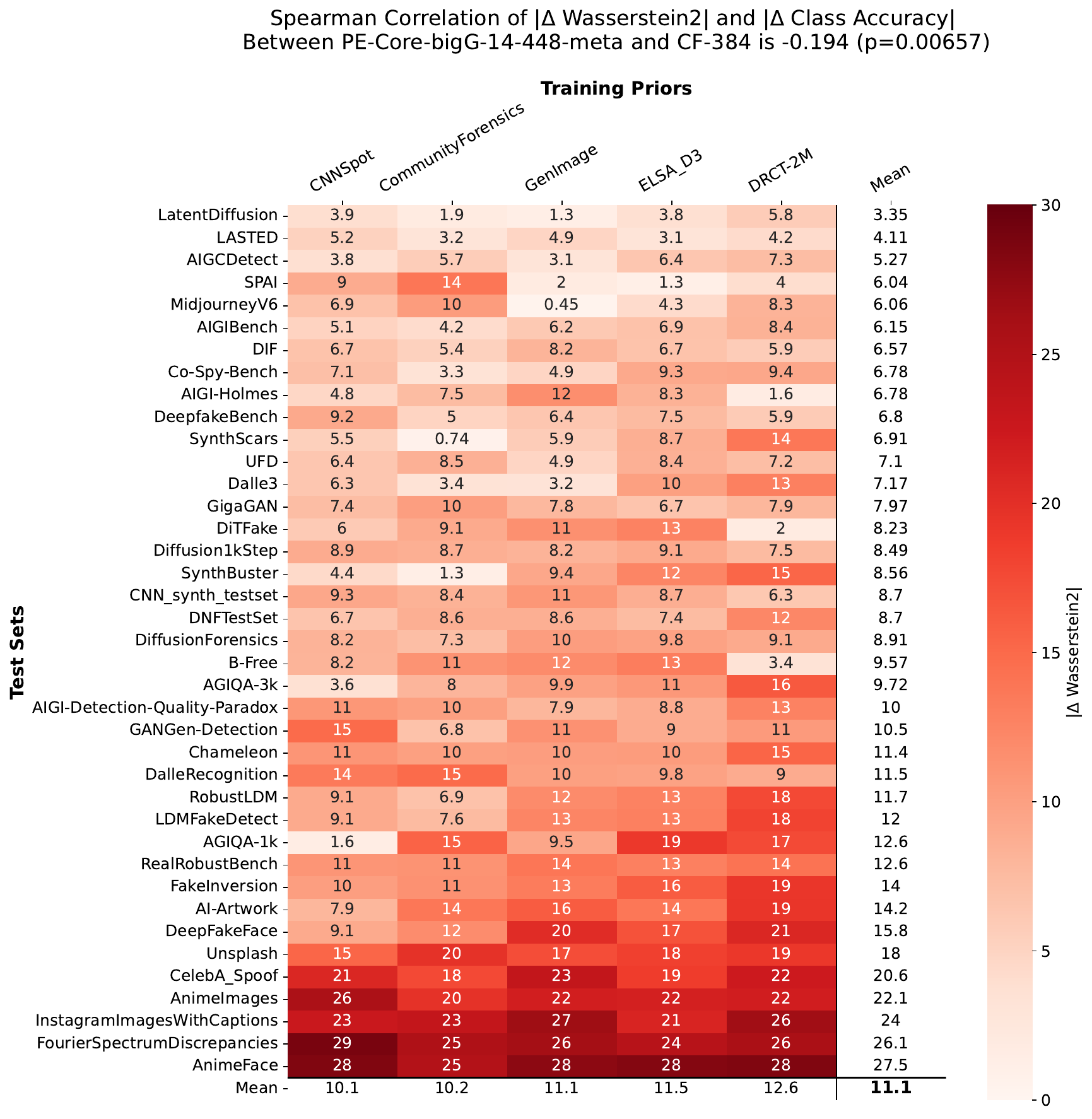}
    \caption{
    \textbf{Gaussian Wasserstein-2 distances are representation-conditioned.}
    Each cell reports the absolute difference in class-averaged Gaussian $\mathcal{W}_2$ for the same train--test pair under two feature spaces, $\Delta\mathcal{W}_2(\mathcal{D}_\text{train},\mathcal{D}_\text{test};\phi_a,\phi_b)$ (row/column means on the margins).
    The sub-figure title reports the Spearman correlation between the flattened cell-wise $\Delta\mathcal{W}_2$ values and the corresponding $\Delta\mathrm{CA}$ values.
    Large changes in $\mathcal{W}_2$ do not necessarily coincide with large changes in Mah-NCM CA, so feature-space shift magnitudes should not be interpreted as detector-agnostic indicators without checking stability across representations.
    }
    \label{fig:w2_sensitivity}
\end{figure*}

All geometric shift metrics in \Benchmark are computed in the encoder's feature space, hence their numerical values are not purely properties of the data distributions, but of the \emph{encoder} used to embed them.
This raises a basic reliability question about whether conclusions drawn from a Gaussian shift metric are stable under reasonable encoder changes, or are artifacts of a particular backbone.

\paragraph{Setup.}
For each training prior $\mathcal{D}_\text{train}$ and test set $\mathcal{D}_\text{test}$, class-conditional Gaussians are fitted to the representation induced by an encoder $\phi$, then a class-averaged Wasserstein-2 shift is computed between the training and test distributions:

\begin{equation}
\mathcal{W}_2^{(\phi)}(\mathcal{D}_\text{train},\mathcal{D}_\text{test})
~=~ \frac{1}{2}\sum_{c\in\{0,1\}}
\mathrm{W}_2\!\left(
\mathcal{N}(\mu_{c,\mathcal{D}_\text{train}}^{(\phi)},\Sigma_{c,\mathcal{D}_\text{train}}^{(\phi)}),
\mathcal{N}(\mu_{c,\mathcal{D}_\text{test}}^{(\phi)},\Sigma_{c,\mathcal{D}_\text{test}}^{(\phi)})
\right).
\end{equation}

To isolate representation effects while keeping the data fixed, \Cref{fig:w2_sensitivity} uses heatmap cells for the encoder-induced change in Gaussian $\mathcal{W}_2$ and panel titles for its Spearman correlation with the corresponding change in Mah-NCM accuracy.
Thus each panel asks whether the train--test pairs whose estimated shift changes most after an encoder swap are also the pairs whose detector accuracy changes most.
High correlation means the shift estimate is decision-aligned for that encoder pair.
Weak correlation means $\mathcal{W}_2$ is responding to feature-space changes that the detector does not use:

\begin{equation}
\Delta\mathcal{W}_2(\mathcal{D}_\text{train},\mathcal{D}_\text{test};\phi_a,\phi_b)
= \left|\mathcal{W}_2^{(\phi_a)}(\mathcal{D}_\text{train},\mathcal{D}_\text{test}) - \mathcal{W}_2^{(\phi_b)}(\mathcal{D}_\text{train},\mathcal{D}_\text{test})\right|,
\end{equation}

\begin{equation}
\Delta\mathrm{CA}(\mathcal{D}_\text{test};\phi_a,\phi_b)
~=~
\left|
\mathrm{CA}^{(\phi_a)}(\mathcal{D}_\text{test})
-
\mathrm{CA}^{(\phi_b)}(\mathcal{D}_\text{test})
\right|.
\end{equation}

\Cref{fig:w2_sensitivity} shows strong representation dependence.
Within a closely related encoder family, \texttt{CF-224} and \texttt{CF-384} \citep{park2025community} have a modest $\Delta\mathcal{W}_2$ on average.
Across larger backbone changes, such as \texttt{PE-Core-bigG-14-448} \citep{pe-core-bigG} versus \texttt{CF-384} \citep{park2025community}, $\Delta\mathcal{W}_2$ increases by an order of magnitude, especially for real-only datasets.
The sensitivity is also uneven.
Particular prior--test pairs exhibit consistently larger $\Delta\mathcal{W}_2$, indicating that some distribution shifts are more representation-fragile. % than others.

A weak association between $\Delta\mathcal{W}_2$ and $\Delta \mathrm{CA}$ implies that global distributional distances can change under encoder re-parameterizations that do not materially affect the \emph{discriminative} geometry used by Mah-NCM.
For example, anisotropic rescaling can inflate Wasserstein distances while leaving the effective decision geometry largely intact.
\Benchmark therefore measures generalization at the level of complete detection systems (training prior $\times$ encoder $\times$ classifier), and the representation axis is not interchangeable.
Any claim that uses $\mathcal{W}_2$ to argue for or against a training prior must therefore be stated as \emph{representation-conditioned}.
Robust conclusions should rely on patterns that persist across encoders (e.g., rank-consistent prior ordering or decision-aligned shift measures), rather than absolute $\mathcal{W}_2$ values in a single feature space.

\paragraph{Practical reporting checklist.}
These experiments suggest a minimal audit for future AIGI detector papers.
Authors should report the released head and the best Gaussian rung on the same frozen features and support prior, include per-dataset CA/AUC rather than only macro averages, state the encoder used for any geometric shift metric, and flag non-Gaussian feature regimes using diagnostics such as those in Appendix~\Cref{tab:univariate_normality_tests,tab:multivariate_normality_tests}.
This makes head-level OOD claims auditable without treating the Gaussian ladder as a replacement for specialized detectors.

\section{Conclusion}
\label{sec:conclusion}

This paper studies prior-conditioned Gaussian discriminants as practical baselines and diagnostics for AI-generated image detection under joint distribution shift.
Using a unified public-data protocol, the ladder is often competitive with trained detector heads under matched priors and encoders, and sometimes exceeds them.
We support a matched head audit plus controlled prior/encoder sweeps, not a fully symmetric decomposition of every detector component.
The source-side recommendation is to compare the trained head with the best Gaussian rung on the same representation and support prior before claiming head-level OOD gains.
Across priors, the winning covariance model indicates which low-order statistics transfer, and the data-efficiency results show that useful moment estimates can be obtained from small support sets.
Future work should characterize failures in multimodal or heavy-tailed feature regimes and design adaptation procedures that improve OOD performance while preserving transferable geometry.

\bibliographystyle{plainnat}
\bibliography{references}
\clearpage
\appendix

\section{On the Computational Paradigm and Practical Utility of Gaussian Discriminants}

Closed-form Gaussian heads differ from gradient-trained heads not only in statistical assumptions but also in optimization burden.
Fitting a rung of the ladder reduces to estimating sample means and (possibly regularized) covariances on a support set, followed by a single solve.
This procedure is deterministic and largely hyperparameter-light.
Aside from standard covariance regularization (e.g., shrinkage or diagonal loading), there is no learning-rate schedule, early stopping, or multi-run tuning.

The computational cost depends on the covariance structure.
Isotropic and diagonal heads require only per-dimension statistics ($O(ND)$).
Full-covariance heads (Mah-NCM/QDA) additionally estimate a $D\times D$ covariance ($O(ND^2)$) and invert or solve linear systems ($O(D^3)$), which can be non-trivial for very large $D$ or very large $N$.
However, this cost is paid once per (prior, encoder) pair and does not scale with the number of training epochs.

In contrast, gradient-based heads typically require many passes over the data and introduce additional variance from stochastic optimization and hyperparameter choices.
For our purposes, this simplicity is a feature.
The Gaussian ladder provides a controlled baseline and diagnostic, while its one-shot fitting reduces optimization confounds and makes head comparisons easier to reproduce.

\section{Proofs of Theoretical Propositions}
\label{sec:proof}
We provide formal derivations for the propositions presented in the theoretical framework \cref{sec:framework}.
We adopt the notation defined in the main text.

\subsection{Proof of \cref{prop:bayes_optimal} (Bayes-Optimal Classifier under Gaussian Assumption)}
\label{sec:proof_lda}

\begin{proposition}
If the class-conditional densities are Gaussian, $p(z \mid y=c) \sim \mathcal{N}(\mu_c, \Sigma_c)$, and the class priors are equal, $P(y=0)=P(y=1)=\frac{1}{2}$, the decision rule that minimizes the probability of error (the Bayes-optimal rule) is given by Quadratic Discriminant Analysis (QDA).
\end{proposition}

\begin{proof}
The Bayes-optimal decision rule minimizes the probability of misclassification by assigning a feature vector $z$ to the class $c$ with the maximum a posteriori (MAP) probability, $P(y=c|z)$.
For a binary classification task, this means that we assign $z$ to class 1 if $P(y=1|z) > P(y=0|z)$ and to class 0 otherwise.

\paragraph{Formulating the Decision Rule.}
Using Bayes' theorem, the posterior probability is $P(y=c|z) = \frac{p(z|y=c)P(y=c)}{p(z)}$.
The MAP decision rule is therefore:
\begin{equation}
\frac{p(z|y=1)P(y=1)}{p(z)} > \frac{p(z|y=0)P(y=0)}{p(z)}
\end{equation}

Since the evidence $p(z)$ is a positive common denominator and we assumed that the class priors are equal ($P(y=1) = P(y=0) = \frac{1}{2}$).
The decision rule simplifies to a comparison of the class-conditional likelihoods:
\begin{equation}
p(z|y=1) > p(z|y=0)
\end{equation}

We define the decision function $g^*(z)$ as the log-likelihood ratio: %, As the logarithm is a strictly monotonic function, the inequality is preserved if we take the logarithm of both sides:
\begin{equation}
g^*(z) = \log p(z|y=1) - \log p(z|y=0)
\label{eq:decision}
\end{equation}
The decision rule using the log-likelihood ratio is to classify as class 1 if $g^*(z) > 0$ and class 0 otherwise.

\paragraph{Introducing Gaussian PDF assumption to decision rule.}
The probability density function (PDF) for a multivariate Gaussian distribution is:
\begin{equation}
p(z|y=c) = \mathcal{N}(z; \mu_c, \Sigma_c) = \frac{1}{(2\pi)^{D/2}|\Sigma_c|^{1/2}} \exp\left(-\frac{1}{2}(z-\mu_c)^\top \Sigma_c^{-1}(z-\mu_c)\right)
\end{equation}
The corresponding log-likelihood for class $c$ is:
\begin{equation}
\log p(z|y=c) = -\frac{D}{2}\log(2\pi) - \frac{1}{2}\log|\Sigma_c| - \frac{1}{2}(z-\mu_c)^\top \Sigma_c^{-1}(z-\mu_c)
\end{equation}

Using this expression in the decision function $g^*(z)$ \cref{eq:decision}, we convert our decision rule into an approximation using Gaussian parameters after simplification.
\begin{equation}
g^*(z) \propto (z-\mu_0)^\top \Sigma_0^{-1}(z-\mu_0) - (z-\mu_1)^\top \Sigma_1^{-1}(z-\mu_1) + \log\frac{|\Sigma_0|}{|\Sigma_1|}
\label{eq:qda}
\end{equation}

This is the discriminant function for Quadratic Discriminant Analysis (QDA), which is therefore Bayes-optimal under the stated Gaussian assumption.

\end{proof}

\subsection{Proof of \cref{prop:homoscedastic} (Optimality under Homoscedasticity)}
\label{sec:proof_mah}

\begin{proposition}
If the classes are additionally assumed to be homoscedastic (i.e., they share a common covariance matrix $\Sigma_0 = \Sigma_1 = \Sigma_p$), the decision rule that minimizes the probability of error (the Bayes-optimal rule) is given by Mahalanobis Nearest Centroid Matching (Mah-NCM).
\end{proposition}

\begin{proof}
We begin with the Bayes-optimal QDA decision function derived above \cref{eq:qda} and apply the homoscedasticity assumption $\Sigma_0 = \Sigma_1 = \Sigma_p$ to the decision function \cref{eq:qda} and after simplifications of the quadratic forms, the decision rule becomes of the form:
\begin{align}
g^*(z) & \propto (z-\mu_0)^\top \Sigma_p^{-1}(z-\mu_0) - (z-\mu_1)^\top \Sigma_p^{-1}(z-\mu_1) \\
& \propto \left( z^\top\Sigma_p^{-1}z - 2\mu_0^\top\Sigma_p^{-1}z + \mu_0^\top\Sigma_p^{-1}\mu_0 \right) - \left( z^\top\Sigma_p^{-1}z - 2\mu_1^\top\Sigma_p^{-1}z + \mu_1^\top\Sigma_p^{-1}\mu_1 \right) \\
& \propto 2\mu_1^\top\Sigma_p^{-1}z - 2\mu_0^\top\Sigma_p^{-1}z + \mu_0^\top\Sigma_p^{-1}\mu_0 - \mu_1^\top\Sigma_p^{-1}\mu_1 \\
& \propto 2(\mu_1 - \mu_0)^\top\Sigma_p^{-1}z - (\mu_1^\top\Sigma_p^{-1}\mu_1 - \mu_0^\top\Sigma_p^{-1}\mu_0) \\
& \propto 2(\mu_1 - \mu_0)^\top\Sigma_p^{-1}z - (\mu_1 - \mu_0)^\top\Sigma_p^{-1}(\mu_1 + \mu_0)
 \label{eq:mah}
\end{align}

Note that this assumption of homoscedasticity $\Sigma_0 = \Sigma_1 = \Sigma_p$ linearizes the decision boundary:
\begin{align}
g^*(z) \propto & w^\top z + b \\
\label{eq:mah-ncm}
\text{where }
w &= \Sigma_p^{-1}(\mu_1 - \mu_0) \\
b &= -\frac{1}{2}(\mu_1 + \mu_0)^\top w = -\frac{1}{2}(\mu_1 + \mu_0)^\top\Sigma_p^{-1}(\mu_1 - \mu_0)
\end{align}

This is the discriminant function for Mahalanobis Nearest Centroid Matching (Mah-NCM), which is therefore Bayes-optimal under the stated homoscedastic Gaussian assumption.

\end{proof}

\subsection{Proof Sketch for \cref{prop:predictive_sufficiency} (Performance Stability under Bounded Drift)}

\begin{proposition}
\label{sec:proof_performance}
Assume a homoscedastic classifier with shared covariance $\Sigma$ and equal priors. If the statistical drift between the training and testing distributions is bounded by $\|\mu_c^{\text{test}} - \mu_c^{\text{train}}\|_2 \le \epsilon_\mu$ and $\|\Sigma^{\text{test}} - \Sigma^{\text{train}}\|_F \le \epsilon_\Sigma$, then $|\mathrm{AUC}_{\text{test}} - \mathrm{AUC}_{\text{train}}| \le L(\epsilon_\mu + \epsilon_\Sigma)$.
\end{proposition}

\begin{proof}[Proof Sketch]
A full proof requires extensive details from perturbation theory.
We provide a rigorous sketch outlining the main logical steps.

\paragraph{Expressing AUC as a function of statistical parameters.}
For a linear classifier with weight vector $w$, the AUC is the probability that a randomly drawn positive sample scores higher than a randomly drawn negative sample: $\text{AUC} = P(w^\top Z_1 > w^\top Z_0)$, where $Z_c \sim \mathcal{N}(\mu_c, \Sigma)$.
Let the score difference be the random variable $Y = w^\top(Z_1 - Z_0)$. Since $Z_1$ and $Z_0$ are independent Gaussian variables, their difference is also Gaussian: $Z_1 - Z_0 \sim \mathcal{N}(\mu_1 - \mu_0, 2\Sigma)$. Therefore, $Y$ is a scalar Gaussian with mean $\mathbb{E}[Y] = w^\top(\mu_1 - \mu_0)$ and variance $\text{Var}(Y) = w^\top(2\Sigma)w$.
The AUC is $P(Y > 0)$, which can be expressed using the CDF of the standard normal distribution, $\Phi$:
\begin{equation}
\text{AUC} = P\left(\frac{Y - \mathbb{E}[Y]}{\sqrt{\text{Var}(Y)}} > \frac{-\mathbb{E}[Y]}{\sqrt{\text{Var}(Y)}}\right) = \Phi\left(\frac{\mathbb{E}[Y]}{\sqrt{\text{Var}(Y)}}\right)
\end{equation}
Under the homoscedastic Gaussian assumption, for the optimal linear classifier Mahalanobis Nearest Centroid Matching (Mah-NCM) where $w = \Sigma^{-1}(\mu_1-\mu_0)$, this simplifies to:
\begin{equation}
\text{AUC} = \Phi\left(\frac{(\mu_1-\mu_0)^\top\Sigma^{-1}(\mu_1-\mu_0)}{\sqrt{2(\mu_1-\mu_0)^\top\Sigma^{-1}\Sigma\Sigma^{-1}(\mu_1-\mu_0)}}\right) = \Phi\left(\sqrt{\frac{\mathcal{D}_\mu}{2}}\right)
\end{equation}
This establishes that the AUC is a direct and smooth function of the Fisher margin $\mathcal{D}_\mu$.

\paragraph{Establishing Lipschitz Continuity.}
It remains to show that $\text{AUC}$ is locally Lipschitz.
The map is a composition of three parts:
\begin{enumerate}
    \item The Fisher Margin $\mathcal{D}_\mu = (\mu_1-\mu_0)^\top\Sigma^{-1}(\mu_1-\mu_0)$ is a smooth (infinitely differentiable) function of its arguments as long as $\Sigma$ is invertible. (Note that all the training datasets we used in the experiments made this assumption of invertible $\Sigma$ valid by ensuring $n>>d$, future works can explore the shrinkage $\Sigma$ using Ledoit-Wolf Shrinkage when training sets are of form $n<<d$, ensuring $\Sigma$ is invertible).
    \item The square root function, which is locally Lipschitz in $(0, \infty)$.
    \item The standard normal CDF $\Phi(x)$, which is globally Lipschitz because its derivative (the normal PDF) is bounded by $1/\sqrt{2\pi}$.
\end{enumerate}
Since the composition of local Lipschitz functions is locally Lipschitz, the overall function $AUC$ is locally Lipschitz.

\paragraph{Deriving the bounds on deviation.}
By the definition of local Lipschitz continuity, for small perturbations in the arguments, the change in the function's value is bounded by a constant times the magnitude of the perturbation.
Therefore, for perturbations bounded by $\epsilon_\mu$ and $\epsilon_\Sigma$, we have
\begin{equation}
    |\mathrm{AUC}_{\text{test}} - \mathrm{AUC}_{\text{train}}| = |H(\mu^{\text{test}}, \Sigma^{\text{test}}) - H(\mu^{\text{train}}, \Sigma^{\text{train}})| \le L \cdot (\|\mu^{\text{test}} - \mu^{\text{train}}\| + \|\Sigma^{\text{test}} - \Sigma^{\text{train}}\|)
\end{equation}
This can be expressed as $|\mathrm{AUC}_{\text{test}} - \mathrm{AUC}_{\text{train}}| \le L_1 \epsilon_\mu + L_2 \epsilon_\Sigma \le L(\epsilon_\mu + \epsilon_\Sigma)$, where $L = \max(L_1, L_2)$. The Lipschitz constant $L$ depends on the local derivatives of $H$, which are functions of the training parameters, critically including the spectral norm of the inverse covariance, $\|\Sigma^{-1}\|_2$.
\end{proof}

\subsection{Proof of \cref{prop:stability} (Stability of the Fisher Margin under Distributional Drift)}
\label{sec:proof_stability}
\begin{proposition}
Assume that both the training and testing domains are characterized by homoscedastic Gaussian parameters $(\mu_c^{\text{train}}, \Sigma_p^{\text{train}})$ and $(\mu_c^{\text{test}}, \Sigma_p^{\text{test}})$, respectively.
If the distributional drift is bounded such that $\|\delta^{\text{test}}-\delta^{\text{train}}\|_2 \le \epsilon$ and $\|\Sigma_p^{\text{test}}-\Sigma_p^{\text{train}}\|_F \le \eta$, where $\delta = \mu_1 - \mu_0$, then the absolute change in the Fisher Margin $\mathcal{D}_\mu$ is bounded, to a first order, by $|\mathcal{D}_\mu^{\text{test}} - \mathcal{D}_\mu^{\text{train}}| \le 2 \|(\Sigma_p^{\text{train}})^{-1}\|_2 \|\delta^{\text{train}}\|_2 \epsilon + O(\eta)$.
\end{proposition}

\begin{proof}
We analyze the change in Fisher Margin $\mathcal{D}_\mu = \delta^\top \Sigma_p^{-1} \delta$ under small perturbations.
We define perturbations as $\Delta\delta = \delta^{\text{test}} - \delta^{\text{train}}$ and $\Delta\Sigma_p = \Sigma_p^{\text{test}} - \Sigma_p^{\text{train}}$, with norms bounded by $\|\Delta\delta\|_2 \le \epsilon$ and $\|\Delta\Sigma_p\|_F \le \eta$.

\paragraph{First-Order Taylor Expansion.}
On expanding $\mathcal{D}_\mu$ around the training parameters $(\delta^{\text{train}}, \Sigma_p^{\text{train}})$, in first order, the change is given by the total derivative:
\begin{equation}
\Delta\mathcal{D}_\mu = \mathcal{D}_\mu^{\text{test}} - \mathcal{D}_\mu^{\text{train}} \approx \left(\nabla_\delta \mathcal{D}_\mu\right)^\top \Delta\delta + \langle \nabla_{\Sigma_p} \mathcal{D}_\mu, \Delta\Sigma_p \rangle_F
\label{eq:taylor}
\end{equation}
where the gradients are evaluated at the training parameters.

\paragraph{Gradient with respect to $\delta$.}
\begin{equation}
\nabla_\delta \mathcal{D}_\mu = 2 \Sigma_p^{-1} \delta
\end{equation}
\paragraph{Gradient with respect to $\Sigma_p$.}
\begin{align}
d\mathcal{D}_\mu &= (d\delta^\top)\Sigma_p^{-1}\delta + \delta^\top\Sigma_p^{-1}(d\delta) + \delta^\top (d\Sigma_p^{-1}) \delta \\
&= 2\delta^\top\Sigma_p^{-1}d\delta - \delta^\top \Sigma_p^{-1}(d\Sigma_p)\Sigma_p^{-1}\delta \\
&= \text{Tr}\left(2\delta^\top\Sigma_p^{-1}d\delta\right) - \text{Tr}\left(\delta^\top \Sigma_p^{-1}(d\Sigma_p)\Sigma_p^{-1}\delta\right) \\
&= \text{Tr}\left(2\delta^\top\Sigma_p^{-1}d\delta\right) - \text{Tr}\left(\Sigma_p^{-1}\delta\delta^\top\Sigma_p^{-1} d\Sigma_p\right)
\end{align}
From the definition of the Frobenius inner product, $\langle A, B \rangle_F = \text{Tr}(A^\top B)$, the gradient with respect to $\Sigma_p$ is:
\begin{equation}
\nabla_{\Sigma_p} \mathcal{D}_\mu = -\left(\Sigma_p^{-1} \delta \delta^\top \Sigma_p^{-1}\right)^\top = -\Sigma_p^{-1} \delta \delta^\top \Sigma_p^{-1}
\end{equation}
since $\Sigma_p$ is symmetric.

\paragraph{Deriving the bound on the magnitude of the change.}
Using the triangle inequality on the first-order expansion \cref{eq:taylor}:
\begin{equation}
|\Delta\mathcal{D}_\mu| \le \left|\left(\nabla_\delta \mathcal{D}_\mu\right)^\top \Delta\delta\right| + \left|\langle \nabla_{\Sigma_p} \mathcal{D}_\mu, \Delta\Sigma_p \rangle_F\right|
\end{equation}
We bound each term separately:

\paragraph{Mean Drift Term.}
By the Cauchy-Schwarz inequality:
\begin{align*}
\left|\left(2 (\Sigma_p^{\text{train}})^{-1} \delta^{\text{train}}\right)^\top \Delta\delta\right| &\le \|2 (\Sigma_p^{\text{train}})^{-1} \delta^{\text{train}}\|_2 \|\Delta\delta\|_2 \\
&\le 2 \|(\Sigma_p^{\text{train}})^{-1}\|_2 \|\delta^{\text{train}}\|_2 \epsilon
\end{align*}
where $\|\cdot\|_2$ is the spectral norm for matrices and the Euclidean norm for vectors.

\paragraph{Covariance Drift Term.}
Using the property of the Frobenius inner product, $|\langle A, B \rangle_F| \le \|A\|_F \|B\|_F$:
\begin{align*}
\left|\langle \nabla_{\Sigma_p} \mathcal{D}_\mu, \Delta\Sigma_p \rangle_F\right| &\le \|\nabla_{\Sigma_p} \mathcal{D}_\mu\|_F \|\Delta\Sigma_p\|_F \\
&\le C \cdot \eta = O(\eta)
\end{align*}
where $C = \|(\Sigma_p^{\text{train}})^{-1} \delta^{\text{train}} (\delta^{\text{train}})^\top (\Sigma_p^{\text{train}})^{-1}\|_F$ is a constant determined by the training distribution.

Combining the bounds for both terms yields:
\begin{equation}
    \left|\mathcal{D}_\mu^{\text{test}} - \mathcal{D}_\mu^{\text{train}}\right| \le 2 \|(\Sigma_p^{\text{train}})^{-1}\|_2 \|\delta^{\text{train}}\|_2 \epsilon + O(\eta)
\end{equation}
This shows that the change in the Fisher Margin $\mathcal{D}_\mu$ is controlled, to a first order, by the magnitude of the drift in the class statistics.
\end{proof}

\section{Analytical Forms of the Gaussian Discriminant Ladder}
\label{sec:analytical_classifiers}

For each rung, we estimate class means $\mu_c$ and (optionally) covariance matrices from a labeled \emph{support} set in the training prior.
Given a test feature $z$, the heads produce unnormalized class scores $g_c(z)$ and predict $\hat{y}(z)=\arg\max_{c\in\{0,1\}} g_c(z)$.
The score functions used in \Cref{tab:compare} are listed below. (Additive constants shared across classes are omitted.)

\paragraph{Class priors.}
To avoid confounding head comparisons with class-imbalance in the support prior, we use \emph{uniform} class priors by default ($\pi_0=\pi_1=\tfrac{1}{2}$), so any $\log\pi_c$ terms are class-independent and may be dropped.
When we report results with empirical priors, $\pi_c$ denotes the support-set class frequency and is stated explicitly.
Balanced evaluation metrics (e.g., balanced accuracy) do \emph{not} in general remove prior-induced offsets in the decision rule.

\paragraph{Euc-NCM (isotropic, shared covariance).}
\begin{equation}
g_c(z) = -\|z-\mu_c\|_2^2 .
\end{equation}

\paragraph{Cos-NCM (cosine similarity).}
\begin{equation}
g_c(z) = \left\langle \frac{z}{\|z\|_2},\,\frac{\mu_c}{\|\mu_c\|_2}\right\rangle .
\end{equation}

\paragraph{GNB (diagonal covariance).}
Let $\Sigma_c=\mathrm{diag}(\sigma_c^2)$.
\begin{equation}
g_c(z)= -\tfrac{1}{2}\sum_{j=1}^d \left(\frac{(z_j-\mu_{c,j})^2}{\sigma_{c,j}^2} + \log \sigma_{c,j}^2 \right) + \log \pi_c .
\end{equation}

\paragraph{Mah-NCM / LDA (shared full covariance).}
Let $\Sigma_0=\Sigma_1=\Sigma$.
\begin{equation}
g_c(z)= z^\top \Sigma^{-1}\mu_c - \tfrac{1}{2}\mu_c^\top \Sigma^{-1}\mu_c + \log\pi_c.
\end{equation}

\paragraph{QDA (class-specific full covariance).}
\begin{equation}
g_c(z)= -\tfrac{1}{2}(z-\mu_c)^\top \Sigma_c^{-1}(z-\mu_c) - \tfrac{1}{2}\log\det(\Sigma_c) + \log \pi_c.
\end{equation}
(With the default uniform priors used in this paper, $\log\pi_c$ is constant across $c$ and is omitted in implementation.)

\section{Details about existing datasets in our \Benchmark evaluation suite}
\label{sec:extended_setup}

The following tables list the datasets used for training (\cref{tab:summarytrain}) and evaluation (\cref{tab:summaryval}) to make the experimental setup reproducible.
Percept-Lens uses public datasets, leaves support priors out of evaluation, and combines generator style, source-domain, and post-processing shifts rather than a single generator-only shift.

\paragraph{Covariance estimation and numerical stability.}
For full-covariance rules (Mah-NCM/QDA), covariance inversion in high-dimensional feature spaces can be ill-conditioned, especially in low-shot regimes.
We therefore regularize covariance estimates before computing $\Sigma^{-1}$ and $\log\det(\Sigma)$.
Concretely, we use an empirical covariance estimate with diagonal loading,
$\widehat{\Sigma} \leftarrow \widehat{\Sigma} + \varepsilon I$,
and compute inverses via a numerically stable pseudo-inverse (small singular values truncated) to avoid catastrophic failures when $\widehat{\Sigma}$ is nearly singular.
Log-determinants are computed from the (clipped) spectrum of $\widehat{\Sigma}$ to guarantee finite values.
All priors and encoders are evaluated with the \emph{same} estimator and regularization hyperparameters to avoid confounding.

We use standard regularization (shrinkage estimators such as Ledoit--Wolf, plus diagonal loading $\epsilon I$) before computing $\Sigma^{-1}$ and log-determinants, when computing the covariances for data-efficiency results in \Cref{tab:efficiency,tab:auc_efficiency}.
This is because under extreme data shrinkage, we enter the $N<D$ regime (number of samples smaller than feature dimension), where empirical covariance estimates are unstable.
All priors and backbones are evaluated with the same estimator to avoid confounding.

\begin{table*}[!t]
\fontsize{11.5}{14.25}\selectfont
\centering
\caption{
\textbf{Summary of training priors used for conditioning.}
We list the public datasets used as training priors in our experiments and report the number of real and synthetic images available in the version used by our pipeline.
\ding{168} indicates that counts correspond to the publicly available subset used in our preprocessing and may differ from the full dataset described in the original source due to corrupted images in the original source.
}\label{tab:summarytrain}
\begin{tabular}{lrrr}
\toprule
\textbf{Training Dataset}
& \textbf{Total Images} & \textbf{Real} & \textbf{Synthetic}
\\
\Midrule
\CNNSpot \citep{wang2020cnngenerated}                &   720,119 &   360,059 &   360,060 \\
\GenImage ~\ding{168} \citep{zhu2023genimage}        & 2,254,762 & 1,116,779 & 1,137,983 \\
\quad \texttt{GenImage-SDv1} \citep{zhu2023genimage} &   323,997 &   162,000 &   161,997 \\
\DRCT ~\ding{168} \citep{chen2024drct}               & 2,247,453 &   118,287 & 2,129,166 \\
\quad \texttt{DRCT-SDv1} \citep{chen2024drct}        &   473,148 &   118,287 &   354,861 \\
\quad \texttt{DRCT-SDv2} \citep{chen2024drct}        &   473,148 &   118,287 &   354,861 \\
\CF ~\ding{168} \citep{park2025community}            & 4,386,820 & 1,678,386 & 2,708,434 \\
\ELSA ~\ding{168} \citep{baraldi2025contrasting}     & 8,417,550 & 1,683,511 & 6,734,039 \\
\bottomrule
\end{tabular}
\end{table*}

\FloatBarrier
\begingroup
\fontsize{9}{11.75}\selectfont
\setlength{\tabcolsep}{4pt}
\centering
\begin{longtable}{@{}>{\raggedright\arraybackslash}p{\dimexpr0.52\linewidth-1.5\tabcolsep\relax}>{\raggedleft\arraybackslash}p{\dimexpr0.16\linewidth-1.5\tabcolsep\relax}>{\raggedleft\arraybackslash}p{\dimexpr0.16\linewidth-1.5\tabcolsep\relax}>{\raggedleft\arraybackslash}p{\dimexpr0.16\linewidth-1.5\tabcolsep\relax}@{}}
\caption{
\textbf{Summary of \Benchmark evaluation datasets.}
We list the public datasets used to evaluate generalization and report the number of real and synthetic images.
\ding{169} indicates that we evaluate on all officially released subsets/variants provided by the source; when a dataset provides multiple transformations of the \emph{same} underlying image, we treat each transformation as a separate evaluation subset and do not interpret the resulting counts as independent samples.
}
\label{tab:summaryval} \\
\\
\toprule
\textbf{Evaluation Dataset} & \textbf{Total Images} & \textbf{Real} & \textbf{Synthetic} \\
\Midrule
\endfirsthead

\toprule
\textbf{Evaluation Dataset} & \textbf{Total Images} & \textbf{Real} & \textbf{Synthetic} \\
\Midrule
\endhead

\midrule
\multicolumn{4}{r}{\footnotesize Continued on next page} \\
\endfoot

\bottomrule
\endlastfoot
\rowcolor{gray!20}
\multicolumn{4}{c}{\textbf{Out-of-Distribution Existing Datasets with only Real Images}} \\
\texttt{Unsplash (Lite Subset)} \citep{unsplash}               &     24,969 &     24,969 & --- \\
\texttt{InstagramImagesWithCaptions} \citep{instagrama}        &     34,927 &     34,927 & --- \\
\texttt{Anime Faces Dataset} \citep{animeface}                 &     63,565 &     63,565 & --- \\
\texttt{Anime Images} \citep{animeimages}                      &     82,975 &     82,975 & --- \\
\texttt{CelebA-Spoof} \ding{169} \citep{zhang2020celebaspoof}  &    561,575 &    561,575 & --- \\
\midrule
\rowcolor{gray!20}
\multicolumn{4}{c}{\textbf{Out-of-Distribution Existing Datasets with only Synthetic Images}} \\
\texttt{AGIQA-1k} \citep{zhang2023perceptual}        &       1,080 & --- &      1,080 \\
\texttt{AGIQA-3k} \citep{li2024agiqa3k}              &       2,982 & --- &      2,982 \\
\texttt{SPAI} \citep{karageorgiou2025anyresolution}  &       3,638 & --- &      3,638 \\
\texttt{SynthBuster Extended} \citep{bammey2024synthbuster,guillaro2025biasfree}   &      11,003 & --- &     11,003 \\
\texttt{SynthScars} \citep{kang2025legion}           &      12,182 & --- &     12,182 \\
\texttt{GigaGAN} \citep{kang2023scaling}             &     170,000 & --- &    170,000 \\
\texttt{LatentDiffusion} \citep{corvi2023detection}  &     216,000 & --- &    216,000 \\
\texttt{MidJourneyV6} \citep{terminusresearch}       &     519,849 & --- &    519,849 \\
\texttt{Co-Spy-Bench} \citep{cheng2025cospy}         &     550,000 & --- &    550,000 \\
\texttt{Dalle3} \citep{Egan_Dalle3_1_Million_2024}   &   1,193,805 & --- &  1,193,805 \\
\midrule
\rowcolor{gray!20}
\multicolumn{4}{c}{\textbf{Out-of-Distribution Existing Datasets with both Real and Synthetic Images}} \\
\texttt{FourierSpectrumDiscrepancies} \citep{dzanic2020fourier} &        90 &      15 &      75 \\
\texttt{FakeInversion} \citep{cazenavette2024fakeinversion} &     1,300 &     650 &     650 \\
\texttt{UniversalFakeDetect} \citep{ojha2023universal}      &    10,000 &   2,000 &   8,000 \\
\texttt{Dalle Recognition Dataset} \citep{airecognition}    &    21,635 &   3,780 &  17,855 \\
\texttt{Chameleon} \citep{yan2024sanity}                    &    26,033 &  14,863 &  11,170 \\
\texttt{AIGI-Detection-Quality-Paradox} \citep{xiao2025are} &    27,864 &   3,864 &  24,000 \\
\texttt{DiTFake} \citep{li2025improving}                    &    30,000 &  15,000 &  15,000 \\
\texttt{Diffusion1kSteps} \citep{tan2024rethinking}         &    35,992 &  18,000 &  17,992 \\
\texttt{GANGen-Detection} \citep{chuangchuangtan-GANGen-Detection} &    36,000 &  18,000 &  18,000 \\
\texttt{RobustLDM} \citep{rajan2024aligned}                 &    42,752 &   6,000 &  36,752 \\
\texttt{RealRobustBench} \citep{li2025bridging}             &    53,999 &  26,999 &  27,000 \\
\texttt{LDMFakeDetect} \citep{rajan2025staypositive}        &    61,352 &   6,000 &  55,352 \\
\texttt{DIF} \citep{sinitsa2024deep}                        &    75,344 &  37,672 &  37,672 \\
\texttt{ForenSynths} \citep{wang2020cnngenerated}           &    90,329 &  45,169 &  45,160 \\
\texttt{DNF-TestSet} \citep{zhang2025diffusion}             &    91,267 &   7,000 &  84,267 \\
\texttt{DeepFakeFace} \citep{song2023robustness}            &   120,000 &  30,000 &  90,000 \\
\texttt{AIGCDetectBench} \citep{zhong2024patchcraft}        &   152,597 &  76,298 &  76,299 \\
\texttt{AIGI-Holmes} \citep{zhou2025aigiholmes}             &   164,996 &  83,850 &  81,146 \\
\texttt{AI-Artwork} \citep{aiartwork}                       &   271,993 &  81,444 & 190,549 \\
\texttt{DiffusionForensics} \citep{wang2023dire}            &   313,368 &  86,000 & 227,368 \\
\texttt{B-Free} \citep{guillaro2025biasfree}                &   361,584 &  52,482 & 309,102 \\
\texttt{LASTED} \citep{wu2025generalizable}                 &   367,533 & 136,287 & 231,246 \\
\texttt{AIGIBench} \citep{li2025artificial}                 &   520,826 & 260,514 & 260,312 \\
\texttt{DeepFakeBench} \citep{yan2023deepfakebench}         &   803,991 &  70,998 & 732,993 \\
\midrule
\rowcolor{gray!30}
\textbf{Percept-Lens Evaluation Suite} & \textbf{7,129,395} & \textbf{1,850,896} & \textbf{5,278,499} \\
\end{longtable}
\endgroup

\section{Extended Investigation of Generalization with \Benchmark suite}
\label{sec:extended_results}

\begin{table*}[!t]
\setlength{\tabcolsep}{4pt}
\centering
\caption{
\textbf{Closed-form Gaussian discriminants versus released AI-generated image detector heads under matched (prior, encoder) conditions.}
Mean class accuracy $\mathrm{CA}$ (dataset-wise balanced accuracy where defined; \Cref{sec:benchmark}), macro-averaged across \Benchmark evaluation datasets.
For each released detector checkpoint, we report (i) the released decision head (\emph{Out-of-the-Shelf}) and (ii) closed-form baselines fitted on the \emph{same frozen encoder features} using the corresponding public training prior: Euc-NCM \citep{wu2025fewshot} and the Gaussian ladder.
(Parentheses indicate the training subset reported by the original work when applicable.)
}
\label{tab:extended_comparative}
\resizebox{0.99\textwidth}{!}{%
\begin{tabular}{@{}lrrrrrr@{}}
\toprule
\textbf{Detection Model}
& \textbf{Out-of-the-Shelf}
& \textbf{Euc-NCM \citep{wu2025fewshot}}
& \textbf{Cos-NCM} & \textbf{GNB} & \textbf{Mah-NCM} & \textbf{QDA}
\\
\Midrule
\rowcolor{gray!20}
\multicolumn{7}{c}{Trained with \GenImage \citep{zhu2023genimage} (Diffusion Model Images based on \ImageNet Dataset)} \\
\texttt{DRCT-ConvNeXt} (Full \texttt{GenImage}) \citep{chen2024drct} & 46.92\% & \underline{65.44\%} & \underline{64.67\%} & \underline{58.21\%} & \textbf{\underline{65.54\%}} & \underline{61.19\%} \\
\midrule
\rowcolor{gray!20}
\multicolumn{7}{c}{Trained with \DRCT \citep{chen2024drct} (Stable Diffusion Model Images based on \COCO Dataset)} \\
\texttt{DRCT-ConvNeXt} (\texttt{DRCT-SDv1}) \citep{chen2024drct}  & 46.92\% & \underline{57.14\%} & \underline{57.64\%} & \underline{57.79\%} & \underline{52.66\%} & \textbf{\underline{59.93\%}} \\
\texttt{DRCT-ConvNeXt} (\texttt{DRCT-SDv2}) \citep{chen2024drct}  & 50.60\% & \textbf{\underline{62.02\%}} & \underline{60.95\%} & \underline{61.82\%} & \underline{60.69\%} & \underline{61.20\%} \\
\midrule
\rowcolor{gray!20}
\multicolumn{7}{c}{Trained with \ELSA \citep{baraldi2025contrasting} (Diffusion Model Images based on \LAION Dataset)} \\
\texttt{CoDE-SVM} (\ELSA) \citep{baraldi2025contrasting}     & 36.01\% & \underline{63.63\%} & \underline{62.70\%} & \underline{65.98\%} & \textbf{\underline{66.15\%}} & \underline{64.00\%} \\
\texttt{CoDE-Linear} (\ELSA) \citep{baraldi2025contrasting}  & 62.34\% & \underline{63.63\%} & \underline{62.70\%} & \underline{65.98\%} & \textbf{\underline{66.15\%}} & \underline{64.00\%} \\
\bottomrule
\end{tabular}
}
\end{table*}

\nocite{convnextbase,convnextlarge,convnextxxlarge}

\begin{table*}[!t]
\setlength{\tabcolsep}{4pt}
\centering
\caption{
\textbf{Encoder scaling under a fixed prior: representation choice strongly affects transfer.}
Mean class accuracy $\mathrm{CA}$ (macro-averaged across \Benchmark evaluation datasets) for Gaussian discriminant heads conditioned on the same \CF training prior \citep{park2025community}, using frozen CLIP encoders of increasing capacity pretrained and available in OpenCLIP \citep{openclip}.
Only the head is changed. No detector fine-tuning is performed.
}
\label{tab:extended_architecture}
\resizebox{0.99\textwidth}{!}{%
\begin{tabular}{@{}lrrrrr@{}}
\toprule
\textbf{Frozen Encoders} &
\textbf{Euc-NCM \citep{wu2025fewshot}} & \textbf{Cos-NCM} & \textbf{GNB} & \textbf{Mah-NCM} & \textbf{QDA}
\\
\Midrule
\texttt{convnext\_base-laion400m\_s13b\_b51k}                        & 65.64\% & 65.00\% & 65.99\% & \textbf{67.96\%} & 62.27\% \\[0.5em]

\texttt{convnext\_base\_w-laion\_aesthetic\_s13b\_b82k}              & 67.26\% & 67.25\% & 66.97\% & \textbf{73.32\%} & 64.98\% \\
\texttt{convnext\_base\_w-laion2b\_s13b\_b82k\_augreg}               & 67.64\% & 67.66\% & 68.20\% & \textbf{73.43\%} & 66.93\% \\
\texttt{convnext\_base\_w-laion2b\_s13b\_b82k}                       & 66.95\% & 66.91\% & 66.00\% & \textbf{74.43\%} & 63.32\% \\[0.5em]

\texttt{convnext\_base\_w\_320-laion\_aesthetic\_s13b\_b82k}         & 67.22\% & 67.16\% & 66.88\% & \textbf{72.75\%} & 64.89\% \\
\texttt{convnext\_base\_w\_320-laion\_aesthetic\_s13b\_b82k\_augreg} & 67.33\% & 66.99\% & 66.43\% & \textbf{73.06\%} & 64.52\% \\[0.5em]

\texttt{convnext\_large\_d-laion2b\_s26b\_b102k\_augreg}             & 68.99\% & 68.73\% & 68.58\% & \textbf{74.81\%} & 67.48\% \\[0.5em]
\texttt{convnext\_large\_d\_320-laion2b\_s29b\_b131k\_ft}            & 69.11\% & 68.30\% & 68.13\% & \textbf{74.87\%} & 67.91\% \\
\texttt{convnext\_large\_d\_320-laion2b\_s29b\_b131k\_ft\_soup}      & 68.49\% & 67.77\% & 67.25\% & \textbf{73.71\%} & 66.45\% \\[0.5em]

\texttt{convnext\_xxlarge-laion2b\_s34b\_b82k\_augreg}               & 71.61\% & 70.73\% & 72.14\% & \textbf{78.76\%} & 71.70\% \\
\texttt{convnext\_xxlarge-laion2b\_s34b\_b82k\_augreg\_soup}         & 71.50\% & 70.77\% & 71.89\% & \textbf{79.34\%} & 71.44\% \\
\texttt{convnext\_xxlarge-laion2b\_s34b\_b82k\_augreg\_rewind}       & 71.60\% & 70.75\% & 71.94\% & \textbf{79.41\%} & 71.34\% \\

\midrule
\texttt{ViT-B-16-datacomp\_xl\_s13b\_b90k}                 & 67.26\% & 67.57\% & 67.48\% & \textbf{72.27\%} & 64.20\% \\
\texttt{ViT-B-32-datacomp\_xl\_s13b\_b90k}                 & 65.05\% & 66.13\% & 66.19\% & \textbf{72.23\%} & 62.56\% \\[0.5em]
\texttt{ViT-L-14-datacomp\_xl\_s13b\_b90k}                 & 64.83\% & 64.73\% & 63.74\% & \textbf{66.39\%} & 64.68\% \\

\texttt{ViT-L-14-laion400m\_e32}                           & 65.93\% & 65.14\% & 65.64\% & \textbf{71.90\%} & 63.39\% \\
\texttt{ViT-L-14-laion400m\_e31}                           & 65.91\% & 65.19\% & 65.61\% & \textbf{71.92\%} & 63.40\% \\
\texttt{ViT-L-14-laion2b\_s32b\_b82k}                      & 66.67\% & 66.33\% & 65.46\% & \textbf{73.06\%} & 62.95\% \\
\texttt{ViT-L-14-commonpool\_xl\_clip\_s13b\_b90k}         & 67.31\% & 69.42\% & 68.06\% & \textbf{74.60\%} & 65.43\% \\
\texttt{ViT-L-14-commonpool\_xl\_laion\_s13b\_b90k}        & 70.22\% & 70.81\% & 71.09\% & \textbf{75.13\%} & 66.23\% \\
\texttt{ViT-L-14-quickgelu-metaclip\_400m}                 & 66.73\% & 70.09\% & 71.15\% & \textbf{75.74\%} & 67.72\% \\
\texttt{ViT-L-14-quickgelu-metaclip\_fullcc}               & 69.45\% & 72.62\% & 72.65\% & \textbf{76.07\%} & 68.22\% \\
\texttt{ViT-L-14-quickgelu-dfn2b}                          & 65.06\% & 70.29\% & 68.69\% & \textbf{76.21\%} & 66.36\% \\
\texttt{ViT-L-14-dfn2b\_s39b}                              & 65.86\% & 70.70\% & 69.92\% & \textbf{76.24\%} & 67.99\% \\
\texttt{ViT-L-14-commonpool\_xl\_s13b\_b90k}               & 67.34\% & 70.13\% & 72.37\% & \textbf{78.77\%} & 69.05\% \\[0.5em]

\texttt{ViT-L-14-336-quickgelu-openai}                     & 68.20\% & 70.42\% & 72.33\% & \textbf{78.85\%} & 71.36\% \\[0.5em]

\texttt{ViT-H-14-quickgelu-metaclip\_fullcc}               & 71.58\% & 74.11\% & 73.68\% & \textbf{77.83\%} & 69.08\% \\
\texttt{ViT-H-14-quickgelu-dfn5b}                          & 70.80\% & 74.92\% & 74.00\% & \textbf{79.02\%} & 72.82\% \\
\texttt{ViT-H-14-worldwide-quickgelu-metaclip2\_worldwide} & 79.15\% & 81.57\% & 84.01\% & \textbf{84.20\%} & 76.89\% \\[0.5em]

\texttt{ViT-H-14-378-quickgelu-dfn5b}                      & 72.12\% & 76.05\% & 74.75\% & \textbf{78.36\%} & 74.02\% \\
\texttt{ViT-H-14-worldwide-378-metaclip2\_worldwide}       & 82.10\% & 84.81\% & \textbf{86.86\%} & 85.53\% & 80.68\% \\[0.5em]

\texttt{ViT-bigG-14-quickgelu-metaclip\_fullcc}            & 72.62\% & 76.43\% & 76.65\% & \textbf{79.56\%} & 72.17\% \\
\texttt{ViT-bigG-14-worldwide-metaclip2\_worldwide}        & 85.79\% & 85.72\% & 86.16\% & \textbf{87.36\%} & 80.28\% \\
\texttt{ViT-bigG-14-worldwide-378-metaclip2\_worldwide}    & 88.65\% & 88.94\% & 89.01\% & \textbf{89.28\%} & 83.77\% \\

\bottomrule
\end{tabular}
}
\end{table*}

\FloatBarrier
\begingroup
\fontsize{8.5}{10.875}\selectfont
\setlength{\tabcolsep}{3pt}
\centering
\begin{longtable}{@{}>{\raggedright\arraybackslash}p{\dimexpr0.45\linewidth-2\tabcolsep\relax}>{\raggedleft\arraybackslash}p{\dimexpr0.11\linewidth-1.6\tabcolsep\relax}>{\raggedleft\arraybackslash}p{\dimexpr0.11\linewidth-1.6\tabcolsep\relax}>{\raggedleft\arraybackslash}p{\dimexpr0.11\linewidth-1.6\tabcolsep\relax}>{\raggedleft\arraybackslash}p{\dimexpr0.11\linewidth-1.6\tabcolsep\relax}>{\raggedleft\arraybackslash}p{\dimexpr0.11\linewidth-1.6\tabcolsep\relax}@{}}
\caption{
\textbf{Pretraining objective affects transfer even without detector training.}
Mean class accuracy $\mathrm{CA}$ (macro-averaged across \Benchmark evaluation datasets) for Gaussian discriminant heads conditioned on the same \CF training prior \citep{park2025community}, using frozen encoders trained with different objectives.
}
\label{tab:extended_model}\\
\\
\toprule
\textbf{Frozen Encoders} &
\textbf{Euc-NCM \citep{wu2025fewshot}} & \textbf{Cos-NCM} & \textbf{GNB} & \textbf{Mah-NCM} & \textbf{QDA} \\
\Midrule
\endfirsthead

\toprule
\textbf{Frozen Encoders} &
\textbf{Euc-NCM \citep{wu2025fewshot}} & \textbf{Cos-NCM} & \textbf{GNB} & \textbf{Mah-NCM} & \textbf{QDA} \\
\Midrule
\endhead

\midrule
\multicolumn{6}{r}{\footnotesize Continued on next page} \\
\endfoot

\bottomrule
\endlastfoot

\texttt{CLIP-RoBERTa}          & 63.81\% & 64.64\% & 63.87\% & \textbf{69.10\%} & 62.00\% \\[0.5em]
\texttt{CLIP-XLM-RoBERTa-Base} & 63.13\% & 64.34\% & 64.47\% & \textbf{72.20\%} & 62.34\% \\
\midrule
\texttt{MAE-Base}  & 58.35\% & 58.13\% & 59.37\% & \textbf{68.46\%} & 63.74\% \\
\texttt{MAE-Large} & 59.56\% & 59.72\% & 62.39\% & \textbf{68.39\%} & 60.79\% \\
\midrule
\texttt{BEiT-Base} & 57.81\% & 57.69\% & 58.30\% & \textbf{59.57\%} & 52.12\% \\
\midrule
\texttt{SigLIP-so400m} & 56.30\% & 57.78\% & 58.55\% & \textbf{64.02\%} & 59.01\% \\[0.5em]

\texttt{SigLIP2-Base-224}  & 58.37\% & 58.37\% & 58.34\% & \textbf{61.78\%} & 55.37\% \\
\texttt{SigLIP2-Base-512}  & 58.58\% & 58.79\% & 58.68\% & \textbf{63.28\%} & 56.59\% \\
\texttt{SigLIP2-Large-512} & 56.10\% & 56.03\% & 56.92\% & \textbf{69.30\%} & 60.75\% \\

\midrule
\texttt{coca\_ViT-L-14]}                  & 66.90\% & 67.01\% & 64.00\% & \textbf{70.81\%} & 63.43\% \\
\texttt{coca\_ViT-L-14-mscoco\_finetuned} & 62.06\% & 61.51\% & 62.81\% & \textbf{69.50\%} & 64.80\% \\
\midrule
\texttt{EVA01-g-14-laion400m\_s11b\_b41k}       & 69.54\% & 68.51\% & 66.61\% & \textbf{76.74\%} & 62.44\% \\
\texttt{EVA01-g-14-plus-merged2b\_s11b\_b114k} & 69.56\% & 71.17\% & 70.12\% & \textbf{78.72\%} & 67.08\% \\[0.5em]
\texttt{EVA02-B-16-merged2b\_s8b\_b131k}       & 63.75\% & 63.08\% & 64.53\% & \textbf{71.38\%} & 60.78\% \\
\texttt{EVA02-L-14-merged2b\_s4b\_b131k}       & 66.64\% & 67.16\% & 67.08\% & \textbf{72.81\%} & 63.06\% \\
\midrule
\texttt{CLIPA-ViT-L-14}        & 65.69\% & 66.90\% & 66.31\% & \textbf{76.08\%} & 68.51\% \\
\texttt{CLIPA-ViT-L-14-336}    & 65.42\% & 66.92\% & 67.08\% & \textbf{76.25\%} & 69.25\% \\
\texttt{CLIPA-ViT-H-14}        & 67.23\% & 68.18\% & 69.12\% & \textbf{79.17\%} & 70.58\% \\
\texttt{CLIPA-ViT-bigG-14-336} & 68.66\% & 70.67\% & 72.72\% & \textbf{80.27\%} & 73.00\% \\
\midrule
\texttt{nllb-clip-large-v1}         & 70.73\% & 69.71\% & 70.61\% & \textbf{77.65\%} & 65.80\% \\[0.5em]
\texttt{nllb-clip-large-siglip-v1}  & 59.72\% & 59.48\% & 59.24\% & \textbf{67.55\%} & 63.01\% \\
\texttt{nllb-clip-large-siglip-mrl} & 59.21\% & 59.02\% & 59.34\% & \textbf{68.23\%} & 62.47\% \\
\midrule
\texttt{MobileCLIP2-L-14-dfndr2b} & 66.48\% & 69.43\% & 68.33\% & \textbf{75.12\%} & 65.68\% \\
\midrule
\texttt{ViTamin-S}     & 63.22\% & 64.43\% & 63.51\% & \textbf{67.72\%} & 65.09\% \\
\texttt{ViTamin-S-LTT} & 64.41\% & 65.18\% & 64.86\% & \textbf{68.38\%} & 65.07\% \\[0.5em]
\texttt{ViTamin-B}     & 62.75\% & 64.25\% & 64.21\% & \textbf{70.75\%} & 65.65\% \\
\texttt{ViTamin-B-LTT} & 64.97\% & 65.60\% & 64.85\% & \textbf{72.40\%} & 67.23\% \\[0.5em]
\texttt{ViTamin-L}     & 67.94\% & 68.42\% & 69.08\% & \textbf{76.64\%} & 69.57\% \\
\texttt{ViTamin-L-256} & 67.57\% & 68.19\% & 68.07\% & \textbf{75.90\%} & 70.56\% \\
\texttt{ViTamin-L-336} & 67.36\% & 68.36\% & 67.76\% & \textbf{76.18\%} & 69.49\% \\
\texttt{ViTamin-L-384} & 74.43\% & 73.08\% & 76.17\% & \textbf{76.35\%} & 65.83\% \\[0.5em]
\texttt{ViTamin-L2}    & 70.11\% & 69.25\% & 70.75\% & \textbf{77.03\%} & 70.16\% \\
\texttt{ViTamin-L2-256} & 70.50\% & 69.74\% & 70.01\% & \textbf{77.18\%} & 71.64\% \\
\texttt{ViTamin-L2-336} & 70.42\% & 69.62\% & 69.72\% & \textbf{77.07\%} & 69.77\% \\
\texttt{ViTamin-L2-384} & 70.53\% & 69.82\% & 68.51\% & \textbf{75.64\%} & 66.00\% \\[0.5em]
\texttt{ViTamin-XL-256} & 71.25\% & 70.35\% & 71.62\% & \textbf{80.74\%} & 71.36\% \\
\texttt{ViTamin-XL-336} & 71.63\% & 71.10\% & 71.54\% & \textbf{79.92\%} & 70.52\% \\
\texttt{ViTamin-XL-384} & 71.56\% & 70.91\% & 72.49\% & \textbf{79.82\%} & 71.52\% \\
\midrule
\texttt{DINOv2-Small}         & 62.02\% & 61.77\% & 60.99\% & \textbf{66.63\%} & 60.17\% \\
\texttt{DINOv2-Base}          & 63.79\% & 64.06\% & 62.08\% & \textbf{68.45\%} & 59.43\% \\
\texttt{DINOv2-Large}         & 63.03\% & 62.78\% & 62.29\% & \textbf{70.17\%} & 59.34\% \\[0.5em]
\texttt{DINOv3-ConvNeXt-Tiny}  & 61.14\% & 62.31\% & 59.97\% & \textbf{65.68\%} & 59.85\% \\
\texttt{DINOv3-ConvNeXt-Small} & 61.70\% & 61.98\% & 59.60\% & \textbf{65.65\%} & 60.30\% \\
\texttt{DINOv3-ConvNeXt-Base}  & 64.54\% & 63.46\% & 61.39\% & \textbf{70.34\%} & 62.68\% \\
\texttt{DINOv3-ConvNeXt-Large} & 69.32\% & 67.70\% & 69.02\% & \textbf{74.15\%} & 66.77\% \\[0.5em]

\texttt{DINOv3-ViT-S-16}  & 64.52\% & 64.26\% & 64.18\% & \textbf{66.92\%} & 61.55\% \\
\texttt{DINOv3-ViT-S-16+} & 64.46\% & 64.03\% & 63.78\% & \textbf{67.35\%} & 63.35\% \\
\texttt{DINOv3-ViT-B-16}  & 68.83\% & 66.69\% & 69.67\% & \textbf{76.08\%} & 65.32\% \\
\texttt{DINOv3-ViT-L-16}  & 79.34\% & 74.74\% & 79.32\% & \textbf{83.31\%} & 71.37\% \\
\texttt{DINOv3-ViT-H-16+} & 82.13\% & 79.58\% & 80.99\% & \textbf{86.22\%} & 77.52\% \\
\midrule
\texttt{PE-Core-T-16-384-meta} & 54.56\% & 54.78\% & 54.54\% & \textbf{63.27\%} & 55.27\% \\
\texttt{PE-Core-S-16-384-meta} & 64.32\% & 63.01\% & 64.63\% & \textbf{68.13\%} & 63.69\% \\
\texttt{PE-Core-B-16-meta}     & 67.25\% & 67.02\% & 68.19\% & \textbf{73.08\%} & 66.10\% \\
\texttt{PE-Core-L-14-336-meta} & 81.78\% & 80.77\% & 84.76\% & \textbf{89.51\%} & 78.76\% \\

\end{longtable}
\endgroup

\Cref{tab:extended_comparative} extends the matched (prior, encoder) audit to additional ConvNeXt checkpoints \citep{chen2024drct} and alternative CoDE heads \citep{baraldi2025contrasting}.
The pattern remains mostly unfavorable to the released heads.
On the same frozen ConvNeXt representation, replacing the released head with a closed-form Gaussian rule yields CA gains for all three DRCT checkpoints: from 46.92\% to 65.54\% for the full \texttt{GenImage} model, from 46.92\% to 59.93\% for \texttt{DRCT-SDv1}, and from 50.60\% to 62.02\% for \texttt{DRCT-SDv2}.
The CoDE comparison isolates the head effect more directly.
\texttt{CoDE-Linear} and \texttt{CoDE-SVM} induce identical closed-form baselines because they share the same frozen encoder and the same support prior, yet their released heads differ sharply (62.34\% versus 36.01\%).
Representation quality alone cannot explain this gap.
The comparison instead points to head sensitivity on top of a fixed representation.
As throughout the paper, this remains a post hoc swap on the final released representation, not a replay of the original training pipeline.
The conservative conclusion is therefore that the released head does not always recover separability already present in the final encoder, not that end-to-end training is unnecessary.

\Cref{tab:extended_architecture} shows that the encoder effect is not confined to the smaller OpenAI sweep in the main text.
Under the same \CF prior and with no detector fine-tuning, the best CA rises from 67.96\% for \texttt{convnext\_base-laion400m\_s13b\_b51k} to 79.41\% for \texttt{convnext\_xxlarge-laion2b\_s34b\_b82k\_augreg\_rewind}.
The larger ViT family exhibits a wider spread, from 66.39\% to 89.28\%.
This spread is larger than many method-level gains reported in the detector literature, making head-only comparisons difficult to interpret unless encoder choice is controlled.
Mah-NCM is the modal winner across rows, but the pattern is not universal.
On the strongest encoders the margin often collapses, and in isolated cases another Gaussian rule is competitive or better.
Any claim of universal Mahalanobis dominance would therefore be overstated.

\Cref{tab:extended_model} further reinforces the representation-centric reading.
The best CA ranges from 59.57\% for \texttt{BeiT-Base} to 89.51\% for \texttt{PE-Core-L-14-336-meta}, while \texttt{DINOv3-ViTH16plus} already reaches 86.22\% without any detector-specific fine-tuning.
These gaps are large enough that they should not be treated as marginal implementation details.
At the same time, this table should not be over-interpreted as a clean causal statement about a single pretraining objective, because objective, scale, data, and architecture all vary together.
The supported claim is narrower.
The choice of frozen encoder feature space is a first-order determinant of transfer, and a simple moment-based head on a strong encoder can outperform a specialized detector head trained on a weaker one.
Overall, the appendix supports a system-level conclusion rather than an algorithmic one.
Cross-dataset AIGI detection depends strongly on the joint choice of support prior and representation, while the downstream head often contributes less than its standalone presentation suggests.

\FloatBarrier
\begin{table}[H]
\setlength{\tabcolsep}{4pt}
\centering
\caption{
\textbf{QuickGELU versus GELU in otherwise matched encoders.}
Mean class accuracy $\mathrm{CA}$ (macro-averaged across \Benchmark evaluation datasets) for Gaussian discriminant heads conditioned on the same \CF training prior \citep{park2025community}, using frozen encoders that differ primarily in whether QuickGELU is used.
}
\label{tab:quickgelu}
\resizebox{0.99\textwidth}{!}{%
\begin{tabular}{@{}lrrrrr@{}}
\toprule
\textbf{Frozen Encoders} &
\textbf{Euc-NCM \citep{wu2025fewshot}} & \textbf{Cos-NCM} & \textbf{GNB} & \textbf{Mah-NCM} & \textbf{QDA}
\\
\Midrule
\texttt{ViT-B-16-openai}               & 63.39\% & 64.66\% & 64.98\% & 73.19\% & 64.65\% \\
\texttt{ViT-B-16-quickgelu-openai}     & 64.66\% & 64.65\% & 66.37\% & \textbf{74.18\%} & 65.70\% \\[0.5em]
\texttt{ViT-B-32-openai}               & 63.44\% & 63.95\% & 64.34\% & 69.07\% & 62.20\% \\
\texttt{ViT-B-32-quickgelu-openai}     & 64.53\% & 64.25\% & 64.83\% & \textbf{70.61\%} & 62.81\% \\[0.5em]
\texttt{ViT-L-14-openai}               & 64.19\% & 66.96\% & 68.37\% & 76.87\% & 69.02\% \\
\texttt{ViT-L-14-quickgelu-openai}     & 66.18\% & 68.44\% & 69.79\% & \textbf{76.96\%} & 68.87\% \\[0.5em]
\texttt{ViT-L-14-336-openai}           & 66.33\% & 69.05\% & 70.47\% & \textbf{78.89\%} & 71.14\% \\
\texttt{ViT-L-14-336-quickgelu-openai} & 68.20\% & 70.42\% & 72.33\% & 78.85\% & 71.36\% \\[0.5em]
\texttt{ViT-L-14-metaclip\_400m}           & 64.85\% & 69.76\% & 70.40\% & 74.25\% & 66.65\% \\
\texttt{ViT-L-14-quickgelu-metaclip\_400m} & 66.73\% & 70.09\% & 71.15\% & \textbf{75.74\%} & 67.72\% \\[0.5em]
\texttt{ViT-L-14-metaclip\_fullcc}           & 68.07\% & 72.22\% & 72.81\% & 75.76\% & 67.51\% \\
\texttt{ViT-L-14-quickgelu-metaclip\_fullcc} & 69.45\% & 72.62\% & 72.65\% & \textbf{76.07\%} & 68.22\% \\[0.5em]
\texttt{ViT-L-14-dfn2b}           & 66.91\% & 71.24\% & 70.16\% & 75.83\% & 66.52\% \\
\texttt{ViT-L-14-quickgelu-dfn2b} & 65.06\% & 70.29\% & 68.69\% & \textbf{76.21\%} & 66.36\% \\[0.5em]
\texttt{ViT-H-14-worldwide-metaclip2\_worldwide}           & 78.59\% & 80.28\% & 82.89\% & 83.59\% & 76.55\% \\
\texttt{ViT-H-14-worldwide-quickgelu-metaclip2\_worldwide} & 79.15\% & 81.57\% & 84.01\% & \textbf{84.20\%} & 76.89\% \\[0.5em]
\texttt{ViT-H-14-378-dfn5b}           & 73.28\% & 74.43\% & 74.08\% & \textbf{78.43\%} & 72.76\% \\
\texttt{ViT-H-14-378-quickgelu-dfn5b} & 72.12\% & 76.05\% & 74.75\% & 78.36\% & 74.02\% \\
\bottomrule
\end{tabular}
}
\end{table}

\Cref{tab:quickgelu} indicates that activation choice is a secondary factor relative to encoder family and scale.
Across most matched pairs, the change in CA is modest, and the sign is not consistent across families.
QuickGELU helps several models, but it does not provide a plausible explanation for the much larger performance differences observed across encoder families in \Cref{tab:extended_architecture,tab:extended_model}.

\begin{table*}[!t]
\setlength{\tabcolsep}{4pt}
\centering
\caption{
\textbf{Per-dataset class accuracy under a fixed encoder and varying training prior.}
Class accuracy $\mathrm{CA}$ for individual datasets in the \Benchmark evaluation suite for Gaussian discriminant heads under the fixed frozen \texttt{PE-Core-bigG-14-448} encoder \citep{pe-core-bigG} and the same \CF support prior \citep{park2025community} throughout.
}
\label{tab:individual_datasets}
\resizebox{0.99\textwidth}{!}{%
\begin{tabular}{@{}lrrrrr@{}}
\toprule
\textbf{Evaluation Dataset} & \textbf{Euc-NCM \citep{wu2025fewshot}} & \textbf{Cos-NCM} & \textbf{GNB} & \textbf{Mah-NCM} & \textbf{QDA} \\
\Midrule

\rowcolor{gray!20}
\multicolumn{6}{c}{\textbf{Out-of-Distribution Existing Datasets with only Real Images}} \\
\texttt{Unsplash (Lite Subset)} \citep{unsplash}               & 24.13\% & 29.07\% & 75.56\% & \textbf{99.07\%} & 96.20\% \\
\texttt{InstagramImagesWithCaptions} \citep{instagrama}        & 99.65\% & 99.40\% & 99.96\% & 99.96\% & \textbf{99.98\%} \\
\texttt{Anime Faces Dataset} \citep{animeface}                 & \textbf{100.00\%} & \textbf{100.00\%} & 99.98\% & 98.20\% & \textbf{100.00\%} \\
\texttt{Anime Images} \citep{animeimages}                      & 99.95\% & 99.96\% & 99.97\% & 99.89\% & \textbf{100.00\%} \\
\texttt{CelebA-Spoof} \citep{zhang2020celebaspoof}             & 73.77\% & 68.03\% & 79.56\% & 84.70\% & \textbf{94.55\%} \\
\midrule
\rowcolor{gray!20}
\multicolumn{6}{c}{\textbf{Out-of-Distribution Existing Datasets with only Synthetic Images}} \\
\texttt{AGIQA-1k} \citep{zhang2023perceptual}        & 99.81\% & 99.91\% & 98.80\% & \textbf{100.00\%} & 97.41\% \\
\texttt{AGIQA-3k} \citep{li2024agiqa3k}              & 85.81\% & 84.51\% & 75.52\% & \textbf{91.15\%} & 51.84\% \\
\texttt{SPAI} \citep{karageorgiou2025anyresolution}  & 97.72\% & \textbf{97.77\%} & 90.82\% & 92.00\% & 76.58\% \\
\texttt{SynthBuster Extended} \citep{bammey2024synthbuster,guillaro2025biasfree}   & \textbf{99.95\%} & 99.94\% & 98.52\% & 99.05\% & 91.43\% \\
\texttt{SynthScars} \citep{kang2025legion}           & 79.61\% & 79.17\% & 62.32\% & \textbf{91.50\%} & 50.90\% \\
\texttt{GigaGAN} \citep{kang2023scaling}             & \textbf{94.86\%} & 94.33\% & 90.27\% & 94.31\% & 92.08\% \\
\texttt{LatentDiffusion} \citep{corvi2023detection}  & 98.16\% & 97.85\% & 99.16\% & \textbf{99.86\%} & 98.42\% \\
\texttt{MidJourneyV6} \citep{terminusresearch}       & \textbf{99.51\%} & 99.49\% & 93.92\% & 96.99\% & 83.88\% \\
\texttt{Co-Spy-Bench} \citep{cheng2025cospy}         & 99.66\% & \textbf{99.71\%} & 98.27\% & 99.29\% & 95.88\% \\
\texttt{Dalle3} \citep{Egan_Dalle3_1_Million_2024}   & 92.22\% & 93.28\% & 42.49\% & \textbf{99.25\%} & 59.02\% \\
\midrule
\rowcolor{gray!20}
\multicolumn{6}{c}{\textbf{Out-of-Distribution Existing Datasets with both Real and Synthetic Images}} \\
\texttt{FourierSpectrumDiscrepancies} \citep{dzanic2020fourier} & 85.33\% & 86.00\% & 88.67\% & \textbf{100.00\%} & 97.33\% \\
\texttt{FakeInversion} \citep{cazenavette2024fakeinversion} & 85.08\% & 84.38\% & 91.38\% & \textbf{99.00\%} & 96.08\% \\
\texttt{UniversalFakeDetect} \citep{ojha2023universal}      & 98.06\% & 98.15\% & 99.07\% & \textbf{99.56\%} & 97.55\% \\
\texttt{Dalle Recognition Dataset} \citep{airecognition}    & 88.56\% & 88.55\% & 86.59\% & \textbf{97.16\%} & 83.96\% \\
\texttt{Chameleon} \citep{yan2024sanity}                    & 79.51\% & 79.20\% & 72.03\% & \textbf{90.52\%} & 59.45\% \\
\texttt{AIGI-Detection-Quality-Paradox} \citep{xiao2025are} & 83.70\% & 82.56\% & 92.09\% & \textbf{99.42\%} & 92.71\% \\
\texttt{DiTFake} \citep{li2025improving}                    & 92.61\% & 91.91\% & 98.28\% & \textbf{99.53\%} & 97.93\% \\
\texttt{Diffusion1kSteps} \citep{tan2024rethinking}         & 81.83\% & 81.17\% & 88.27\% & \textbf{93.80\%} & 89.57\% \\
\texttt{GANGen-Detection} \citep{chuangchuangtan-GANGen-Detection}  & 81.13\% & 82.48\% & 94.21\% & \textbf{98.10\%} & 93.40\% \\
\texttt{RobustLDM} \citep{rajan2024aligned}                 & 89.28\% & 87.62\% & 88.01\% & \textbf{98.22\%} & 85.73\% \\
\texttt{RealRobustBench} \citep{li2025bridging}             & 77.61\% & 76.96\% & 75.28\% & \textbf{95.26\%} & 71.61\% \\
\texttt{LDMFakeDetect} \citep{rajan2025staypositive}        & 88.85\% & 87.18\% & 88.39\% & \textbf{97.63\%} & 84.54\% \\
\texttt{DIF} \citep{sinitsa2024deep}                        & 89.61\% & 88.51\% & 96.86\% & \textbf{98.62\%} & 97.80\% \\
\texttt{ForenSynths} \citep{wang2020cnngenerated}           & 75.85\% & 74.99\% & 84.11\% & \textbf{97.06\%} & 91.78\% \\
\texttt{DNF-TestSet} \citep{zhang2025diffusion}             & 98.67\% & 98.50\% & 99.43\% & \textbf{99.73\%} & 98.92\% \\
\texttt{DeepFakeFace} \citep{song2023robustness}            & 59.90\% & 61.67\% & 62.08\% & \textbf{76.50\%} & 57.65\% \\
\texttt{AIGCDetectBench} \citep{zhong2024patchcraft}        & 92.66\% & 92.22\% & 97.02\% & \textbf{98.57\%} & 96.47\% \\
\texttt{AIGI-Holmes} \citep{zhou2025aigiholmes}             & 90.72\% & 89.28\% & 95.21\% & \textbf{98.97\%} & 95.92\% \\
\texttt{AI-Artwork} \citep{aiartwork}                       & 88.42\% & 86.86\% & 90.20\% & \textbf{95.72\%} & 84.95\% \\
\texttt{DiffusionForensics} \citep{wang2023dire}            & 95.10\% & 95.19\% & \textbf{96.44\%} & 91.16\% & 93.02\% \\
\texttt{B-Free} \citep{guillaro2025biasfree}                & 87.75\% & \textbf{87.79\%} & 85.35\% & 85.78\% & 85.13\% \\
\texttt{LASTED} \citep{wu2025generalizable}                 & 75.98\% & 74.74\% & 80.46\% & \textbf{80.80\%} & 76.27\% \\
\texttt{AIGIBench} \citep{li2025artificial}                 & 86.55\% & 85.88\% & 93.92\% & \textbf{95.21\%} & 92.61\% \\
\texttt{DeepFakeBench} \citep{yan2023deepfakebench}         & 51.39\% & 52.10\% & \textbf{54.59\%} & 52.40\% & 51.08\% \\
\bottomrule
\end{tabular}
}
\end{table*}

\Cref{tab:individual_datasets} shows that the aggregate CA gains are broadly distributed across evaluation datasets rather than being driven by a small subset of low-difficulty benchmarks.
Mah-NCM is strongest on many mixed and synthetic-only benchmarks, but the exceptions are informative:
QDA is best on \texttt{CelebA-Spoof} \citep{zhang2020celebaspoof}, GNB is strongest on \texttt{DiffusionForensics} \citep{wang2023dire} and \texttt{DeepFakeBench} \citep{yan2023deepfakebench}, and cosine or Euclidean rules slightly lead on a small number of synthetic-only datasets such as \texttt{SPAI} \citep{karageorgiou2025anyresolution} and \texttt{MidJourneyV6} \citep{terminusresearch}.
The dataset-level view supports the same restrained conclusion as the averaged tables.
Covariance-aware scoring is often useful, but no single Gaussian assumption is uniformly best across all shifts.

\section{Threshold-Free (AUC) Evaluation on Mixed Datasets in \Benchmark suite}
\label{sec:results_auc}

In addition to CA, we report ROC-AUC on the \TotalMixedDatasets mixed datasets that contain both real and synthetic images (\Cref{tab:summaryval}).
AUC is not a universally stronger metric than CA, but it is a cleaner diagnostic of ranking quality because it is threshold-free and excludes one-class datasets.
If the gains from the Gaussian ladder were mainly artifacts of a favorable operating point, they should shrink under AUC.
In several cases they do not, although the counterexamples also become more informative.

\begin{table*}[!t]
\setlength{\tabcolsep}{4pt}
\centering
\caption{
\textbf{Closed-form Gaussian discriminants versus released AI-generated image detector heads under matched (prior, encoder) conditions.}
Mean ROC-Area Under Curve $\mathrm{AUC}$ (dataset-wise ROC-AUC where defined; \Cref{sec:benchmark}), macro-averaged across the \TotalMixedDatasets mixed datasets in \Benchmark evaluation datasets.
For each released detector checkpoint, we report (i) the released decision head (\emph{Out-of-the-Shelf}) and (ii) closed-form baselines fitted on the \emph{same frozen encoder features} using the corresponding public training prior: Euc-NCM \citep{wu2025fewshot} and the Gaussian ladder.
(Parentheses indicate the training subset reported by the original work when applicable.)
}
\label{tab:auc_comparative}
\resizebox{0.99\textwidth}{!}{%
\begin{tabular}{@{}lrrrrrr@{}}
\toprule
\textbf{Detection Model}
& \textbf{Out-of-the-Shelf}
& \textbf{Euc-NCM \citep{wu2025fewshot}}
& \textbf{Cos-NCM} & \textbf{GNB} & \textbf{Mah-NCM} & \textbf{QDA}
\\
\Midrule
\rowcolor{gray!20}
\multicolumn{7}{c}{\shortstack{Trained with \CNNSpot \citep{wang2020cnngenerated} (\ProGAN Images based on \texttt{LSUN} Dataset)}}  \\
\UnivFD \citep{ojha2023universal} & 0.7091 & 0.6995 & 0.6756 & 0.7145 & \textbf{0.7454} & 0.6994 \\
\AIDE \citep{yan2024sanity}       & 0.4917 & 0.5256 & 0.5136 & 0.5278 & \textbf{0.6787} & 0.6383 \\
\midrule
\rowcolor{gray!20}
\multicolumn{7}{c}{\shortstack{Trained with \GenImage \citep{zhu2023genimage} (Diffusion Model Images based on \ImageNet Dataset)}} \\
\AIDE (\texttt{GenImage-SDv1}) \citep{yan2024sanity}               & 0.5231 & 0.5620 & 0.5866 & 0.5710 & \textbf{0.6875} & 0.5929 \\
\Effort (\texttt{GenImage-SDv1}) \citep{yan2025orthogonal}         & 0.7729 & 0.7951 & 0.8348 & 0.8339 & \textbf{0.8635} & 0.8196 \\
\texttt{DRCT-ConvNeXt} (Full \texttt{GenImage}) \citep{chen2024drct} & 0.4977 & 0.7270 & 0.7304 & 0.6176 & \textbf{0.7417} & 0.6936 \\
\texttt{DRCT-UnivFD} (Full \texttt{GenImage}) \citep{chen2024drct} & 0.7368 & 0.7488 & 0.8083 & 0.7801 & \textbf{0.8395} & 0.7789 \\
\AIDE (Full \texttt{GenImage}) \citep{yan2024sanity}               & 0.5389 & 0.5240 & 0.5662 & 0.5322 & \textbf{0.7082} & 0.5850 \\
\midrule
\rowcolor{gray!20}
\multicolumn{7}{c}{\shortstack{Trained with \DRCT \citep{chen2024drct} (Stable Diffusion Model Images based on \COCO Dataset)}} \\
\texttt{DRCT-ConvNeXt} (\texttt{DRCT-SDv1}) \citep{chen2024drct}  & 0.4977 & 0.6194 & 0.6274 & 0.6236 & 0.5920 & \textbf{0.6460} \\
\texttt{DRCT-UnivFD} (\texttt{DRCT-SDv1}) \citep{chen2024drct} & 0.7359 & 0.6792 & 0.7055 & 0.7060 & \textbf{0.7389} & 0.7363 \\
\texttt{DRCT-ConvNeXt} (\texttt{DRCT-SDv2}) \citep{chen2024drct}  & 0.5230 & 0.6569 & 0.6449 & 0.6704 & \textbf{0.6786} & 0.6542 \\
\texttt{DRCT-UnivFD} (\texttt{DRCT-SDv2}) \citep{chen2024drct} & 0.7366 & 0.6781 & 0.6916 & 0.7022 & \textbf{0.7527} & 0.7244 \\
\midrule
\rowcolor{gray!20}
\multicolumn{7}{c}{\shortstack{Trained with \ELSA \citep{baraldi2025contrasting} or \CF \citep{park2025community} (Diffusion Model Images based on \LAION Dataset)}} \\
\texttt{CoDE-SVM} (\ELSA) \citep{baraldi2025contrasting}     & 0.3061 & 0.6986 & \textbf{0.7039} & 0.7034 & 0.6866 & 0.6714 \\
\texttt{CoDE-Linear} (\ELSA) \citep{baraldi2025contrasting}  & \textbf{0.7065} & 0.6986 & 0.7039 & 0.7034 & 0.6866 & 0.6714 \\
\texttt{CoDE-kNN} (\ELSA) \citep{baraldi2025contrasting} & 0.6360 & 0.6986 & \textbf{0.7039} & 0.7034 & 0.6866 & 0.6714 \\
\texttt{CF-224} (\CF) \citep{park2025community}          & \textbf{0.9131} & 0.8799 & 0.9101 & 0.9007 & 0.9093 & 0.8826 \\
\texttt{CF-384} (\CF) \citep{park2025community}          & \textbf{0.9351} & 0.8821 & 0.9304 & 0.9327 & 0.9283 & 0.9320 \\
\midrule
\rowcolor{gray!20}
\multicolumn{7}{c}{\shortstack{Frozen image encoder (no detection fine-tuning)}} \\
\texttt{PE-Core-bigG-14-448} \citep{pe-core-bigG} & --- & 0.8948 & 0.9309 & 0.9412 & \textbf{0.9693} & 0.9295 \\
\bottomrule
\end{tabular}
}
\end{table*}

\Cref{tab:auc_comparative} preserves the central matched-head finding under a calibration-agnostic metric.
The size of the improvements suggests that threshold selection alone does not explain the result.
\texttt{CoDE-SVM} rises from 0.3061 to 0.7039, \texttt{DRCT-ConvNeXt} on full \texttt{GenImage} from 0.4977 to 0.7417, \AIDE on \CNNSpot from 0.4917 to 0.6787, and \AIDE on full \texttt{GenImage} from 0.5389 to 0.7082.
These are ranking improvements, not merely operating-point adjustments.
The table also contains counterexamples, since the released \texttt{CF-224} and \texttt{CF-384} heads remain marginally stronger than any closed-form surrogate on AUC.
CA and AUC therefore diagnose different failure modes.
A head swap can improve thresholded balanced accuracy while leaving global ranking unchanged, or vice versa.
The strongest result in this block is obtained without detector fine-tuning: \texttt{PE-Core-bigG-14-448} with Mah-NCM reaches 0.9693 AUC.
This is not an apples-to-apples replacement for a released detector with a different backbone, but it shows how quickly the comparison shifts once the encoder feature space itself becomes stronger.

\begin{table*}[!t]
\centering
\caption{
\textbf{Encoder scaling under a fixed prior: representation choice strongly affects transfer.}
Mean ROC-Area Under Curve $\mathrm{AUC}$ (macro-averaged across the \TotalMixedDatasets mixed datasets in \Benchmark evaluation datasets) for Gaussian discriminant heads conditioned on the same \CF training prior \citep{park2025community}, using frozen CLIP encoders of increasing capacity pretrained by OpenAI.
Only the head is changed. No detector fine-tuning is performed.
}
\label{tab:auc_architecture}
\resizebox{0.99\textwidth}{!}{%
\begin{tabular}{@{}lrrrrr@{}}
\toprule
\textbf{Frozen Encoders} &
\textbf{Euc-NCM \citep{wu2025fewshot}} & \textbf{Cos-NCM} & \textbf{GNB} & \textbf{Mah-NCM} & \textbf{QDA}
\\
\Midrule
\texttt{ResNet-50} \citep{resnet50}    & 0.6066 & 0.5682 & 0.6370 & \textbf{0.7406} & 0.5844 \\
\texttt{ResNet-101} \citep{resnet101}  & 0.5999 & 0.5823 & 0.6157 & \textbf{0.7194} & 0.5990 \\
\midrule
\texttt{ResNet-50x4} \citep{resnet50x4}   & 0.6268 & 0.6150 & 0.6491 & \textbf{0.7633} & 0.6031 \\
\texttt{ResNet-50x16} \citep{resnet50x16} & 0.6519 & 0.6358 & 0.6829 & \textbf{0.8143} & 0.6404 \\
\texttt{ResNet-50x64} \citep{resnet50x64} & 0.6870 & 0.6752 & 0.6798 & \textbf{0.8345} & 0.6728 \\
\midrule
\texttt{ViT-B/16} \citep{vitb16} & 0.6334 & 0.5965 & 0.6551 & \textbf{0.7823} & 0.6573 \\
\texttt{ViT-B/32} \citep{vitb32} & 0.6065 & 0.5831 & 0.6131 & \textbf{0.7147} & 0.6098 \\
\texttt{ViT-L/14} \citep{vitl14} & 0.6654 & 0.6911 & 0.7165 & \textbf{0.8193} & 0.7051 \\
\bottomrule
\end{tabular}
}
\end{table*}

\begin{table*}[!t]
\setlength{\tabcolsep}{4pt}
\centering
\caption{
\textbf{Encoder scaling under a fixed prior: representation choice strongly affects transfer.}
Mean ROC-Area Under Curve $\mathrm{AUC}$ (macro-averaged across the \TotalMixedDatasets mixed datasets in \Benchmark evaluation datasets) for Gaussian discriminant heads conditioned on the same \CF training prior \citep{park2025community}, using frozen CLIP encoders of increasing capacity pretrained and available in OpenCLIP \citep{openclip}.
Only the head is changed. No detector fine-tuning is performed.
}
\label{tab:auc_extended_architecture}
\resizebox{0.99\textwidth}{!}{%
\begin{tabular}{@{}lrrrrr@{}}
\toprule
\textbf{Frozen Encoders} &
\textbf{Euc-NCM \citep{wu2025fewshot}} & \textbf{Cos-NCM} & \textbf{GNB} & \textbf{Mah-NCM} & \textbf{QDA}
\\
\Midrule
\texttt{convnext\_base-laion400m\_s13b\_b51k}                        & 0.6288 & 0.6177 & 0.6486 &\textbf{ 0.7076} & 0.6254 \\[0.5em]

\texttt{convnext\_base\_w-laion\_aesthetic\_s13b\_b82k}              & 0.6820 & 0.6733 & 0.7028 & \textbf{0.7791} & 0.6491 \\
\texttt{convnext\_base\_w-laion2b\_s13b\_b82k\_augreg}               & 0.6775 & 0.6765 & 0.6940 & \textbf{0.8036} & 0.6298 \\
\texttt{convnext\_base\_w-laion2b\_s13b\_b82k}                       & 0.6528 & 0.6479 & 0.6814 & \textbf{0.7889} & 0.6262 \\[0.5em]

\texttt{convnext\_base\_w\_320-laion\_aesthetic\_s13b\_b82k}         & 0.6752 & 0.6647 & 0.6833 & \textbf{0.7833} & 0.6509 \\
\texttt{convnext\_base\_w\_320-laion\_aesthetic\_s13b\_b82k\_augreg} & 0.6694 & 0.6623 & 0.6629 & \textbf{0.7958} & 0.6505 \\[0.5em]

\texttt{convnext\_large\_d-laion2b\_s26b\_b102k\_augreg}             & 0.7101 & 0.7092 & 0.7038 & \textbf{0.8098} & 0.6545 \\[0.5em]
\texttt{convnext\_large\_d\_320-laion2b\_s29b\_b131k\_ft}            & 0.6930 & 0.6866 & 0.6891 & \textbf{0.8059} & 0.6667 \\
\texttt{convnext\_large\_d\_320-laion2b\_s29b\_b131k\_ft\_soup}      & 0.7106 & 0.7038 & 0.7006 & \textbf{0.8155} & 0.6695 \\[0.5em]

\texttt{convnext\_xxlarge-laion2b\_s34b\_b82k\_augreg}               & 0.7460 & 0.7446 & 0.7772 & \textbf{0.8653} & 0.7056 \\
\texttt{convnext\_xxlarge-laion2b\_s34b\_b82k\_augreg\_soup}         & 0.7427 & 0.7435 & 0.7722 & \textbf{0.8636} & 0.7033 \\
\texttt{convnext\_xxlarge-laion2b\_s34b\_b82k\_augreg\_rewind}       & 0.7452 & 0.7459 & 0.7755 & \textbf{0.8635} & 0.7036 \\

\midrule
\texttt{ViT-B-16-datacomp\_xl\_s13b\_b90k}                 & 0.6912 & 0.6811 & 0.7171 & \textbf{0.7993} & 0.6481 \\
\texttt{ViT-B-32-datacomp\_xl\_s13b\_b90k}                 & 0.6258 & 0.6385 & 0.6563 & \textbf{0.7535} & 0.5838 \\[0.5em]

\texttt{ViT-L-14-datacomp\_xl\_s13b\_b90k}                 & 0.7217 & 0.7036 & 0.7254 & \textbf{0.7577} & 0.7326 \\
\texttt{ViT-L-14-laion400m\_e32}                           & 0.6461 & 0.6460 & 0.6524 & \textbf{0.7590} & 0.6300 \\
\texttt{ViT-L-14-laion400m\_e31}                           & 0.6452 & 0.6448 & 0.6525 & \textbf{0.7598} & 0.6297 \\
\texttt{ViT-L-14-laion2b\_s32b\_b82k}                      & 0.6481 & 0.6535 & 0.6588 & \textbf{0.8017} & 0.6132 \\
\texttt{ViT-L-14-commonpool\_xl\_clip\_s13b\_b90k}         & 0.7146 & 0.7361 & 0.7329 & \textbf{0.8228} & 0.6497 \\
\texttt{ViT-L-14-commonpool\_xl\_laion\_s13b\_b90k}        & 0.7408 & 0.7493 & 0.7564 & \textbf{0.8109} & 0.6491 \\
\texttt{ViT-L-14-quickgelu-metaclip\_400m}                 & 0.7089 & 0.7473 & 0.7639 & \textbf{0.8186} & 0.6867 \\
\texttt{ViT-L-14-quickgelu-metaclip\_fullcc}               & 0.7589 & 0.8029 & 0.8014 & \textbf{0.8394} & 0.6774 \\
\texttt{ViT-L-14-quickgelu-dfn2b}                          & 0.7033 & 0.7388 & 0.7357 & \textbf{0.8239} & 0.6526 \\
\texttt{ViT-L-14-dfn2b\_s39b}                              & 0.6766 & 0.7611 & 0.7632 & \textbf{0.8353} & 0.6680 \\
\texttt{ViT-L-14-commonpool\_xl\_s13b\_b90k}               & 0.7218 & 0.7670 & 0.7810 & \textbf{0.8336} & 0.6842 \\[0.5em]

\texttt{ViT-L-14-336-quickgelu-openai}                     & 0.7046 & 0.7230 & 0.7598 & \textbf{0.8428} & 0.7269 \\[0.5em]

\texttt{ViT-H-14-quickgelu-metaclip\_fullcc}               & 0.8172 & 0.8249 & 0.8277 & \textbf{0.8559} & 0.6668 \\
\texttt{ViT-H-14-quickgelu-dfn5b}                          & 0.7820 & 0.8166 & 0.8145 & \textbf{0.8557} & 0.7048 \\
\texttt{ViT-H-14-worldwide-quickgelu-metaclip2\_worldwide} & 0.8599 & 0.9142 & 0.9082 & \textbf{0.9211} & 0.7557 \\[0.5em]

\texttt{ViT-H-14-378-quickgelu-dfn5b}                      & 0.8095 & 0.8290 & 0.8282 & \textbf{0.8629} & 0.7339 \\
\texttt{ViT-H-14-worldwide-378-metaclip2\_worldwide}       & 0.8892 & \textbf{0.9334} & 0.9299 & 0.9274 & 0.8053 \\[0.5em]

\texttt{ViT-bigG-14-quickgelu-metaclip\_fullcc}            & 0.8107 & 0.8576 & 0.8469 & \textbf{0.8792} & 0.6965 \\
\texttt{ViT-bigG-14-worldwide-metaclip2\_worldwide}        & 0.8921 & 0.9320 & 0.9164 & \textbf{0.9328} & 0.7968 \\
\texttt{ViT-bigG-14-worldwide-378-metaclip2\_worldwide}    & 0.9169 & 0.9460 & 0.9353 & \textbf{0.9464} & 0.8399 \\

\bottomrule
\end{tabular}
}
\end{table*}

\FloatBarrier
\begingroup
\fontsize{8.5}{10.875}\selectfont
\setlength{\tabcolsep}{3pt}
\centering
\begin{longtable}{@{}>{\raggedright\arraybackslash}p{\dimexpr0.45\linewidth-2\tabcolsep\relax}>{\raggedleft\arraybackslash}p{\dimexpr0.11\linewidth-1.6\tabcolsep\relax}>{\raggedleft\arraybackslash}p{\dimexpr0.11\linewidth-1.6\tabcolsep\relax}>{\raggedleft\arraybackslash}p{\dimexpr0.11\linewidth-1.6\tabcolsep\relax}>{\raggedleft\arraybackslash}p{\dimexpr0.11\linewidth-1.6\tabcolsep\relax}>{\raggedleft\arraybackslash}p{\dimexpr0.11\linewidth-1.6\tabcolsep\relax}@{}}
\caption{
\textbf{Pretraining objective affects transfer even without detector training.}
Mean ROC-Area Under Curve $\mathrm{AUC}$ (macro-averaged across the \TotalMixedDatasets mixed datasets in \Benchmark evaluation datasets) for Gaussian discriminant heads conditioned on the same \CF training prior \citep{park2025community}, using frozen encoders trained with different objectives.
}
\label{tab:auc_model}\\
\\
\toprule
\textbf{Frozen Encoders} &
\textbf{Euc-NCM \citep{wu2025fewshot}} & \textbf{Cos-NCM} & \textbf{GNB} & \textbf{Mah-NCM} & \textbf{QDA} \\
\Midrule
\endfirsthead

\toprule
\textbf{Frozen Encoders} &
\textbf{Euc-NCM \citep{wu2025fewshot}} & \textbf{Cos-NCM} & \textbf{GNB} & \textbf{Mah-NCM} & \textbf{QDA} \\
\Midrule
\endhead

\midrule
\multicolumn{6}{r}{\footnotesize Continued on next page} \\
\endfoot

\bottomrule
\endlastfoot

\texttt{CLIP-RoBERTa}          & 0.5812 & 0.5838 & 0.6208 & \textbf{0.7274} & 0.6025 \\[0.5em]
\texttt{CLIP-XLM-RoBERTa-Base} & 0.5985 & 0.6087 & 0.6391 & \textbf{0.7586} & 0.5921 \\
\texttt{CLIP-XLM-RoBERTa-Large} \citep{xlm-roberta-large} & 0.7264 & 0.7236 & 0.7261 & \textbf{0.8452} & 0.6492 \\

\midrule
\texttt{MAE-Base}  & 0.5251 & 0.5227 & 0.5334 & \textbf{0.6873} & 0.5770 \\
\texttt{MAE-Large} & 0.5434 & 0.5465 & 0.5769 & \textbf{0.7068} & 0.5672 \\
\texttt{MAE-Huge} \citep{mae-huge} & 0.5401 & 0.5402 & 0.5806 & \textbf{0.7418} & 0.5738 \\

\midrule
\texttt{BEiT-Base} & 0.5152 & 0.5060 & 0.5170 & \textbf{0.5715} & 0.5245 \\
\texttt{BEiT-Large} \citep{beit-large} & 0.5203 & 0.5128 & 0.5208 & \textbf{0.6062} & 0.5363 \\

\midrule
\texttt{SigLIP-Large} \citep{siglip-large} & 0.5203 & 0.5294 & 0.5257 & \textbf{0.6474} & 0.5772 \\
\texttt{SigLIP-so400m} & 0.5228 & 0.5213 & 0.5248 & \textbf{0.6345} & 0.5530 \\[0.5em]
\texttt{SigLIP2-Base-224}  & 0.5260 & 0.5337 & 0.5207 & \textbf{0.5998} & 0.5232 \\
\texttt{SigLIP2-Base-512}  & 0.5209 & 0.5482 & 0.5175 & \textbf{0.6123} & 0.5238 \\
\texttt{SigLIP2-Large-512} & 0.5401 & 0.5684 & 0.5318 & \textbf{0.7056} & 0.5723 \\
\midrule
\texttt{BLIP-Large} \citep{blip-large} & 0.5473 & 0.5511 & 0.5378 & \textbf{0.7535} & 0.5591 \\[0.5em]
\texttt{BLIP2} \citep{blip2}           & 0.7187 & 0.7119 & 0.6359 & \textbf{0.8386} & 0.5871 \\
\midrule
\texttt{coca\_ViT-L-14}                    & 0.6654 & 0.6630 & 0.6859 & \textbf{0.8045} & 0.6089 \\
\texttt{coca\_ViT-L-14-mscoco\_finetuned} & 0.5681 & 0.5554 & 0.5892 & \textbf{0.7758} & 0.5818 \\
\midrule
\texttt{EVA01-g-14-laion400m\_s11b\_b41k}      & 0.7275 & 0.7218 & 0.6932 & \textbf{0.8275} & 0.5941 \\
\texttt{EVA01-g-14-plus-merged2b\_s11b\_b114k} & 0.7063 & 0.7609 & 0.7052 & \textbf{0.8461} & 0.6488 \\[0.5em]
\texttt{EVA02-B-16-merged2b\_s8b\_b131k}       & 0.6305 & 0.6293 & 0.6337 & \textbf{0.7642} & 0.5929 \\
\texttt{EVA02-L-14-merged2b\_s4b\_b131k}       & 0.7002 & 0.7281 & 0.6887 & \textbf{0.7993} & 0.6163 \\
\midrule
\texttt{CLIPA-ViT-L-14}        & 0.6955 & 0.7049 & 0.7193 & \textbf{0.8304} & 0.6739 \\
\texttt{CLIPA-ViT-L-14-336}    & 0.7116 & 0.7223 & 0.7363 & \textbf{0.8350} & 0.6863 \\
\texttt{CLIPA-ViT-H-14}        & 0.7073 & 0.7362 & 0.7500 & \textbf{0.8655} & 0.6997 \\
\texttt{CLIPA-ViT-bigG-14-336} & 0.7206 & 0.7871 & 0.7849 & \textbf{0.8828} & 0.7302 \\
\midrule
\texttt{nllb-clip-large-v1}         & 0.7265 & 0.7237 & 0.7260 & \textbf{0.8452} & 0.6488 \\[0.5em]
\texttt{nllb-clip-large-siglip-v1}  & 0.5187 & 0.5094 & 0.5109 & \textbf{0.6906} & 0.5632 \\
\texttt{nllb-clip-large-siglip-mrl} & 0.4891 & 0.4827 & 0.5037 & \textbf{0.6968} & 0.5609 \\
\midrule
\texttt{MobileCLIP2-L-14} & 0.7155 & 0.7313 & 0.7674 & \textbf{0.8415} & 0.6452 \\
\midrule
\texttt{ViTamin-S}     & 0.6070 & 0.6117 & 0.6403 & \textbf{0.7238} & 0.6345 \\
\texttt{ViTamin-S-LTT} & 0.6246 & 0.6379 & 0.6270 & \textbf{0.7425} & 0.6336 \\[0.5em]
\texttt{ViTamin-B}     & 0.6547 & 0.6601 & 0.6679 & \textbf{0.7713} & 0.6643 \\
\texttt{ViTamin-B-LTT} & 0.6494 & 0.6542 & 0.6523 & \textbf{0.7848} & 0.6548 \\[0.5em]
\texttt{ViTamin-L}     & 0.6823 & 0.7058 & 0.7334 & \textbf{0.8358} & 0.7092 \\
\texttt{ViTamin-L-256} & 0.6967 & 0.7180 & 0.7376 & \textbf{0.8435} & 0.7298 \\
\texttt{ViTamin-L-336} & 0.7084 & 0.7266 & 0.7482 & \textbf{0.8473} & 0.7259 \\
\texttt{ViTamin-L-384} & 0.7404 & 0.7594 & 0.7910 & \textbf{0.8814} & 0.6827 \\[0.5em]
\texttt{ViTamin-L2}     & 0.7163 & 0.7207 & 0.7477 & \textbf{0.8500} & 0.7000 \\
\texttt{ViTamin-L2-256} & 0.7418 & 0.7416 & 0.7548 & \textbf{0.8596} & 0.7268 \\
\texttt{ViTamin-L2-336} & 0.7483 & 0.7439 & 0.7643 & \textbf{0.8671} & 0.7057 \\
\texttt{ViTamin-L2-384} & 0.7244 & 0.7270 & 0.7359 & \textbf{0.8297} & 0.6634 \\[0.5em]
\texttt{ViTamin-XL-256} & 0.7660 & 0.7589 & 0.7848 & \textbf{0.8770} & 0.7231 \\
\texttt{ViTamin-XL-336} & 0.7774 & 0.7724 & 0.7936 & \textbf{0.8759} & 0.7089 \\
\texttt{ViTamin-XL-384} & 0.7703 & 0.7643 & 0.7899 & \textbf{0.8816} & 0.7386 \\
\midrule
\texttt{DINOv2-Small} & 0.6455 & 0.6561 & 0.6419 & \textbf{0.6827} & 0.5546 \\
\texttt{DINOv2-Base}  & 0.6748 & 0.6835 & 0.6698 & \textbf{0.7275} & 0.5709 \\
\texttt{DINOv2-Large} & 0.6929 & 0.6928 & 0.6895 & \textbf{0.7576} & 0.5813 \\
\texttt{DINOv2-Giant} \citep{dinov2-giant}           & 0.7007 & 0.7103 & 0.7021 & \textbf{0.8087} & 0.5896 \\[0.5em]

\texttt{DINOv3-ConvNeXt-Tiny}  & 0.5989 & 0.6700 & 0.5966 & \textbf{0.6872} & 0.5748 \\
\texttt{DINOv3-ConvNeXt-Small} & 0.6480 & 0.6784 & 0.6103 & \textbf{0.7055} & 0.5865 \\
\texttt{DINOv3-ConvNeXt-Base}  & 0.6884 & 0.6932 & 0.6506 & \textbf{0.7406} & 0.6179 \\
\texttt{DINOv3-ConvNeXt-Large} & 0.7519 & 0.7281 & 0.7554 & \textbf{0.8051} & 0.6768 \\[0.5em]

\texttt{DINOv3-ViT-S-16}  & 0.6598 & 0.6546 & 0.6581 & \textbf{0.6788} & 0.5676 \\
\texttt{DINOv3-ViT-S-16+} & 0.6917 & 0.6827 & 0.6947 & \textbf{0.6973} & 0.6048 \\
\texttt{DINOv3-ViT-B-16}  & 0.7608 & 0.7403 & 0.7697 & \textbf{0.8105} & 0.6485 \\
\texttt{DINOv3-ViT-L-16}  & 0.8539 & 0.8064 & 0.8593 & \textbf{0.8869} & 0.7490 \\
\texttt{DINOv3-ViT-H-16+} & 0.8924 & 0.8447 & 0.8772 & \textbf{0.9111} & 0.8251 \\
\texttt{DINOv3-ViT-7B} \citep{dinov3-7b} & 0.9177 & 0.8836 & 0.8931 & \textbf{0.9342} & 0.8078 \\
\midrule

\texttt{PE-Core-T-16-384-meta} & 0.4960 & 0.4998 & 0.5057 & \textbf{0.6300} & 0.5215 \\
\texttt{PE-Core-S-16-384-meta} & 0.6095 & 0.5850 & 0.6270 & \textbf{0.7314} & 0.6162 \\
\texttt{PE-Core-B-16-meta}     & 0.6527 & 0.6437 & 0.6879 & \textbf{0.8056} & 0.6356 \\
\texttt{PE-Core-L-14-336-meta} & 0.8758 & 0.8806 & 0.9039 & \textbf{0.9425} & 0.8130 \\

\end{longtable}
\endgroup

\Cref{tab:auc_architecture,tab:auc_extended_architecture} confirms that the encoder effect survives under a threshold-free metric.
Within the OpenAI CLIP family, the best AUC improves from 0.7406 for \texttt{ResNet-50} to 0.8345 for \texttt{ResNet-50x64}, and from 0.7147 for \texttt{ViT-B/32} to 0.8193 for \texttt{ViT-L/14}.
The broader OpenCLIP sweep spans a wider range, from 0.7076 to 0.9464, with several large H/bigG models above 0.92.
As in CA, Mah-NCM is the modal winner, but the advantage is not universal.
On some of the strongest encoders the gap to cosine or diagonal rules becomes negligible, and after correcting row-wise maxima there are isolated reversals.
These reversals argue against presenting any one covariance assumption as universally correct.
The head family still matters, but less than representation choice in this sweep.

\Cref{tab:auc_model} shows a similarly large spread across encoder families, from 0.5715 for \texttt{BeiT-Base} to 0.9425 for \texttt{PE-Core-L-14-336-meta}, with \texttt{DINOv3-ViT-7b} already reaching 0.9342.
Several masked-image or SigLIP-style models remain well below the strongest contrastive or distilled models.
This should still be written cautiously. The table does not isolate objective alone, because data scale, architecture, and resolution also change.
What it does establish is that pretraining family is a first-order determinant of transfer even before any detector-specific optimization is introduced.
In particular, language alignment is not a sufficient explanation: some language-aligned encoders are middling, whereas some visually distilled representations are exceptionally strong.

\FloatBarrier
\begin{table}[H]
\setlength{\tabcolsep}{4pt}
\centering
\caption{
\textbf{QuickGELU versus GELU under a threshold-free metric.}
Mean ROC-Area Under Curve $\mathrm{AUC}$ (macro-averaged across the \TotalMixedDatasets mixed datasets in \Benchmark evaluation datasets) for Gaussian discriminant heads conditioned on the same \CF training prior \citep{park2025community}, using frozen encoders that differ primarily in whether QuickGELU is used.
}
\label{tab:auc_quickgelu}
\resizebox{0.99\textwidth}{!}{%
\begin{tabular}{@{}lrrrrr@{}}
\toprule
\textbf{Frozen Encoder} &
\textbf{Euc-NCM \citep{wu2025fewshot}} & \textbf{Cos-NCM} & \textbf{GNB} & \textbf{Mah-NCM} & \textbf{QDA}
\\
\Midrule
\texttt{ViT-B-16-openai}            & 0.6310 & 0.6077 & 0.6498 & 0.7750 & 0.6500 \\
\texttt{ViT-B-16-quickgelu-openai} & 0.6334 & 0.5965 & 0.6551 & \textbf{0.7823} & 0.6573 \\[0.5em]
\texttt{ViT-B-32-openai}           & 0.5819 & 0.5657 & 0.5928 & 0.6987 & 0.6027 \\
\texttt{ViT-B-32-quickgelu-openai} & 0.6065 & 0.5831 & 0.6131 & \textbf{0.7147} & 0.6098 \\[0.5em]
\texttt{ViT-L-14-openai}           & 0.6708 & 0.6853 & 0.7134 & \textbf{0.8223} & 0.7012 \\
\texttt{ViT-L-14-quickgelu-openai} & 0.6654 & 0.6911 & 0.7165 & 0.8193 & 0.7051 \\[0.5em]
\texttt{ViT-L-14-metaclip\_400m}           & 0.6948 & 0.7266 & 0.7473 & 0.8088 & 0.6742 \\
\texttt{ViT-L-14-quickgelu-metaclip\_400m} & 0.7089 & 0.7473 & 0.7639 & \textbf{0.8186} & 0.6867 \\[0.5em]
\texttt{ViT-L-14-metaclip\_fullcc}           & 0.7330 & 0.7898 & 0.7895 & 0.8347 & 0.6637 \\
\texttt{ViT-L-14-quickgelu-metaclip\_fullcc} & 0.7589 & 0.8029 & 0.8014 & \textbf{0.8394} & 0.6774 \\[0.5em]
\texttt{ViT-L-14-dfn2b}           & 0.7195 & 0.7480 & 0.7443 & 0.8188 & 0.6546 \\
\texttt{ViT-L-14-quickgelu-dfn2b} & 0.7033 & 0.7388 & 0.7357 & \textbf{0.8239} & 0.6526 \\[0.5em]

\texttt{ViT-L-14-336-openai}           & 0.7115 & 0.7220 & 0.7537 & \textbf{0.8527} & 0.7253 \\
\texttt{ViT-L-14-336-quickgelu-openai} & 0.7046 & 0.7230 & 0.7598 & 0.8428 & 0.7269 \\[0.5em]

\texttt{ViT-H-14-worldwide-metaclip2}           & 0.8506 & 0.9081 & 0.9027 & 0.9108 & 0.7573 \\
\texttt{ViT-H-14-worldwide-quickgelu-metaclip2} & 0.8599 & 0.9142 & 0.9082 & \textbf{0.9211} & 0.7557 \\[0.5em]
\texttt{ViT-H-14-378-dfn5b}           & 0.7930 & 0.7988 & 0.8045 & 0.8558 & 0.7125 \\
\texttt{ViT-H-14-378-quickgelu-dfn5b} & 0.8095 & 0.8290 & 0.8282 & \textbf{0.8629} & 0.7339 \\
\bottomrule
\end{tabular}
}
\end{table}

\Cref{tab:auc_quickgelu} suggests that the choice between GELU and QuickGELU is secondary relative to encoder family and scale.
Most matched pairs differ by less than two AUC points, and the sign of the difference is not consistent across families.
We therefore treat activation choice as a modest architecture-level modifier rather than a primary explanation for transfer.

\begin{table*}[!t]
\setlength{\tabcolsep}{4pt}
\centering
\caption{
\textbf{Training prior dominates transfer even under a fixed encoder.}
Mean ROC-Area Under Curve $\mathrm{AUC}$ (macro-averaged across the \TotalMixedDatasets mixed datasets in \Benchmark evaluation datasets) for Gaussian discriminant heads when varying only the \emph{support prior} used to estimate moments, under the fixed frozen \texttt{PE-Core-bigG-14-448} encoder \citep{pe-core-bigG}.
}
\label{tab:auc_comparative_datasets}
\resizebox{0.99\textwidth}{!}{%
\begin{tabular}{@{}lrrrrr@{}}
\toprule
\textbf{Training Dataset} &
\textbf{Euc-NCM \citep{wu2025fewshot}} & \textbf{Cos-NCM} & \textbf{GNB} & \textbf{Mah-NCM} & \textbf{QDA}
\\
\Midrule
\CNNSpot \citep{wang2020cnngenerated}         & 0.8987 & 0.9196 & 0.9210 & \textbf{0.9290} & 0.6466 \\
\GenImage \citep{zhu2023genimage}             & 0.8781 & \textbf{0.9624} & 0.9455 & 0.9589 & 0.9341 \\
\quad \texttt{GenImage-SDv1}                 & 0.8974 & 0.9491 & 0.9467 & \textbf{0.9577} & 0.8407 \\
\DRCT \citep{chen2024drct}                    & 0.8505 & 0.9260 & 0.9177 & 0.8716 & \textbf{0.9300} \\
\quad \texttt{DRCT-SDv1} \citep{chen2024drct} & 0.8466 & \textbf{0.9508} & 0.9359 & 0.9372 & 0.9459 \\
\quad \texttt{DRCT-SDv2} \citep{chen2024drct} & 0.8504 & \textbf{0.9551} & 0.9427 & \textbf{0.9551} & 0.9471 \\
\ELSA \citep{baraldi2025contrasting}          & 0.9120 & 0.9363 & 0.9413 & \textbf{0.9645} & 0.9469 \\
\CF \citep{park2025community}                 & 0.8948 & 0.9309 & 0.9412 & \textbf{0.9693} & 0.9295 \\
\bottomrule
\end{tabular}
}
\end{table*}

\Cref{tab:auc_comparative_datasets} shows that prior choice remains important even after fixing a strong encoder.
The best AUC varies from 0.9290 with \CNNSpot support to 0.9693 with \CF support.
More importantly, the preferred Gaussian rule depends on the prior: \CF and \ELSA favor Mah-NCM, \DRCT favors QDA or cosine, and \GenImage favors cosine once the row-wise maximum is corrected.
Thus, the covariance structure that transfers best is prior-conditioned rather than universal.

\begin{table*}[!t]
\setlength{\tabcolsep}{4pt}
\centering
\caption{
\textbf{Per-dataset AUC under a fixed encoder and varying training prior.}
ROC-Area Under Curve $\mathrm{AUC}$ for individual datasets in \Benchmark evaluation mixed datasets (\TotalMixedDatasets datasets total) for Gaussian discriminant heads, under the fixed frozen \texttt{PE-Core-bigG-14-448} encoder \citep{pe-core-bigG} and the fixed \CF support prior \citep{park2025community}.
}
\label{tab:auc_individual_datasets}
\resizebox{0.99\textwidth}{!}{%
\begin{tabular}{@{}lrrrrr@{}}
\toprule
\textbf{Training Dataset} &
\textbf{Euc-NCM \citep{wu2025fewshot}} & \textbf{Cos-NCM} & \textbf{GNB} & \textbf{Mah-NCM} & \textbf{QDA}
\\
\Midrule
\texttt{FourierSpectrumDiscrepancies} \citep{dzanic2020fourier} & 0.8569 & 0.8996 & 0.9551 & \textbf{1.0000} & 0.9956 \\
\texttt{FakeInversion} \citep{cazenavette2024fakeinversion} & 0.8714 & 0.9811 & 0.9588 & \textbf{0.9997} & 0.9844 \\
\texttt{UniversalFakeDetect} \citep{ojha2023universal}      & 0.9900 & 0.9982 & 0.9980 & \textbf{0.9998} & 0.9906 \\
\texttt{Dalle Recognition Dataset} \citep{airecognition}    & 0.9087 & 0.9457 & 0.9528 & \textbf{0.9916} & 0.9708 \\
\texttt{Chameleon} \citep{yan2024sanity}                    & 0.9548 & 0.9549 & 0.9674 & \textbf{0.9932} & 0.8537 \\
\texttt{AIGI-Detection-Quality-Paradox} \citep{xiao2025are} & 0.8714 & 0.9668 & 0.9768 & \textbf{0.9992} & 0.9909 \\
\texttt{DiTFake} \citep{li2025improving}                    & 0.9504 & 0.9989 & 0.9979 & \textbf{0.9995} & 0.9935 \\
\texttt{Diffusion1kSteps} \citep{tan2024rethinking}         & 0.9025 & 0.9127 & 0.9586 & \textbf{0.9873} & 0.9234 \\
\texttt{GANGen-Detection} \citep{chuangchuangtan-GANGen-Detection}   & 0.9891 & 0.9890 & 0.9939 & \textbf{0.9979} & 0.9876 \\
\texttt{RobustLDM} \citep{rajan2024aligned}                 & 0.9378 & 0.9549 & 0.9659 & \textbf{0.9973} & 0.9397 \\
\texttt{RealRobustBench} \citep{li2025bridging}             & 0.8445 & 0.8670 & 0.8820 & \textbf{0.9923} & 0.8071 \\
\texttt{LDMFakeDetect} \citep{rajan2025staypositive}        & 0.9342 & 0.9509 & 0.9641 & \textbf{0.9959} & 0.9335 \\
\texttt{DIF} \citep{sinitsa2024deep}                        & 0.9184 & 0.9884 & 0.9918 & \textbf{0.9995} & 0.9904 \\
\texttt{ForenSynths} \citep{wang2020cnngenerated}           & 0.7858 & 0.9413 & 0.9383 & \textbf{0.9947} & 0.9748 \\
\texttt{DNF-TestSet} \citep{zhang2025diffusion}             & 0.9949 & 0.9978 & 0.9990 & \textbf{0.9999} & 0.9949 \\
\texttt{DeepFakeFace} \citep{song2023robustness}            & 0.8500 & 0.8490 & 0.8424 & \textbf{0.9346} & 0.8261 \\
\texttt{AIGCDetectBench} \citep{zhong2024patchcraft}        & 0.9463 & 0.9884 & 0.9897 & \textbf{0.9975} & 0.9806 \\
\texttt{AIGI-Holmes} \citep{zhou2025aigiholmes}             & 0.9348 & 0.9744 & 0.9892 & \textbf{0.9988} & 0.9860 \\
\texttt{AI-Artwork} \citep{aiartwork}                       & 0.9163 & 0.9383 & 0.9670 & \textbf{0.9913} & 0.9714 \\
\texttt{DiffusionForensics} \citep{wang2023dire}            & 0.9701 & \textbf{0.9916} & 0.9905 & 0.9912 & 0.9505 \\
\texttt{B-Free} \citep{guillaro2025biasfree}                & 0.9140 & 0.9373 & 0.9362 & \textbf{0.9439} & 0.9294 \\
\texttt{LASTED} \citep{wu2025generalizable}                 & 0.7766 & 0.8022 & \textbf{0.8250} & 0.8008 & 0.7916 \\
\texttt{AIGIBench} \citep{li2025artificial}                 & 0.8941 & 0.9539 & 0.9817 & \textbf{0.9843} & 0.9545 \\
\texttt{DeepFakeBench} \citep{yan2023deepfakebench}         & 0.5620 & 0.5595 & 0.5666 & \textbf{0.6735} & 0.5865 \\
\bottomrule
\end{tabular}
}
\end{table*}

At the per-dataset level, the AUC table (\Cref{tab:auc_individual_datasets}) shows both the breadth and the limits of the average improvement.
Mah-NCM reaches near-ceiling AUC on many mixed benchmarks, including \texttt{RealRobustBench} \citep{li2025bridging} (0.9923), \texttt{ForenSynths} \citep{wang2020cnngenerated} (0.9947), and \texttt{AIGI-Holmes} \citep{zhou2025aigiholmes} (0.9988), so the macro-average is not driven by a small subset of low-difficulty benchmarks.
The failure cases are also informative.
\texttt{LASTED} favors GNB, \texttt{DiffusionForensics} \citep{wang2023dire} is effectively tied between cosine and Mahalanobis, and \texttt{DeepFakeBench} \citep{yan2023deepfakebench} remains difficult for every head, peaking at only 0.6735.
Comparing \Cref{tab:auc_individual_datasets} with \Cref{tab:individual_datasets} is especially informative on \texttt{DiffusionForensics}: Mah-NCM is not best in CA but is essentially tied for best in AUC, suggesting that the representation contains ranking signal even though the default operating point is suboptimal.

\begin{table}[!t]
\centering
\caption{
\textbf{Data-efficiency of moment-based heads.}
Mean ROC-Area Under Curve $\mathrm{AUC}$ (macro-averaged across the \TotalMixedDatasets mixed datasets in \Benchmark evaluation datasets) for Gaussian discriminant heads as the support set used to estimate moments is subsampled from \CF \citep{park2025community}, under the fixed frozen \texttt{PE-Core-bigG-14-448} encoder \citep{pe-core-bigG}.
(Superscript/subscript denote max/min deviations (in percentage points) from the reported mean).
}
\label{tab:auc_efficiency}
\resizebox{0.99\linewidth}{!}{%
\begin{tabular}{@{}lccccc@{}}
\toprule
\textbf{Amount} &
\textbf{Euc-NCM \citep{wu2025fewshot}} & \textbf{Cos-NCM} & \textbf{GNB} & \textbf{Mah-NCM} & \textbf{QDA} \\
\Midrule
0.001\% & 0.8345$^{+0.0224}_{-0.0407}$ & \textbf{0.9356}$^{+0.0037}_{-0.0042}$ & 0.9239$^{+0.0134}_{-0.0185}$ & 0.8788$^{+0.0132}_{-0.0187}$ & 0.9025$^{+0.0277}_{-0.0148}$ \\[0.5em]
0.005\% & 0.8852$^{+0.0052}_{-0.0032}$ & 0.9313$^{+0.0041}_{-0.0036}$ & 0.9379$^{+0.0039}_{-0.0032}$ & \textbf{0.9383}$^{+0.0037}_{-0.0049}$ & 0.9089$^{+0.0122}_{-0.0137}$ \\[0.5em]
0.01\%  & 0.8897$^{+0.0080}_{-0.0044}$ & 0.9319$^{+0.0018}_{-0.0029}$ & 0.9396$^{+0.0052}_{-0.0048}$ & \textbf{0.9469}$^{+0.0061}_{-0.0045}$ & 0.8991$^{+0.0166}_{-0.0197}$ \\[0.5em]
0.1\%   & 0.8919$^{+0.0050}_{-0.0036}$ & 0.9303$^{+0.0036}_{-0.0032}$ & 0.9403$^{+0.0015}_{-0.0013}$ & \textbf{0.9638}$^{+0.0017}_{-0.0019}$ & 0.8308$^{+0.0044}_{-0.0119}$ \\[0.5em]
1\%     & 0.8947$^{+0.0010}_{-0.0012}$ & 0.9310$^{+0.0009}_{-0.0010}$ & 0.9410$^{+0.0002}_{-0.0001}$ & \textbf{0.9668}$^{+0.0014}_{-0.0008}$ & 0.8478$^{+0.0035}_{-0.0054}$ \\[0.5em]
\midrule
100\%   & 0.8948 & 0.9309 & 0.9412 & \textbf{0.9693} & 0.9295 \\
\bottomrule
\end{tabular}%
}
\end{table}

\Cref{tab:auc_efficiency} adds a small-data qualifier.
At the smallest support size, simpler rules can be more reliable than shared full-covariance estimation.
With only 0.001\% of \CF, Cos-NCM already reaches 0.9356 AUC, whereas Mah-NCM drops to 0.8788.
Mah-NCM overtakes once modest support is available, reaching 0.9638 at 0.1\% and 0.9668 at 1\%, close to the full-support 0.9693.
QDA is visibly unstable and non-monotone under subsampling, consistent with a higher-variance class-specific covariance estimate.
For few-shot adaptation, the data do not support the general rule ``always use the richest Gaussian model.''
They support the more specific rule ``match the covariance model to the amount of support available.''

Overall, the AUC results do not merely replicate the CA results.
The threshold-free view strengthens the main claim.
Low-order feature geometry remains highly competitive, while the failure cases show that the residual value of a trained head, when it exists, is real but narrower than many standalone detector comparisons suggest.

\section{Class-Conditional Gaussianity Diagnostics of Evaluation Features}
\label{sec:gaussianity_diagnostics}

A natural objection to the Gaussian ladder is that frozen-encoder features on shifted evaluation data need not be exactly Gaussian.
The relevant object, however, is the \emph{class-conditional} feature distribution \(p_\phi(z\mid y=c,\mathcal D)\), not the pooled distribution \(p_\phi(z\mid \mathcal D)\): even if each class-conditional component were Gaussian, their mixture would in general not be Gaussian.
We therefore report class-conditional Gaussianity diagnostics for each dataset--class subset in feature space.

These quantities are \emph{diagnostics}, not formal hypothesis tests.
They summarize departures from Gaussian structure, but they do not provide calibrated \(p\)-values and should not be interpreted as accepting or rejecting exact normality.

Let \(z_i=\phi(x_i)\in\mathbb{R}^d\) denote the frozen-encoder feature of sample \(x_i\), restricted to a fixed dataset \(\mathcal D\) and class \(c\in\{0,1\}\).
With \(n\) samples in that subset, empirical mean \(\bar z\), and regularized sample covariance
\begin{equation}
\widehat{\Sigma}
=
\frac{1}{n-1}\sum_{i=1}^{n}(z_i-\bar z)(z_i-\bar z)^\top + \varepsilon I_d,
\qquad \varepsilon = 10^{-8},
\end{equation}
we compute both marginal and multivariate diagnostics.

\paragraph{Finite-sample caveats.}
Marginal skewness/kurtosis estimates can be noisy when the subset size $n$ is small, and multivariate diagnostics based on Mahalanobis radii require $n>d$ for a well-conditioned full-rank covariance estimate.
Accordingly, we interpret large deviations most strongly on dataset--class subsets with ample sample sizes, and we treat results for very small subsets as qualitative indicators rather than definitive evidence of non-Gaussianity.

For each coordinate \(j\in\{1,\dots,d\}\), let \(\hat\gamma_{1,j}\) and \(\hat\gamma_{2,j}\) denote the bias-corrected sample skewness and Pearson kurtosis, respectively, so that the Gaussian reference is \(\hat\gamma_{1,j}=0\) and \(\hat\gamma_{2,j}=3\).
We summarize these marginal statistics by their median and interquartile range across dimensions, and we report the proportion of dimensions satisfying the loose Gaussian screen
\begin{equation}
\mathrm{PctClose}
=
\frac{100}{d}\sum_{j=1}^{d}
\mathbf{1}\!\left(
|\hat\gamma_{1,j}| < 0.5
\;\wedge\;
2 < \hat\gamma_{2,j} < 5
\right).
\end{equation}

At the multivariate level, we compute the empirical squared Mahalanobis radii
\begin{equation}
\delta_i^2
=
(z_i-\bar z)^\top \widehat{\Sigma}^{-1}(z_i-\bar z),
\end{equation}
and the normalized Mardia kurtosis
\begin{equation}
\widetilde{\beta}_{2,d}
=
\frac{1}{d(d+2)}
\cdot
\frac{1}{n}\sum_{i=1}^{n}(\delta_i^2)^2,
\end{equation}

whose Gaussian reference value is \(1\). For completeness, we also report
\begin{equation}
R_{\chi^2}
=
\frac{1}{nd}\sum_{i=1}^{n}\delta_i^2.
\end{equation}

\(R_{\chi^2}\) should be interpreted with caution because \(\widehat{\Sigma}\) is estimated from the same sample, and therefore
\begin{equation}
\frac{1}{nd}\sum_{i=1}^{n}\delta_i^2
\approx
\frac{n-1}{n}
\end{equation}

in the ideal full-rank unregularized setting, so values near \(1\) are largely expected and do \emph{not} constitute strong evidence of Gaussianity.

To summarize the \emph{second-order} geometry of each class-conditional feature cloud, let
$\lambda_1\ge \cdots \ge \lambda_d$ denote the eigenvalues of $\widehat{\Sigma}$ and define
$p_j=\lambda_j/\sum_{k=1}^d \lambda_k$.
We report the effective rank
\begin{equation}
E_{\mathrm{rank}}(\widehat{\Sigma})
=
\exp\!\left(-\sum_{j=1}^{d} p_j \log p_j\right),
\end{equation}
which can be interpreted as the effective number of principal directions carrying variance,
and the anisotropy
\begin{equation}
\mathrm{Aniso}(\widehat{\Sigma})
=
\frac{\lambda_1}{\frac{1}{d}\sum_{j=1}^{d}\lambda_j}
=
\frac{d\,\lambda_1}{\mathrm{tr}(\widehat{\Sigma})},
\end{equation}
which equals $1$ for isotropic covariance and increases as variance concentrates into a small number of directions.

\FloatBarrier
\begingroup
\fontsize{8.75}{11.25}\selectfont
\setlength{\tabcolsep}{3pt}
\centering
\begin{longtable}{@{}>{\raggedright\arraybackslash}p{\dimexpr0.36\linewidth-1.6\tabcolsep\relax}>{\raggedright\arraybackslash}p{\dimexpr0.09\linewidth-1.6\tabcolsep\relax}>{\raggedleft\arraybackslash}p{\dimexpr0.22\linewidth-1.6\tabcolsep\relax}>{\raggedleft\arraybackslash}p{\dimexpr0.22\linewidth-1.6\tabcolsep\relax}>{\raggedleft\arraybackslash}p{\dimexpr0.11\linewidth-1.6\tabcolsep\relax}@{}}
\caption{
\textbf{Marginal Gaussianity diagnostics for class-conditional evaluation features.}
For each dataset--class subset, we compute the bias-corrected sample skewness \(\hat\gamma_{1,j}\) and Pearson kurtosis \(\hat\gamma_{2,j}\) for every feature dimension \(j\).
We report the median with first/third quartiles [Q1,Q3] across dimensions, together with the percentage of dimensions satisfying the heuristic Gaussian screen \( |\hat\gamma_{1,j}|<0.5 \) and \( 2<\hat\gamma_{2,j}<5 \).
These quantities are descriptive diagnostics rather than formal normality tests.
}
\label{tab:univariate_normality_tests} \\
\\
\toprule
\textbf{Dataset} & \textbf{Class} & \multicolumn{1}{c}{\(\hat\gamma_{1,j}\)} & \multicolumn{1}{c}{\(\hat\gamma_{2,j}\)} & \multicolumn{1}{c}{\textbf{PctClose}} \\
\Midrule
\endfirsthead

\toprule
\textbf{Dataset} & \textbf{Class} & \multicolumn{1}{c}{\(\hat\gamma_{1,j}\)} & \multicolumn{1}{c}{\(\hat\gamma_{2,j}\)} & \multicolumn{1}{c}{\textbf{PctClose}} \\
\Midrule
\endhead

\midrule
\multicolumn{5}{r}{\footnotesize Continued on next page} \\
\endfoot

\bottomrule
\endlastfoot
\rowcolor{gray!20}
\multicolumn{5}{c}{\textbf{Out-of-Distribution Existing Datasets with only Real Images}} \\
Unsplash \citep{unsplash}                   & Real &  0.0045 [-0.07,0.08] & 3.0577 [2.99,3.13] & 99.77 \\
InstagramImagesWithCaptions \citep{instagrama} & Real & -0.0001 [-0.07,0.06] & 3.0327 [2.97,3.01] & 98.98 \\
AnimeFace \citep{animeface}                  & Real & -0.0058 [-0.16,0.16] & 3.1563 [3.01,3.30] & 94.45 \\
AnimeImages \citep{animeimages}                & Real & -0.0027 [-0.08,0.08] & 3.0867 [3.02,3.17] & 99.53 \\
CelebA-Spoof \citep{zhang2020celebaspoof}               & Real & -0.0081 [-0.12,0.12] & 3.1013 [2.99,3.21] & 99.38 \\
\midrule
\rowcolor{gray!20}
\multicolumn{5}{c}{\textbf{Out-of-Distribution Existing Datasets with only Synthetic Images}} \\
AGIQA-1k \citep{zhang2023perceptual}            & Fake & -0.0034 [-0.13,0.11] & 2.9904 [2.84,3.16] & 97.81  \\
AGIQA-3k \citep{li2024agiqa3k}            & Fake & -0.0037 [-0.09,0.09] & 3.0732 [2.97,3.19] & 99.45  \\
SPAI \citep{karageorgiou2025anyresolution}                & Fake &  0.0018 [-0.06,0.06] & 3.0477 [2.97,3.14] & 99.69  \\
SynthBuster Extended \citep{bammey2024synthbuster,guillaro2025biasfree} & Fake & -0.0004 [-0.08,0.07] & 3.0470 [2.96,3.14] & 99.38  \\
SynthScars \citep{kang2025legion}          & Fake & -0.0037 [-0.06,0.05] & 3.0438 [2.98,3.11] & 99.92  \\
GigaGAN \citep{kang2023scaling}             & Fake &  0.0016 [-0.08,0.08] & 3.1237 [3.05,3.12] & 100.00 \\
LatentDiffusion \citep{corvi2023detection}     & Fake &  0.0043 [-0.08,0.09] & 3.0735 [2.95,3.19] & 99.45  \\
MidjourneyV6 \citep{terminusresearch}        & Fake &  0.0008 [-0.04,0.04] & 3.0391 [2.99,3.09] & 100.00 \\
Co-Spy-Bench \citep{cheng2025cospy}        & Fake & -0.0005 [-0.05,0.05] & 3.0572 [3.01,3.11] & 100.00 \\
Dalle3 \citep{Egan_Dalle3_1_Million_2024}              & Fake &  0.0003 [-0.04,0.04] & 3.0211 [2.98,3.06] & 99.38  \\
\midrule
\rowcolor{gray!20}
\multicolumn{5}{c}{\textbf{Out-of-Distribution Existing Datasets with both Real and Synthetic Images}} \\
FourierSpectrumDiscrepancies \citep{dzanic2020fourier}  & Real & -0.0496 [-0.41,0.35] & 2.4977 [2.01,3.16] & 38.44  \\
FourierSpectrumDiscrepancies \citep{dzanic2020fourier}  & Fake &  0.0286 [-0.31,0.33] & 2.9860 [2.60,3.52] & 70.39  \\[0.5em]
FakeInversion \citep{cazenavette2024fakeinversion}                 & Real &  0.0027 [-0.08,0.09] & 3.0182 [2.88,3.18] & 99.69  \\
FakeInversion \citep{cazenavette2024fakeinversion}                 & Fake & -0.0019 [-0.08,0.09] & 3.0231 [2.89,3.18] & 100.0  \\[0.5em]
\texttt{UniversalFakeDetect} \citep{ojha2023universal}                            & Real &  0.0019 [-0.06,0.06] & 3.0278 [2.94,3.12] & 99.22  \\
\texttt{UniversalFakeDetect} \citep{ojha2023universal}                            & Fake &  0.0021 [-0.07,0.07] & 3.1456 [3.07,3.25] & 99.84  \\[0.5em]
DalleRecognition \citep{airecognition}              & Real &  0.0025 [-0.07,0.08] & 3.0176 [2.92,3.12] & 99.61  \\
DalleRecognition \citep{airecognition}              & Fake & -0.0002 [-0.04,0.05] & 3.0393 [2.98,3.10] & 99.77  \\[0.5em]
Chameleon \citep{yan2024sanity}                     & Real &  0.0000 [-0.05,0.05] & 3.0130 [2.96,3.07] & 100.0  \\
Chameleon \citep{yan2024sanity}                     & Fake & -0.0073 [-0.09,0.08] & 3.0824 [2.99,3.19] & 99.84  \\[0.5em]
AIGI-Detection-Quality-Paradox \citep{xiao2025are} & Real & -0.0002 [-0.05,0.05] & 3.0274 [2.96,3.12] & 98.36  \\
AIGI-Detection-Quality-Paradox \citep{xiao2025are} & Fake & -0.0020 [-0.07,0.07] & 3.0670 [2.99,3.15] & 99.84  \\[0.5em]
DiTFake \citep{li2025improving}                       & Real &  0.0054 [-0.05,0.06] & 3.0002 [2.93,3.07] & 99.84  \\
DiTFake \citep{li2025improving}                       & Fake &  0.0043 [-0.07,0.08] & 2.9940 [2.91,3.09] & 100.0  \\[0.5em]
Diffusion1kStep \citep{tan2024rethinking}               & Real &  0.0062 [-0.07,0.08] & 3.0114 [2.92,3.11] & 98.67  \\
Diffusion1kStep \citep{tan2024rethinking}               & Fake & -0.0010 [-0.16,0.17] & 3.1901 [2.99,3.39] & 97.27  \\[0.5em]
GANGen-Detection \citep{chuangchuangtan-GANGen-Detection}              & Real & -0.0022 [-0.10,0.11] & 3.1389 [3.02,3.28] & 97.27  \\
GANGen-Detection \citep{chuangchuangtan-GANGen-Detection}               & Fake &  0.0091 [-0.17,0.19] & 3.3372 [3.19,3.51] & 94.38  \\[0.5em]
RobustLDM \citep{rajan2024aligned}                     & Real &  0.0029 [-0.07,0.07] & 3.0167 [2.93,3.11] & 99.77  \\
RobustLDM \citep{rajan2024aligned}                     & Fake & -0.0015 [-0.10,0.10] & 3.0445 [2.93,3.17] & 99.38  \\[0.5em]
RealRobustBench \citep{li2025bridging}               & Real &  0.0017 [-0.06,0.07] & 3.0111 [2.94,3.08] & 99.69  \\
RealRobustBench \citep{li2025bridging}               & Fake &  0.0038 [-0.09,0.10] & 3.0029 [2.90,3.13] & 99.22  \\[0.5em]
LDMFakeDetect \citep{rajan2025staypositive}                 & Real &  0.0029 [-0.07,0.07] & 3.0167 [2.93,3.11] & 99.77  \\
LDMFakeDetect \citep{rajan2025staypositive}                 & Fake & -0.0044 [-0.10,0.09] & 3.0534 [2.95,3.19] & 99.45  \\[0.5em]
DIF \citep{sinitsa2024deep}                           & Real &  0.0001 [-0.06,0.06] & 3.0288 [2.97,3.09] & 99.84  \\
DIF  \citep{sinitsa2024deep}                          & Fake &  0.0032 [-0.09,0.10] & 3.0732 [2.94,3.21] & 99.53  \\[0.5em]
ForenSynths \citep{wang2020cnngenerated}                   & Real &  0.0143 [-0.23,0.24] & 3.2898 [2.97,3.50] & 97.27  \\
ForenSynths \citep{wang2020cnngenerated}                   & Fake &  0.0119 [-0.20,0.23] & 3.1450 [2.84,3.40] & 93.98  \\[0.5em]
DNFTestSet \citep{zhang2025diffusion}                    & Real &  0.0006 [-0.06,0.07] & 3.0247 [2.95,3.10] & 99.69  \\
DNFTestSet \citep{zhang2025diffusion}                    & Fake &  0.0038 [-0.11,0.13] & 3.1119 [2.95,3.24] & 99.45  \\[0.5em]
DeepFakeFace \citep{song2023robustness}                  & Real & -0.0021 [-0.06,0.06] & 3.0008 [2.95,3.06] & 99.84  \\
DeepFakeFace  \citep{song2023robustness}                 & Fake & -0.0027 [-0.09,0.08] & 3.0757 [3.01,3.15] & 99.84  \\[0.5em]
AIGCDetectBench \citep{zhong2024patchcraft}               & Real &  0.0020 [-0.05,0.05] & 3.0221 [2.97,3.07] & 99.61  \\
AIGCDetectBench \citep{zhong2024patchcraft}               & Fake &  0.0079 [-0.09,0.10] & 3.1043 [3.00,3.21] & 99.61  \\[0.5em]
AIGI-Holmes  \citep{zhou2025aigiholmes}                   & Real &  0.0012 [-0.05,0.05] & 3.0068 [2.95,3.06] & 99.77  \\
AIGI-Holmes  \citep{zhou2025aigiholmes}                   & Fake &  0.0000 [-0.09,0.09] & 3.0812 [2.99,3.17] & 100.0  \\[0.5em]
AI-Artwork \citep{aiartwork}                    & Real & -0.0032 [-0.06,0.06] & 3.0122 [2.95,3.08] & 99.77  \\
AI-Artwork  \citep{aiartwork}                   & Fake &  0.0050 [-0.12,0.12] & 3.0849 [2.95,3.23] & 99.30  \\[0.5em]
DiffusionForensics \citep{wang2023dire}            & Real &  0.0079 [-0.12,0.12] & 3.1291 [3.02,3.22] & 99.45  \\
DiffusionForensics \citep{wang2023dire}            & Fake &  0.0071 [-0.31,0.32] & 3.7335 [3.49,4.01] & 71.48  \\[0.5em]
B-Free \citep{guillaro2025biasfree}                        & Real &  0.0049 [-0.05,0.06] & 2.9950 [2.93,3.07] & 99.92  \\
B-Free \citep{guillaro2025biasfree}                       & Fake &  0.0006 [-0.06,0.06] & 3.0396 [2.98,3.10] & 100.0  \\[0.5em]
LASTED \citep{wu2025generalizable}                        & Real &  0.0078 [-0.06,0.07] & 3.0478 [2.98,3.11] & 100.0  \\
LASTED \citep{wu2025generalizable}                        & Fake & -0.0028 [-0.12,0.12] & 3.2195 [3.11,3.34] & 99.22  \\[0.5em]
AIGIBench \citep{li2025artificial}                     & Real &  0.0001 [-0.06,0.06] & 3.0447 [2.99,3.11] & 99.92  \\
AIGIBench \citep{li2025artificial}                     & Fake &  0.0067 [-0.11,0.13] & 3.1626 [3.05,3.28] & 99.77  \\[0.5em]
DeepFakeBench \citep{yan2023deepfakebench}                 & Real & -0.0008 [-0.17,0.18] & 3.1662 [2.97,3.34] & 98.52  \\
DeepFakeBench  \citep{yan2023deepfakebench}                & Fake &  0.0101 [-0.16,0.17] & 3.0508 [2.75,3.31] & 94.69  \\

\end{longtable}
\endgroup

\FloatBarrier
\begingroup
\fontsize{10}{13.75}\selectfont
\setlength{\tabcolsep}{4pt}
\centering
\begin{longtable}[c]{@{}llrrrr@{}}
\caption{
\textbf{Multivariate diagnostics for class-conditional evaluation features.}
Let $r_i = z_i-\bar z$ and let $\widehat{\Sigma}$ be the regularized sample covariance of the class-conditional features.
We compute squared Mahalanobis radii
$\delta_i^2 = r_i^\top \widehat{\Sigma}^{-1} r_i$.
We report (i) the normalized Mardia kurtosis
$\widetilde{\beta}_{2,d}=\frac{1}{d(d+2)}\frac{1}{n}\sum_{i=1}^{n}(\delta_i^2)^2$,
whose Gaussian reference value is $1$,
(ii) the normalized mean radius
$R_{\chi^2}=\frac{1}{nd}\sum_{i=1}^{n}\delta_i^2$ (a limited self-consistency diagnostic; with same-sample mean/covariance estimation it is expected to be near $1$ even under non-Gaussianity),
and two spectrum/conditioning summaries of $\widehat{\Sigma}$:
(iii) the effective rank
$E_{\mathrm{rank}}=\exp\!\big(-\sum_{j=1}^d p_j \log p_j\big)$ with $p_j=\lambda_j/\sum_k \lambda_k$,
and (iv) anisotropy
$\mathrm{Aniso}=\lambda_1 / \big(\frac{1}{d}\sum_{j=1}^{d}\lambda_j\big)= d\lambda_1/\mathrm{tr}(\widehat{\Sigma})$,
where $\lambda_1\ge \cdots \ge \lambda_d$ are the eigenvalues of $\widehat{\Sigma}$.
$E_{\mathrm{rank}}\in[1,d]$ summarizes effective dimensionality, while $\mathrm{Aniso}\ge 1$ measures dominance of the top variance direction.
All quantities are descriptive diagnostics rather than calibrated hypothesis tests.
(\textbf{Note:} We report dataset/classes where number of samples is greater than number of feature dimensions.)
}
\label{tab:multivariate_normality_tests} \\
\\
\toprule
\textbf{Dataset} & \textbf{Class} & \multicolumn{1}{c}{\(\widetilde{\beta}_{2,d}\)} & \multicolumn{1}{c}{\(R_{\chi^2}\)} & \multicolumn{1}{c}{$E_{\mathrm{rank}}$} & \multicolumn{1}{c}{$\mathrm{Aniso}$} \\
\Midrule
\endfirsthead

\toprule
\textbf{Dataset} & \textbf{Class} & \multicolumn{1}{c}{\(\widetilde{\beta}_{2,d}\)} & \multicolumn{1}{c}{\(R_{\chi^2}\)} & \multicolumn{1}{c}{$E_{\mathrm{rank}}$} & \multicolumn{1}{c}{$\mathrm{Aniso}$} \\
\Midrule
\endhead

\midrule
\multicolumn{6}{r}{\footnotesize Continued on next page} \\
\endfoot

\bottomrule
\endlastfoot
\rowcolor{gray!20}
\multicolumn{6}{c}{\textbf{\shortstack{Out-of-Distribution Existing Datasets with \\only Real Images}}} \\
Unsplash \citep{unsplash}                   & Real & 1.1225 & 0.9998 & 227.80 &  75.23 \\
InstagramImagesWithCaptions \citep{instagrama} & Real & 1.1092 & 0.9998 & 251.47 &  74.98 \\
AnimeFace \citep{animeface}                  & Real & 1.1957 & 0.9995 &  61.89 & 293.73 \\
AnimeImages \citep{animeimages}                & Real & 1.1715 & 0.9998 & 254.45 &  61.78 \\
CelebA-Spoof \citep{zhang2020celebaspoof}               & Real & 1.1694 & 0.9998 & 171.74 & 166.86 \\
\midrule
\rowcolor{gray!20}
\multicolumn{6}{c}{\textbf{\shortstack{Out-of-Distribution Existing Datasets with \\only Synthetic Images}}} \\
AGIQA-3k \citep{li2024agiqa3k}            & Fake & 1.0319 & 0.9992 & 210.56 & 76.71 \\
SPAI \citep{karageorgiou2025anyresolution}                & Fake & 1.0574 & 0.9994 & 267.96 & 51.12 \\
SynthBuster Extended \citep{bammey2024synthbuster,guillaro2025biasfree} & Fake & 1.0879 & 0.9996 & 194.50 & 91.25 \\
SynthScars \citep{kang2025legion}          & Fake & 1.0690 & 0.9997 & 250.11 & 60.05 \\
GigaGAN \citep{kang2023scaling}             & Fake & 1.1644 & 0.9998 & 268.03 & 80.58 \\
LatentDiffusion \citep{corvi2023detection}     & Fake & 1.1179 & 0.9997 & 212.48 & 76.55 \\
MidjourneyV6 \citep{terminusresearch}        & Fake & 1.0934 & 0.9998 & 281.47 & 50.06 \\
Co-Spy-Bench \citep{cheng2025cospy}        & Fake & 1.0772 & 0.9998 & 299.09 & 39.92 \\
Dalle3 \citep{Egan_Dalle3_1_Million_2024}              & Fake & 1.0995 & 0.9998 & 280.66 & 54.19 \\
\midrule
\rowcolor{gray!20}
\multicolumn{6}{c}{\textbf{\shortstack{Out-of-Distribution Existing Datasets with \\both Real and Synthetic Images}}} \\
\texttt{UniversalFakeDetect} \citep{ojha2023universal}                            & Real & 1.0030 & 0.9991 & 339.72 & 76.99 \\
\texttt{UniversalFakeDetect} \citep{ojha2023universal}                            & Fake & 1.0742 & 0.9996 & 282.82 & 52.51 \\[0.5em]
DalleRecognition \citep{airecognition}              & Real & 1.0335 & 0.9995 & 229.58 & 111.29\\
DalleRecognition \citep{airecognition}              & Fake & 1.0719 & 0.9997 & 291.76 & 39.74 \\[0.5em]
Chameleon \citep{yan2024sanity}                     & Real & 1.0619 & 0.9997 & 262.78 & 77.87 \\
Chameleon \citep{yan2024sanity}                     & Fake & 1.1232 & 0.9996 & 194.40 & 119.96 \\[0.5em]
AIGI-Detection-Quality-Paradox \citep{xiao2025are} & Real & 1.0430 & 0.9995 & 319.79 & 61.35 \\
AIGI-Detection-Quality-Paradox \citep{xiao2025are} & Fake & 1.1087 & 0.9997 & 263.10 & 55.39 \\[0.5em]
DiTFake \citep{li2025improving}                       & Real & 1.0507 & 0.9997 & 281.90 & 51.76\\
DiTFake \citep{li2025improving}                       & Fake & 1.0900 & 0.9997 & 213.86 & 58.78\\[0.5em]
Diffusion1kStep \citep{tan2024rethinking}               & Real & 1.0448 & 0.9997 & 222.71 & 109.72 \\
Diffusion1kStep \citep{tan2024rethinking}               & Fake & 1.2005 & 0.9997 & 148.31 & 148.77 \\[0.5em]
GANGen-Detection \citep{chuangchuangtan-GANGen-Detection}              & Real & 1.1642 & 0.9997 & 187.29 & 101.55 \\
GANGen-Detection \citep{chuangchuangtan-GANGen-Detection}              & Fake & 1.2674 & 0.9996 & 146.67 & 122.43 \\[0.5em]
RobustLDM \citep{rajan2024aligned}                     & Real & 1.0408 & 0.9996 & 249.19 & 98.38 \\
RobustLDM \citep{rajan2024aligned}                     & Fake & 1.2930 & 0.9997 & 177.50 & 72.18 \\[0.5em]
RealRobustBench \citep{li2025bridging}               & Real & 1.0891 & 0.9998 & 239.64 & 73.09\\
RealRobustBench \citep{li2025bridging}               & Fake & 1.1513 & 0.9997 & 173.27 & 98.95\\[0.5em]
LDMFakeDetect \citep{rajan2025staypositive}                 & Real & 1.0408 & 0.9996 & 249.19 & 98.38\\
LDMFakeDetect \citep{rajan2025staypositive}                 & Fake & 1.3017 & 0.9997 & 178.20 & 71.55\\[0.5em]
DIF \citep{sinitsa2024deep}                           & Real & 1.0587 & 0.9998 & 309.16 & 60.52\\
DIF \citep{sinitsa2024deep}                           & Fake & 1.1797 & 0.9997 & 192.48 & 89.98\\[0.5em]
ForenSynths \citep{wang2020cnngenerated}                   & Real & 1.2396 & 0.9998 & 182.97 & 193.47\\
ForenSynths \citep{wang2020cnngenerated}                   & Fake & 1.2685 & 0.9997 & 104.62 & 218.03\\[0.5em]
DNFTestSet \citep{zhang2025diffusion}                    & Real & 1.0450 & 0.9997 & 274.26 & 84.16 \\
DNFTestSet  \citep{zhang2025diffusion}                   & Fake & 1.1293 & 0.9997 & 198.64 & 117.17\\[0.5em]
DeepFakeFace \citep{song2023robustness}                  & Real & 1.0758 & 0.9998 & 237.47 & 74.91 \\
DeepFakeFace \citep{song2023robustness}                  & Fake & 1.1205 & 0.9998 & 221.72 & 72.87 \\[0.5em]
AIGCDetectBench \citep{zhong2024patchcraft}               & Real & 1.0540 & 0.9998 & 323.16 & 69.08 \\
AIGCDetectBench \citep{zhong2024patchcraft}               & Fake & 1.1375 & 0.9997 & 228.13 & 83.67\\[0.5em]
AIGI-Holmes  \citep{zhou2025aigiholmes}                   & Real & 1.0686 & 0.9998 & 303.84 & 49.64 \\
AIGI-Holmes  \citep{zhou2025aigiholmes}                   & Fake & 1.1375 & 0.9998 & 222.87 & 64.21\\[0.5em]
AI-Artwork \citep{aiartwork}                    & Real & 1.0767 & 0.9998 & 220.19 & 71.57 \\
AI-Artwork \citep{aiartwork}                    & Fake & 1.2529 & 0.9997 & 144.42 & 114.62 \\[0.5em]
DiffusionForensics \citep{wang2023dire}            & Real & 1.1251 & 0.9998 & 236.08 & 129.66 \\
DiffusionForensics  \citep{wang2023dire}           & Fake & 1.4326 & 0.9997 & 144.56 & 190.06\\[0.5em]
B-Free \citep{guillaro2025biasfree}                        & Real & 1.0637 & 0.9998 & 278.64 & 52.92 \\
B-Free \citep{guillaro2025biasfree}                        & Fake & 1.0798 & 0.9998 & 267.28 & 55.60\\[0.5em]
LASTED \citep{wu2025generalizable}                        & Real & 1.0803 & 0.9998 & 300.73 & 58.42 \\
LASTED  \citep{wu2025generalizable}                       & Fake & 1.3136 & 0.9997 & 210.48 & 111.79\\[0.5em]
AIGIBench \citep{li2025artificial}                     & Real & 1.0864 & 0.9998 & 303.42 & 65.94 \\
AIGIBench \citep{li2025artificial}                     & Fake & 1.1784 & 0.9998 & 217.27 & 77.97\\[0.5em]
DeepFakeBench \citep{yan2023deepfakebench}                 & Real & 1.1376 & 0.9996 & 128.46 & 208.58 \\
DeepFakeBench \citep{yan2023deepfakebench}                 & Fake & 1.2176 & 0.9996 & 73.41 & 310.36 \\
\end{longtable}
\endgroup

Applied to features from the strongest frozen encoder in our sweep, \texttt{PE-Core-bigG-14-448} \citep{pe-core-bigG}, the diagnostics support Gaussian heads as approximate second-order models, but not as exact density models.
In \Cref{tab:univariate_normality_tests}, most dataset--class subsets have median skewness near zero, median Pearson kurtosis near three, and high \texttt{PctClose}, often above 97\%.
The clearest marginal exceptions are the real and fake subsets of \texttt{FourierSpectrumDiscrepancies} \citep{dzanic2020fourier} and the fake subset of \texttt{DiffusionForensics} \citep{wang2023dire}, where deviations from the Gaussian reference are large.

For the Gaussian ladder, the multivariate diagnostics in \Cref{tab:multivariate_normality_tests} are more relevant.
The normalized Mardia kurtosis \(\widetilde{\beta}_{2,d}\) is often only moderately above its Gaussian reference value of \(1\), but the larger deviations occur disproportionately in fake subsets.
These same subsets frequently exhibit lower effective rank and higher anisotropy than their real counterparts, indicating that variance is concentrated in fewer dominant directions.
This pattern is not universal, but it recurs often enough to make covariance-aware scoring plausible.
By contrast, \(R_{\chi^2}\) remains close to \(1\) almost everywhere, as expected under same-sample estimation, and should be interpreted only as a consistency check.

The resulting conclusion is deliberately narrow.
Evaluation features are not exactly Gaussian, and several datasets exhibit clear departures from Gaussianity.
For many dataset--class subsets, however, first- and second-order summaries describe a useful part of the class structure, while the failures mark cases where higher-order or multimodal structure likely matters.
This supports Gaussian heads as controlled, interpretable probes in matched support-prior and encoder audits, not as literal generative models of the feature distribution.

Read together with the transfer results, these diagnostics support a methodological claim rather than a distributional one:
when the support prior and frozen encoder feature space are fixed, Gaussian heads indicate the extent to which the observed head-level behavior is already available in low-order geometry.
They do not imply that low-order geometry is the only transfer mechanism or that Gaussian assumptions hold globally.
Head-level comparisons are therefore most interpretable when both quantities are controlled.
\end{document}